\documentclass[11pt]{article}
\usepackage[letterpaper, textwidth=17cm, centering, top=2.5cm, bottom=2.5cm]{geometry}

\usepackage[utf8]{inputenc}
\usepackage{amsthm}
\usepackage{adjustbox}
\usepackage{mathtools}
\usepackage{tikz}
\usetikzlibrary{arrows.meta, positioning}
\usepackage{amsmath}
\usepackage{bm}
\usepackage{yhmath}
\usepackage{enumitem}
\usepackage{graphicx}
\usepackage{algorithm}
\usepackage{subcaption}
\usepackage{multirow}
\usepackage{dcolumn}
\usepackage{array}
\usepackage{algorithmicx}  
\usepackage{algpseudocode} 
\usepackage{amssymb} 
\usepackage{mathrsfs}
\usepackage{cellspace}
\def\P{\mathbb{P}}
\def\R{\mathbb{R}}

\def\S{\mathbb{S}}

\def\la{\langle}
\def\PSD{{\S^k_{++}}}
\def\ra{\rangle}

\usepackage{natbib} 

\usepackage{amsmath}
\usepackage{amsfonts}
\usepackage{amssymb}
\usepackage{xcolor}
\usepackage{epsfig}
\usepackage{makecell}
\usepackage{amsmath}
\usepackage{amsfonts}
\usepackage{amssymb}
\usepackage{dsfont}
\usepackage{mathrsfs}
\usepackage{bbm}
\usepackage{amsmath}
\usepackage{amsfonts}
\usepackage{amssymb}
\usepackage{multirow}
\usepackage{mathrsfs}
\let\amssymbboxplus\boxplus
\let\amssymbboxminus\boxminus

\renewcommand{\boxplus}{\mathbin{\mathop\amssymbboxplus}}
\renewcommand{\boxminus}{\mathbin{\mathop\amssymbboxminus}}

\usepackage{makecell} 
\usepackage{boldline}
\usepackage{diagbox}
\usepackage{threeparttable}
\usepackage{adjustbox}

\usepackage[title]{appendix}

\usepackage[colorlinks=true]{hyperref}

\newtheorem{theorem}{Theorem}
\newtheorem{lemma}{Lemma}
\newtheorem{proposition}{Proposition}

\newtheorem{remark}{Remark}

\newtheorem{assumption}{Assumption}
\newtheorem{corollary}{Corollary}

\newcommand{\setd}{{ d \kern -.15em l}}
\newcommand{\hatsetd}{ d \hat{\kern -.15em l }}
\newcommand{\dd}{\mathsf {d\kern -0.07em l}}

\newcommand{\bgeqn}{\begin{eqnarray}}
\newcommand{\edeqn}{\end{eqnarray}}
\newcommand{\bgeq}{\begin{eqnarray*}}
\newcommand{\edeq}{\end{eqnarray*}}

\newcommand{\opt}{^\star}
\newcommand{\Let}{\triangleq}

\def\st{\text{s.t.}}
\def\diag{\mathrm{diag}}
\def\transpose{\top}

\def\P{\mathbb{P}}
\def\PP{\mathbb{P}}
\def\R{\mathbb{R}}

\def\S{\mathbb{S}}

\def\E {\mathbb{E}}
\def\tr{\mathrm{Tr}}
\def\la{\langle}
\def\PSD{{\S^p_{+}}}

\def\PD{{\S^p_{++}}}
\def\ra{\rangle}

\newcommand{\mc}{\mathcal}
\newcommand{\half}{\frac{1}{2}}

\usepackage[colorlinks=true]{hyperref}
 \hypersetup{urlcolor=blue,
 citecolor=blue,linkcolor=blue}
   
\makeatletter
\g@addto@macro{\UrlBreaks}{\UrlOrds}
\makeatother

\newcommand{\wh}{\widehat}
\newcommand{\covsa}{\wh \Sigma}
\newcommand{\Pnom}{\wh \PP}

\title{SCOPE: Spectral Concentration by \\Distributionally Robust Joint Covariance-Precision Estimation}

\author{Renjie Chen\thanks{The authors are with the Chinese University of Hong Kong (\url{rchen@se.cuhk.edu.hk, nguyen@se.cuhk.edu.hk, hfxu@se.cuhk.edu.hk}).} \and Viet Anh Nguyen\footnotemark[1] \and Huifu Xu\footnotemark[1]} 

\date{}

\begin{document}

\maketitle

\begin{abstract}
    We propose a distributionally robust formulation for simultaneously estimating the covariance matrix and the precision matrix of a random vector. 
    The proposed model minimizes the worst-case weighted sum of the Frobenius loss of the covariance estimator and Stein's loss of the precision matrix estimator against all distributions from an ambiguity set centered at the nominal distribution. The radius of the ambiguity set is measured via convex spectral divergence. We demonstrate that the proposed distributionally robust estimation model can be reduced to a convex optimization problem, thereby yielding quasi-analytical estimators. The joint estimators are shown to be nonlinear shrinkage estimators. The eigenvalues of the estimators are shrunk nonlinearly towards a positive scalar, where the scalar is determined by the weight coefficient of the loss terms. By tuning the coefficient carefully, the shrinkage corrects the spectral bias of the empirical covariance/precision matrix estimator.
    By this property, we call the proposed joint estimator the \textbf{S}pectral concentrated \textbf{CO}variance and \textbf{P}recision matrix \textbf{E}stimator (\textbf{SCOPE}). We demonstrate that the shrinkage effect improves the condition number of the estimator. We provide a parameter-tuning scheme that adjusts the shrinkage target and intensity that is asymptotically optimal. Numerical experiments on synthetic and real data show that our shrinkage estimators perform competitively against state-of-the-art estimators in practical applications. 
\end{abstract}

\section{Introduction}

 The covariance matrix $\Sigma_0\Let\E_{\P}[(\xi-\E_{\P}[\xi])(\xi-\E_{\P}[\xi])^\transpose]$ of a random vector $\xi \in \R^p$ with distribution $\PP$ is a widely-used summary statistic that captures the dispersion of $\xi$. The inverse of the covariance matrix, $\Sigma_0^{-1}$, is called the precision matrix. It is intuitive from the terminology that a random vector with a large covariance matrix is of low precision, and vice versa~\citep{ref:nguyen2022distributionally}. 
The covariance matrix and the precision matrix, and their estimators, are ubiquitous in data-driven applications.
In Markowitz's mean-variance model~\citep{ref:Mark1952}, the covariance matrix defines the risk of the given portfolio, and the precision matrix determines the optimal portfolio that achieves the minimal risk. In risk-sensitive control, the policy optimization objective itself depends on both the covariance matrix and the precision matrix of the Gaussian noise~\citep[Appendix~1]{ref:wang2020game}. In Gaussian graphical model~\citep{ref:yuan2007model,ref:zhao2019cancer}, a Gaussian random variable is associated with a graph. The covariance matrix encodes the marginal correlations of the nodes, and the precision matrix encodes the conditional independence of the nodes on the graph. In moment-based distributionally robust optimization~\citep{ref:Delage2010}, the first moment uncertainty is captured by an ellipsoid whose shape matrix is the precision matrix, and the second moment uncertainty involves a matrix inequality defined by the covariance matrix. 
These examples show that the covariance matrix and the precision matrix are often used in conjunction in practical applications. However, in most data-driven problems, the true covariance matrix and precision matrix are unknown. This requires a statistical model that provides stable estimators of the covariance and precision matrix.

The estimation of the covariance matrix and the precision matrix is widely studied in the literature. Since the true distribution of $\xi$ 
is usually not accessible, and the calculation of multi-dimensional integration could be prohibitively expensive (even if we have access to $\P$), we typically have to estimate the true distribution via $n$ independent and identically distributed samples $\xi_1,\ldots,\xi_n\sim \P$. Let $\widehat{\mu}\Let \frac{1}{n}\sum_{i=1}^n \xi_i$ be the sample mean. The most straightforward unbiased estimator of the covariance matrix is the sample covariance matrix $\widehat{\Sigma}\Let \frac{1}{n-1}\sum_{i=1}^n (\xi_i-\widehat{\mu})(\xi_i-\widehat{\mu})^\transpose$. When $\wh{\Sigma}$ is non-singular, a standard estimator of the precision matrix is $\wh{\Sigma}^{-1}$. However, this approach does not give a valid estimator of the precision matrix when $\wh\Sigma$ is singular, which is a circumstance that occurs frequently in the high-dimensional and data-deficient regime, i.e., when the sample size is less than the dimension of the problem ($n<p$)~\citep{ref:wainwright2019high}. 
Another drawback of the sample covariance matrix is that it suffers from systematic spectral bias. To see this, we start from the estimation bias of the largest and smallest eigenvalues of $\Sigma_0$. Let $\lambda_{\max}(X)$ and $\lambda_{\min}(X)$ be the largest eigenvalue and the smallest eigenvalue of the matrix $X$, respectively.
Since $\lambda_{\max}(\cdot)$ is a convex function and $\lambda_{\min}(\cdot)$ is a concave function over the space of symmetric matrices, then by Jensen's inequality, it holds that
\begin{align}
    \lambda_{\max}(\Sigma_0)=\lambda_{\max}(\E_n[\wh\Sigma]) \leq \E_n[\lambda_{\max}(\wh\Sigma)]~\text{and}~\lambda_{\min}(\Sigma_0)=\lambda_{\min}(\E_n[\wh\Sigma]) \geq \E_n[\lambda_{\min}(\wh\Sigma)],\label{ineq:Jensen}
\end{align}
where $\E_n$ denotes the expectation with respect to the $n$-product measures of the data-generating distribution $\P$ and the expectation is taken on $\wh\Sigma$. Inequalities in~\eqref{ineq:Jensen} reveal that, on average, $\lambda_{\max}(\wh\Sigma)$ overestimates its true counterpart $\lambda_{\max}(\Sigma_0)$, while $\lambda_{\min}(\wh\Sigma)$ underestimates its true counterpart $\lambda_{\min}(\Sigma_0)$. 
The spectral bias occurs not only in the estimation of the largest and smallest eigenvalues but also in the estimation of all other eigenvalues.
Let $\lambda(\Sigma_0)\Let (\lambda_1,\ldots,\lambda_p)$ be the eigenvalues of $\Sigma_0$, and $\lambda(\wh\Sigma)\Let (\wh\lambda_1,\ldots,\wh\lambda_p)$ be the eigenvalues of $\wh\Sigma$, both in increasing order. In~\cite{ref:LEDOIT2004365}, the authors point out that, on average, $\lambda(\wh\Sigma)$ is more dispersed than $\lambda(\Sigma_0)$, in the sense that
\[
    \frac{1}{p}\E_n\left[\sum_{i=1}^p (\wh\lambda_i-\overline{\lambda})^2\right] = \frac{1}{p}\sum_{i=1}^p (\lambda_i-\overline{\lambda})^2 + \E_n[\|\wh\Sigma-\Sigma_0\|_F^2] > \frac{1}{p}\sum_{i=1}^p (\lambda_i-\overline{\lambda})^2,
\]
where $\overline{\lambda}\Let \frac{1}{p}\sum_{i=1}^p \lambda_i$.
The inequality above implies that systematic bias may occur in the estimation of all eigenvalues in a similar manner of~\eqref{ineq:Jensen}: 
\begin{align}\label{bias-2}
    \begin{cases}
    \text{$\wh\lambda_i$ overestimates $\lambda_i$,} & \text{when $\lambda_i$ is large,} \\
    \text{$\wh\lambda_i$ underestimates $\lambda_i$,} & \text{when $\lambda_i$ is small.}
    \end{cases}
\end{align}
From a distributional point of view, the dispersion of $\lambda(\wh\Sigma)$ can also be explained by the Marchenko–Pastur law~\citep{ref:marvcenko1967distribution}. The law states that in the high-dimensional regime where $p,n\to \infty, p/n \to c$ and $\Sigma_0 = I$, the eigenvalues of $\wh\Sigma$ are supported on the interval $[(1-\sqrt{c})^2, (1+\sqrt{c})^2]$. It is intuitive to see that the ratio $p/n$ governs the degree of dispersion: when the sample size $n$ is large relative to $p$, the eigenvalues of $\wh\Sigma$ concentrate around $1$; conversely, when the sample size is relatively small, the eigenvalues spread over a wider interval.

The spectral bias poses challenges on both the computational and modeling sides. Firstly, it leads to an inflated condition number of $\wh\Sigma$, making the calculation of $\wh{\Sigma}^{-1}$ computationally unstable. Recall that the condition number of $\wh\Sigma$ is defined by $\kappa(\wh\Sigma)\Let \wh\lambda_p/\wh\lambda_1$. 
A combination of the spectral bias effects in~\eqref{ineq:Jensen} inflates the condition number $\kappa(\wh\Sigma)$. Such spectral bias and ill-conditioning not only bring in the difficulty of calculating $\wh\Sigma^{-1}$, but also amplify the estimation error in downstream precision-matrix-related models.
Markowitz's mean-variance model could be an example, since the construction of the optimal portfolio involves the precision matrix estimator~\citep{ref:LEDOIT2003603,ref:michaud1989markowitz}.
From a modeling perspective, one may view $\xi$ as random returns of risky assets and interpret the eigenvalues of $\wh\Sigma$ as the risk or volatility of the portfolios induced by the eigenvectors of $\wh\Sigma$. The spectral bias, as indicated in~\eqref {bias-2}, suggests that we overestimate the potential risk of the more risky portfolio and underestimate the potential risk of the less risky portfolio. Such estimation bias could be misleading and may result in suboptimal decision-making in risk-averse applications.

An effective and promising way to correct the spectral bias and reduce the mean square error of the covariance-precision matrix estimator is to shrink the sample covariance matrix towards a well-conditioned and data-\textit{in}sensitive target matrix. In~\cite{ref:stein1975estimation,ref:stein1986lectures}, Charles Stein first proposed that shrinking the sample covariance matrix toward a constant matrix helps reduce the mean squared error of the estimation. 
Since Stein's seminal work, the study of Stein-type shrinkage estimators has been an important research area in statistics. Formally, let $\wh\Sigma=\widehat{V}\diag(\wh\lambda_1,\ldots,\wh\lambda_p)\widehat{V}^\transpose$ be the spectral decomposition of the sample covariance matrix. A covariance matrix estimator $\widetilde{\Sigma}$ is called a Stein-type shrinkage estimator if it is of the form $\widetilde{\Sigma}=\wh V \diag(\widetilde{\lambda}_i,\ldots,\widetilde{\lambda}_p) \wh V^\transpose$, where $\widetilde{\lambda}_i\Let\varphi(\wh\lambda_i),i=1,\ldots,p$ are the shrunk eigenvalues induced by the eigenvalue mapping $\varphi(\cdot)$~\citep{ref:condition2012won}. Note that Stein-type shrinkage estimators are {\it rotation-equivariant}, 
i.e., let $\widetilde{\Sigma}(\{\xi_i\}_{i=1}^n)$ denote the estimator on the dataset $\{\xi_i\}_{i=1}^n$ and $R$ be a rotation matrix. Then the estimator $\widetilde{\Sigma}(\{R\xi_i\}_{i=1}^n)$ on the rotated dataset $\{R\xi_i\}_{i=1}^n$ is equivalent to $R\widetilde{\Sigma}(\{\xi_i\}_{i=1}^n) R^\transpose$.

In the study of Stein-type shrinkage covariance-precision matrix estimators, there are two main research directions: linear shrinkage and nonlinear shrinkage, which are determined by whether the eigenvalue mapping $\varphi$ is linear or nonlinear. A linear shrinkage estimator is usually induced by a convex combination of the sample covariance matrix $\wh\Sigma$ and a data-insensitive shrinkage target, i.e., $\widetilde{\Sigma}=(1- t)\wh\Sigma+ t T$. Typical choices of shrinkage target $T$ include: (a) $\frac{1}{p}\tr(\Sigma_0)I$~\citep{ref:LEDOIT2004365}, which is the optimal scalar matrix that minimize the mean squared error of the shrinkage estimator, (b) the constant correlation model~\citep{ref:ledoit2003honey}, in which the target is a sample constant correlation matrix, (c) the single index model~\citep{ref:LEDOIT2003603}, in which the target is the sum of a rank-one matrix and a diagonal matrix that represent systematic and
idiosyncratic risk factors of Sharpe's single index model~\citep{ref:sharpe1963simplified}, (d) semantic similarity shrinkage~\citep{ref:becquin2023semantic}, in which the target is constructed from semantic similarity of the textual
descriptions or knowledge graphs. The combination coefficient $ t$ is usually chosen by minimizing the mean squared error of the induced shrinkage estimator. A linear shrinkage estimator is Stein-type (rotation-equivariant) only if the target commutes with $\wh\Sigma$.

As an extension of the linear shrinkage estimator,~\cite{ref:Ledoit2012,ref:LedoitWolfNonlinear} proposed a Stein-type nonlinear shrinkage estimator with a well-constructed eigenvalue mapping that is asymptotically optimal with respect to the mean squared error. The optimal shrinkage estimator of the precision matrix can be constructed in a similar manner~\citep{ref:BODNAR2016223,ref:ledoit2022quadratic}.~\cite{ref:condition2012won} proposed a condition-number-regularized covariance estimation model. The model seeks to maximize the log-likelihood of the estimator, subject to the constraint that the condition number of the estimator does not exceed a predetermined upper bound. It turns out that the optimizer can be viewed as a Stein-type nonlinear shrinkage estimator with a piece-wise linear eigenvalue mapping. Recently,~\cite{ref:liu2025covariance} points out that in financial markets, the returns of multiple assets are usually positively correlated, and considers the covariance matrix estimation problem under this decision scenario. Unlike traditional (linear or nonlinear) shrinkage estimators that shrink all the eigenvalues towards the same target, the proposed Stein-type shrinkage estimator adjusts the eigenvalues in pairs and shrinks the gap between pairs of eigenvalues.
For a review of shrinkage covariance estimators, see~\cite{ref:10.1093/jjfinec/nbaa007}.

Recently, it has been demonstrated that Stein-type shrinkage estimators can also be constructed from a distributionally robust optimization (DRO) perspective. \cite{ref:yue2024geometric} use a second-moment-based distributionally robust approach to model the estimation of covariance matrices, and show that the resulting estimator is a Stein-type estimator that shrinks the eigenvalues of the nominal covariance estimator towards zero in a nonlinear sense. Similarly,~\cite{ref:nguyen2022distributionally}~considers the Wasserstein distributionally robust precision matrix estimation problem. The proposed Stein-type precision matrix estimator shrinks the eigenvalues of the precision matrix toward zero in a nonlinear manner. 
These two seminal works establish a connection between the distributionally robust optimization models from the OR society and the Stein-type estimators from the statistics society. 
However, neither of them addresses the issue of systematic spectral bias~\eqref{bias-2}. Both methods shrink the covariance or precision matrix estimator towards the zero matrix, and thus only correct one side of the spectral bias in~\eqref{bias-2}. Such a one-sided correction will aggravate the spectral bias on the other side and lead to a misunderstanding about the variance. For a comparison of the shrinkage estimators proposed in the literature, refer to Table~\ref{tab:comp-estimator}.

The tuning of the radius of the divergence-based ambiguity set is another crucial aspect of the DRO model. The radius represents the conservativeness of the model, specifically controlling the intensity of shrinkage in the Stein-type distributionally robust covariance/precision matrix estimators mentioned above. A larger radius implies greater shrinkage. From a theoretical perspective, the radius of the ambiguity set can be selected by using a finite sample guarantee principle~\citep{ref:mohajerin2018data,ref:yue2024geometric}, in which case the radius is chosen so that the ground truth lies in the ambiguity set with high probability. In practice, the cross-validation (CV) method is widely used for parameter tuning. For applications of the CV method in radius tuning in DRO, see the numerical experiment parts in~\cite{ref:mohajerin2018data, ref:yue2024geometric,ref:gao2023distributionally,ref:nguyen2022distributionally}. The numerical experiments show that, compared to the optimal radius suggested by the CV method, the finite sample guarantee principle tends to be more conservative.
In a recent paper,~\cite{ref:BLANCHET2019618} propose a customized radius tuning scheme for the Wasserstein distributionally robust precision matrix estimation model~\citep{ref:nguyen2022distributionally}. They select the optimal radius by minimizing the estimation loss of the DRO estimator induced by the radius. They demonstrate that the order of the radius proposed by their scheme aligns with the order suggested by the CV method, thereby validating the empirical findings and providing theoretical support for them.

\begin{table}[]
\begin{tabular}{|Sc|Sc|Sc|Sc|Sc|}
\hline
Estimator & \makecell{Estimator\\Type} & \makecell{Shrinkage\\Type} & \makecell{Shrinkage\\Target} & \makecell{Shrinkage Intensity} \\ \hline 
\cite{ref:LEDOIT2004365} &\makecell{Covariance Estimator}  &Linear  &$\frac{1}{p}\tr(\Sigma_0)I$  &\makecell{Controlled by\\ Combination Coefficient}  \\ \hline
\cite{ref:Ledoit2012} &\makecell{Covariance Estimator}  &Nonlinear  &N.A.  &N.A  \\ \hline
\cite{ref:yue2024geometric} &\makecell{Covariance Estimator}  &Nonlinear  &0  &\makecell{Controlled by\\Radius of Ambiguity Set}  \\ \hline
\cite{ref:nguyen2022distributionally} &\makecell{Precision Estimator}  &Nonlinear  &0  &\makecell{Controlled by\\Radius of Ambiguity Set}  \\ \hline
\makecell{SCOPE\\(this paper)} &\makecell{Covariance \&\\Precision Estimator} &Nonlinear  &$\sqrt{\frac{1}{\tau}}I$  &\makecell{Controlled by\\Radius of Ambiguity Set}  \\ \hline
\end{tabular}
\caption{A comparison of the shrinkage estimators of the covariance-precision matrix in the literature and the estimator proposed in this paper. The parameter $\tau$ in the shrinkage target of SCOPE is chosen by the statistician.}
\label{tab:comp-estimator}
\end{table}

\begin{table}[h]
\centering
\renewcommand{\arraystretch}{2}
\begin{tabular}{|l|c|c|c|}
\hline
\textbf{Divergence function} & \(D(\Sigma, {\wh \Sigma})\) & $d(a,b)$ &Domain of $D$\\ \hline
Kullback-Leibler  & 
\(\frac{1}{2} \left( \tr({\wh \Sigma}^{-1}\Sigma) - p + \log\det({\wh \Sigma}\Sigma^{-1}) \right)\) &$\frac{1}{2}\left(\frac{a}{b}-1-\log\frac{a}{b}\right)$ & 
\(\PD \times \PD\)  \\ \hline
Wasserstein  & 
\(\tr\left(\Sigma + {\wh \Sigma} - 2 (\Sigma^{\frac{1}{2}} {\wh \Sigma} \Sigma^{\frac{1}{2}})^{\frac{1}{2}}\right)\) &$a+b-2\sqrt{ab}$ & 
\(\PSD \times \PSD\) \\ \hline
Symmetrized Stein  & 
\(\frac{1}{2} \left( \tr({\wh \Sigma}^{-1} \Sigma) + \tr(\Sigma^{-1} {\wh \Sigma}) - 2p \right)\) &$\frac{1}{2}\left(\frac{b}{a}+\frac{a}{b}-2\right)$
&\(\PD \times \PD\) \\ \hline
Squared Frobenius & 
\(\tr \left( (\Sigma - {\wh \Sigma})^2 \right)\)& $(a-b)^2$ & 
\(\PSD \times \PSD\) \\ \hline
Weighted Frobenius &
\(\tr \left( (\Sigma - {\wh \Sigma}) {\wh \Sigma}^{-1} \right)\)& $\frac{(a-b)^2}{b}$ & 
\(\PSD \times \PD\) \\ \hline
\end{tabular}
\caption{Popular divergence functions and their generators discussed in this paper. See Assumption~\ref{ass:Spectral-divergence} for the definition of the generator of a divergence.}
\label{tab:divergences}
\end{table}

\noindent \textbf{Contributions.}
 We propose a moment-based distributionally robust optimization model that jointly estimates the covariance matrix and precision matrix. The proposed DRO model minimizes the worst-case sum of the Frobenius loss of the covariance estimator and weighted Stein's loss of the precision matrix estimator. The ambiguity set contains distributions whose covariance matrix is close to the nominal covariance matrix. The distance between matrices is measured via divergence functions over the space of positive semi-definite matrices.

 Because the precision matrix is the inverse of the covariance matrix, the estimation problem involves a bilinear constraint restricting the matrix product of the covariance and precision estimates to be the identity matrix. This bilinear constraint injects nonconvexities into the estimation problem, which is fundamentally different from the nonconvexity in earlier distributionally robust estimation frameworks.
The main contributions of this paper can be summarized as follows.
\begin{itemize}
    \item Under regularity assumptions on the divergence function, the DRO covariance-precision matrix estimator can be found by solving a convex optimization model, despite the nonconvex bilinear constraint. 
    
    \item Additionally, when the divergence function is a convex spectral divergence, we show that the DRO covariance-precision matrix estimator is a Stein-type nonlinear shrinkage estimator and admits a quasi-closed form.
    We provide an explicit expression of the shrinkage target of the estimator, parametrized by the weight coefficient of the objective function. By tuning the coefficient $\tau$ judiciously, we can ensure that the target is a well-conditioned matrix. Consequently, the spectral bias~\eqref{bias-2} of the estimator is corrected from two sides: both the large and the small eigenvalues are shrunk to the middle. As a result of the bias correction, the condition number of the estimator is also improved. By this spectral correction effect, we call the proposed estimator the \textbf{S}pectrum concentrated \textbf{CO}variance and \textbf{P}recision matrix \textbf{E}stimator (\textbf{SCOPE}).
    \item We provide explicit expressions of the estimator when the divergence function is a commonly-used convex spectral divergence, such as the Kullback-Leibler divergence, Wasserstein divergence, Symmetrized Stein divergence, Squared Frobenius, and Weighted Frobenius divergence. See Table~\ref{tab:divergences} for definitions of these divergences.
    \item We propose a parameter-tuning scheme that adjusts the shrinkage target and the shrinkage intensity by minimizing the violation of the optimality of the underlying true covariance $\Sigma_0$. The result shows that the optimal shrinkage target is a scalar matrix with the same norm as $\Sigma_0$, and the optimal order to the radius of the ambiguity set is $O(n^{-2})$. 
\end{itemize}

The remainder of the paper is organized as follows.
In Section~\ref{sec:prob-statement}, we propose our distributionally robust joint covariance-precision estimation model, demonstrate the difficulty of the model, and compare it with the distributionally robust covariance estimation model proposed by~\cite{ref:yue2024geometric}. In Section~\ref{sec:tractable-reform}, we show that under regularity assumptions on the divergence functions, the covariance-precision matrix estimator induced by the DRO model admits a quasi-closed form. We further demonstrate that the DRO model shrinks the estimator towards a well-conditioned target, thereby correcting the spectral bias from both sides. Specific covariance-precision matrix estimators induced by spectral convex divergence functions are listed in Table~\ref{tab:vphi-gamma}. In Section~\ref{sec:tuning}, we propose a practical parameter tuning scheme and adjust the shrinkage target and the shrinkage intensity accordingly. Finally, we report numerical results based on synthetic and real data in Section~\ref{sec:exp}.

\noindent \textbf{Notations.} By convention, we use $\R, \R_+, \R_{++}$, and $\overline{\R}=\R\cup \{+\infty\}$ to denote the entire real line, the set of non-negative real numbers, the set of strictly positive real numbers, and the extended real line. The $p$-dimensional Euclidean space and its subsets of non-negative real vectors and strictly positive real vectors are denoted by $\R^p, \R_+^p, \R_{++}^p$, respectively. The space of $p\times p$ real matrices is denoted by $\R^{p\times p}$, and the space of $p\times p$ symmetric matrices, the cone of positive semi-definite matrices, and the cone of strictly positive definite matrices are denoted by $\S^p, \PSD$, and $\PD$, respectively. The identity matrix is denoted by $I$. For a vector $x\in \R^p$, we use $x^{\uparrow}$ to denote the vector obtained by rearranging the entries of $x$ in non-decreasing order. The trace of a matrix $S\in \R^{p\times p}$ is defined as $\tr(S)\Let \sum_{i=1}^p S_{ii}$, and the inner product of two matrices $X,Y\in\R^{p\times p}$ is defined via $\la X, Y\ra \Let \tr(X^\transpose Y)$. The Frobenius norm of a matrix $S\in\R^{p\times p}$ is defined by $\|S\|_F\Let \sqrt{\la S,S\ra}$. By slight abuse of notation, we use $\diag(S)$ to denote the vector of diagonal entries of matrix $S\in \R^{p\times p}$, and $\diag(x)$ to denote the diagonal matrix whose diagonal is given by vector $x\in \R^p$. The domain of a non-negative function $f$ is defined as $\mathrm{dom}(f):=\{x:f(x)<\infty\}$.

\section{Distributionally Robust Joint Covariance-Precision Estimation}\label{sec:prob-statement}

Let $\xi$ be a random vector in $\R^p$ with mean zero and covariance matrix $\Sigma_0$. 
Let $\Pnom$ be a nominal distribution constructed from $n$ samples of $\xi$, and $\wh\Sigma\Let\E_{\Pnom}[\xi\xi^\transpose]$ be the nominal covariance matrix estimator. A typical choice of $\Pnom$ could be the empirical distribution of $n$ independent realizations of $\xi$, and $\wh\Sigma$ is the sample covariance matrix of the $n$ samples in this case. However, at this stage, we make no assumptions about the choice of $\Pnom$ and $\wh\Sigma$ to preserve the flexibility of our model.
For convenience in modeling,~\cite{ref:yue2024geometric} proposed to view $\wh \Sigma$ as the unique solution to the minimization problem:
\begin{equation}
    \wh \Sigma
    =\; \arg \min\limits_{\Sigma \in \PSD}  ~\|\Sigma\|_F^2 -2\la\Sigma,\wh\Sigma\ra \;=\; \arg \min\limits_{\Sigma \in \PSD} ~ \underbrace{\|\Sigma\|_F^2-2\E_{\wh{\mathbb{P}}}[\la \Sigma, \xi \xi^\transpose \ra ]}_{\Let \mathrm{FrobeniusLoss}(\Sigma, \Pnom)},
    \label{prob:jointly-estimate-model-cov}
\end{equation}
where the expectation is taken over $\xi$. The objective function of~\eqref{prob:jointly-estimate-model-cov} is termed the Frobenius loss in the variable $\Sigma$, evaluated under the distribution $\Pnom$. 
On the other hand, if the underlying distribution of $\xi$ is assumed to be a Gaussian distribution, the maximum likelihood estimator $\wh X$ for the precision matrix $\Sigma_0^{-1}$ is given by the unique minimizer (if it exists) of the minimization problem:
\begin{equation}
    \wh X= \arg\min\limits_{X \in \PD} ~ -\log\det X+\la X,\widehat{\Sigma}\ra \;=\; \arg\min\limits_{X \in \PD} ~ \underbrace{-\log\det X+\E_{\wh{\mathbb{P}}}[\la X,\xi \xi^\transpose\ra]}_{\Let \mathrm{SteinLoss}(X, \Pnom)}.
    \label{prob:jointly-estimate-model-prec}
\end{equation}
The objective function of~\eqref{prob:jointly-estimate-model-prec} is known as Stein's loss in the variable $X$~\citep{ref:james1961estimation}. The model~\eqref{prob:jointly-estimate-model-prec} can also be viewed as a Bregman divergence minimization problem that minimizes the divergence between $X$ and its true counterpart $\Sigma_0^{-1}$, and thus is also applicable for the estimation of non-Gaussian random variables~\citep{ref:ravikumar2011high}. A downside of problem \eqref{prob:jointly-estimate-model-prec} is that when $\covsa$ is singular, the minimizer does not exist, hence it requires that the sample covariance matrix $\covsa$ be nonsingular.

Motivated by problems~\eqref{prob:jointly-estimate-model-cov} and~\eqref{prob:jointly-estimate-model-prec}, we estimate the covariance matrix and the precision matrix simultaneously by solving the following problem:
\begin{equation}
    \begin{array}{cl}
    \min\limits_{\Sigma,X} &  \mathrm{SteinLoss}(X, \Pnom) + \frac{\tau}{2} \mathrm{FrobeniusLoss}(\Sigma, \Pnom) \\
    \text{s.t.} & \Sigma\in \PD,~X\in \PD,~X\Sigma=I.
    \end{array}
    \label{prob:jointly-estimate-model}
\end{equation}
The objective function is a linear combination of the Stein loss and the Frobenius loss under a non-negative weighting coefficient $\tau$, 
where $\tau$ balances the two losses. The constraint $X\Sigma = I$ ensures that the estimators of covariance and precision matrices are inverses of each other, and thus consistent with their definition. When $\widehat{\Sigma}$ is nonsingular, the optimal solution to \eqref{prob:jointly-estimate-model} is $(\Sigma\opt,X\opt)= (\widehat{\Sigma}, \widehat{\Sigma}^{-1})$. This recovers the classical covariance and inverse covariance estimators. 
However, if $\widehat{\Sigma}$ is singular (also known as rank deficient), problem~\eqref{prob:jointly-estimate-model} is unbounded below and there is no optimal solution.\footnote{See Lemma~\ref{lemma:unbounded} for a formal statement.}

Given the above argument, the joint estimation model~\eqref{prob:jointly-estimate-model} is not suitable for a high-dimensional setting, where the sample size $n$ is less than or comparable to the problem dimension $p$.
Moreover, the classical model~\eqref{prob:jointly-estimate-model} also suffers from distribution ambiguity: 
In data-driven problems, the available data are often limited, 
and the empirical distribution $\wh\P$ may not approximate the data-generating distribution $\P$ well. As a consequence, the out-of-sample performance is not reliable. 

To address these issues, we consider a distributionally robust estimation formulation motivated by model~\eqref{prob:jointly-estimate-model}. Consider a moment-based ambiguity set ``centered'' at the nominal distribution $\Pnom$ with radius $ \rho>0$:
\begin{equation} \label{eq:ambiguity}
    \mathcal{P}_ \rho(\wh{\P}) \Let \left\{\mathbb{Q}: \E_{\mathbb{Q}}[\xi] = 0, \E_{\mathbb{Q}}[\xi\xi^\transpose] = S, D\left(S, \E_{\wh{\P}}[\xi\xi^\transpose]\right)\leq  \rho\right\}.
\end{equation}
The set $\mathcal{P}_ \rho(\wh{\P})$ contains all the distributions with mean zero and second-moment matrix $S$ that is of a divergence at most $ \rho$ from the nominal covariance matrix estimator $\E_{\wh{\P}}[\xi\xi^\transpose] = \covsa$.
Here, we utilize a divergence function $D$ in the space of positive semi-definite matrices to measure the discrepancy between two matrices. Divergences are generalized distance functions that need not be symmetric or satisfy the triangle inequality. When we construct the ambiguity set $\mathcal{P}_ \rho(\wh{\P})$, the choice of $\rho$ reveals the conservativeness of the set: when we believe the nominal value is accurate, we shall set a smaller $\rho$, and if not, we shall set a larger $\rho$.

To robustify the empirical model~\eqref{prob:jointly-estimate-model} against all potential mismatches against distributions in $\mathcal{P}_\rho (\wh{\P})$, we define the pair of distributionally robust covariance-precision matrix estimators $(\Sigma\opt, X\opt)$ as the solution of a min-max problem
\begin{equation}
    (\Sigma\opt, X\opt) \Let \arg\min\limits_{\substack{\Sigma,X \in \PD \\ X \Sigma = I}}\;\max\limits_{\mathbb{Q}\in \mathcal{P}_ \rho(\wh{\P})}  -\log\det X+\E_{{\mathbb{Q}}}[\la X,\xi\xi^\transpose\ra] +\frac{1}{2}\tau \left(\|\Sigma\|_F^2-2\E_{{\mathbb{Q}}}[\la \Sigma, \xi \xi^\transpose\ra]\right). \tag{DRO}
    \label{prob:jointly-model-dro}
\end{equation}
The uniqueness of the solution is guaranteed under some moderate conditions on the divergence function $D$. We will formally prove this result in Theorem~\ref{thm:convex-reform}.
Since the objective function of the problem above depends linearly on $\E_{\mathbb{Q}}[\xi\xi^\transpose]$, we can reformulate the DRO as a robust optimization (RO) problem:
\begin{equation}\label{prob:robust-model}
    \displaystyle (\Sigma\opt, X\opt) = \arg\min_{\substack{\Sigma,X \in \PD \\ X \Sigma = I}} ~\max\limits_{S \in \PSD: D(S, \widehat{\Sigma})\leq  \rho}~\left\{ f(\Sigma, X, S) \triangleq -\log\det X+\la X,S\ra +\frac{1}{2}\tau \left(\|\Sigma\|_F^2-2\la \Sigma, S\ra\right) \right\}. \tag{RO}
\end{equation}
The objective function $f(\Sigma, X, S)$ is convex in $(\Sigma, X)$ and concave in $S$. 
Later, we will see that the weighting coefficient $\tau$ parameterizes the shrinkage target of the model, while the radius $\rho$ controls the shrinkage effect. We will discuss their choices in Sections~\ref{sec:tuning}.
In the remainder of this section, we discuss the main difficulty of~\eqref{prob:robust-model} and delineate the key differences between the DRO joint estimation problem~\eqref{prob:jointly-model-dro} and the existing formulations in the literature.

\noindent \textbf{Difficulty of~\eqref{prob:robust-model} 
and comparison with~\citet{ref:yue2024geometric}.} In~\citet[equation~(4)]{ref:yue2024geometric}, the authors propose a DRO formulation of the covariance estimation problem~\eqref{prob:jointly-estimate-model-cov} using the same ambiguity set $\mathcal{P}_ \rho(\wh{\P})$ as in~\eqref{eq:ambiguity}. Their model reduces to the following robust  optimization problem 
\begin{equation} \label{eq:yue}
    \min_{\Sigma \in \PSD } ~\max\limits_{S \in \PSD: D(S, \widehat{\Sigma}) \le  \rho}~ \left\{ g(\Sigma, S) \triangleq \|\Sigma\|_F^2 - 2 \la \Sigma,S\ra \right\},
\end{equation}
where the divergence $D$ is chosen from a family of (possibly nonconvex) divergence functions. 
The main difficulty in~\eqref{eq:yue} arises from the non-convexity of the divergence function $D$.
To obtain a tractable reformulation of~\eqref{eq:yue}, the authors exploit the structure of the divergence function and propose a (generalized) minimax theorem for geodesically convex-concave functions that is applicable to~\eqref{eq:yue}; see~\cite[theorem~3]{ref:yue2024geometric}. In contrast, the difficulty in problem~\eqref{prob:robust-model} arises from the bilinear constraint $X\Sigma = I$, which makes it difficult to justify whether exchanging the order of min-max will preserve the optimal value and optimal solution to the problem.  This bilinear term cannot be convexified by elementary operations. 
We consider two standard ways to eliminate the bilinear constraint.
\begin{enumerate}[leftmargin=5mm, label = (\roman*)]
    \item Eliminate $\Sigma$ by substituting it with $X^{-1}$. Consequently, problem~\eqref{prob:robust-model} becomes an optimization problem with variables $(X, S)$ only:
    \begin{equation}\label{prob:robust-model-X}
        \displaystyle \min_{X \in \PD} ~\max\limits_{S \in \PSD: D(S, \widehat{\Sigma})\leq  \rho}~\left\{ f_1(X, S) \triangleq -\log\det X+\la X,S\ra +\frac{1}{2}\tau \left(\|X^{-1}\|_F^2 - 2\la X^{-1}, S\ra\right) \right\}.
    \end{equation}
    The objective function $f_1(X, S)$ is, in general, not convex in  $X$. 
    The non-convexity arises from the term $\|X^{-1}\|_F^2 - 2\la X^{-1}, S\ra$. To see this, let us consider the 1-dimensional case and take $S=1$, this term becomes $h(x)  x^{-2}-2x^{-1}$. Since
    $h^{\prime\prime}(2)=-\frac{1}{8}$, then $h(x)$ is nonconvex in $x$. So for a large $\tau$, $f_1(X, S)$ is not convex in $X$. 
        Moreover, consider the inner maximization problem of~\eqref{prob:robust-model-X} in the variable $S$, i.e.,
    \[
        \max\limits_{S \in \PSD: D(S, \widehat{\Sigma})\leq  \rho} \la X, S\ra - \tau\la X^{-1}, S\ra \;=\; \max\limits_{S \in \PSD: D(S, \widehat{\Sigma})\leq  \rho} \la X-\tau X^{-1}, S\ra.
    \]
    This problem is not even a geodesically convex problem because the matrix coefficient $X-\tau X^{-1}$ can be indefinite for some $X\in \PSD$. So the {\it geodesically convex-concave} minimax theorem in~\cite[theorem 3]{ref:yue2024geometric} does not apply to~\eqref{prob:robust-model-X}.

    \item Alternatively, we can eliminate $X$ by substituting $X$ with $\Sigma^{-1}$. In this case, we obtain an optimization in only the $(\Sigma, S)$ variable:
    \begin{equation}\label{prob:robust-model-Sigma}
        \displaystyle \min_{\Sigma \in \PD} ~\max\limits_{S \in \PSD: D(S, \widehat{\Sigma})\leq  \rho}~\left\{ f_2(\Sigma, S) \triangleq \log\det \Sigma +\la \Sigma^{-1},S\ra +\frac{1}{2}\tau \left(\|\Sigma \|_F^2 - 2\la \Sigma, S\ra\right) \right\}.
    \end{equation}
    The objective function $f_2$, again, is neither convex-concave nor geodesically convex-concave, and the existing methods fail to derive the minimax property of~\eqref{prob:robust-model-Sigma}.
\end{enumerate} 
The discussions above show that the bilinear constraint has made the model intrinsically nonconvex. To ensure a tractable formulation of~\eqref{prob:robust-model}, we may 
have to restrict the choice of function $D$.
In the forthcoming discussions, we will concentrate on the case where the divergence $D$ is convex. It turns out that problem~\eqref{prob:robust-model} admits a convex reformulation under the convexity assumption of the divergence, see the details in the next section.

\section{Representation and Characterization of the Estimators}
\label{sec:tractable-reform}

In this section, we formally derive the semi-analytical form of the DRO joint covariance-precision estimator induced by~\eqref{prob:robust-model}. Toward that goal, Section~\ref{sec:convex-reform} shows that under a convexity assumption on the divergence $D$,~\eqref{prob:robust-model} reduces to a convex optimization problem, i.e.,~\eqref{prob:P-Mat}. In section~\ref{sec:assump}, we introduce the concept of spectral divergence, and show that when divergence $D$ is further assumed to be spectral divergence, the optimal solution to problem~\eqref{prob:P-Mat} admits a quasi-closed form. Finally, we derive the quasi-closed form of the distributionally robust covariance-precision matrix estimator $(\Sigma\opt,X\opt)$, and propose a computationally efficient framework for the estimator.

\subsection{Convex Reformulation}\label{sec:convex-reform}

Before delving into the details of the reformulation, we first note that to ensure a well-posed~\eqref{prob:robust-model} problem, we assume the following throughout the paper.
\begin{assumption}[Regularity of nominal]\label{ass:regularity-nominal}
    For the nominal covariance matrix estimator $\wh\Sigma$, it holds that $(\wh\Sigma,\wh\Sigma)\in\mathrm{dom}(D)$.
\end{assumption}
Assumption~\ref{ass:regularity-nominal} is to ensure well-definedness of  $D(\wh\Sigma,\wh\Sigma)$. We make this assumption as $D(\wh\Sigma,\wh\Sigma)=\infty$ for some divergence functions, such as Kullback-Leibler divergence, when $\wh\Sigma$ is singular. 
Moreover, the convex reformulation of~\eqref{prob:robust-model} relies on the following convexity assumption on the choice of divergence $D$.
\begin{assumption}[Convex divergence]\label{ass:convex-divergence}
    The divergence $D:\PSD\times \PSD\to \R_+$ satisfy:
\begin{enumerate}[label=(\roman*)]
        \item $D$ is non-negative, and satisfies the identity of indiscernibles, that is, for any $(X,Y)\in\mathrm{dom}(D)$, $D(X,Y)=0$ if and only if $X=Y$,
        \item $D(\cdot,\wh\Sigma)$ is convex, continuous, and differentiable,
        \item $D(\cdot, \wh\Sigma)$ is coercive, i.e., $D(X,\wh\Sigma)\to +\infty$ as $\|X\|_F \to \infty$.
    \end{enumerate}
\end{assumption}
Under Assumptions~\ref{ass:regularity-nominal} and~\ref{ass:convex-divergence}, it holds that $D(\wh\Sigma,\wh\Sigma)=0$, and thus the feasible set $\{\Sigma\in\PD: D(\Sigma, \widehat{\Sigma})\leq  \rho\}$ is non-empty.
The following theorem states that under these assumptions, the optimal solution to~\eqref{prob:robust-model} can be
derived by the optimal solution to the following convex optimization problem:
\begin{align}\tag{P-Mat}\label{prob:P-Mat}
\max_{\Sigma\in\PD: D(\Sigma, \widehat{\Sigma})\leq  \rho}~ \log\det \Sigma - \frac{1}{2}\tau \|\Sigma\|_F^2.
\end{align}

\begin{theorem}[Convex reformulation of \eqref{prob:robust-model}]
\label{thm:convex-reform}
    Let Assumptions~\ref{ass:regularity-nominal} and~\ref{ass:convex-divergence} hold. Then
    \begin{enumerate}[label=(\roman*)]
        \item \eqref{prob:P-Mat} admits a unique optimal solution $\Sigma\opt$,
        \item \eqref{prob:robust-model} admits a unique optimal solution pair, which can be represented $(\Sigma=\Sigma\opt, X=(\Sigma\opt)^{-1})$.
    \end{enumerate}
\end{theorem}

The theorem enables us to solve~\eqref{prob:robust-model} by solving~\eqref{prob:P-Mat}. In the next subsections, we discuss the construction of the optimal solution to~\eqref{prob:P-Mat}. 
First, we show in Section~\ref{sec:trivial-sol} that when the constraint of~\eqref{prob:P-Mat} is redundant, the optimal solution to~\eqref{prob:P-Mat} is a scaled identity matrix. Next, we demonstrate, in Section~\ref{sec:assump},
that when the constraint is active, the optimal solution to~\eqref{prob:P-Mat} admits a quasi-closed form. 
This gives us an efficient way to construct the estimator induced by~\eqref{prob:robust-model}, which is stated in the forthcoming Theorem~\ref{thm:close-form-dro}.

\subsection{A Trivial Solution}\label{sec:trivial-sol}

We first observe that when the radius $ \rho$ is sufficiently large, the constraint $D(\Sigma, \covsa) \le  \rho$ becomes redundant in~\eqref{prob:P-Mat} in the sense that the global maximizer of the objective function lies in the feasible set. In that case, the optimal solution is  $\Sigma\opt=\sqrt{\frac{1}{\tau}} I$. The next proposition states this.

\begin{proposition}[Unbinding constraint] \label{prop:unbinding}
    Let  $\rho$ be such that
    $D\left(\sqrt{\frac{1}{\tau}} I, \covsa \right)\leq  \rho$. Then $\Sigma\opt=\sqrt{\frac{1}{\tau}} I$ is the optimal solution to~\eqref{prob:P-Mat}.
\end{proposition}

Proposition~\ref{prop:unbinding} states that when $ \rho$ is sufficiently large, 
the optimal solution to~\eqref{prob:P-Mat} is a scaled identity matrix. The magnitude of this scaling depends on $\tau$, and we will see in Section~\ref{sec:nonlinear-shrink} that the scaled identity matrix can be viewed as the shrinkage target of the proposed covariance estimator. 

\subsection{Quasi-closed Form of the Estimators}\label{sec:assump}

We now turn to the construction of $\Sigma\opt$ when $D(\sqrt{\frac{1}{\tau}} I, \covsa ) \geq  \rho$, in which case the trivial solution in Proposition~\ref{prop:unbinding} is no longer valid. 
To ensure the quasi-closed form of the optimal solution to~\eqref{prob:P-Mat}, we make the following {\it spectral divergence} assumption on divergence $D$. The idea is from~\citet[Assumption 2]{ref:yue2024geometric}, which summarizes a common property that is satisfied by many widely-used divergences: the divergence between two positive semi-definite matrices is determined by their spectra. 
For examples of divergences that satisfy Assumption~\ref{ass:Spectral-divergence}, see Table~\ref{tab:divergences}.

\begin{assumption}[Spectral divergence]\label{ass:Spectral-divergence} 
The divergence $D:\PSD\times \PSD\to \R_+$ satisfies the following structural conditions:
        \begin{enumerate}[label=(\roman*)]
    \item (Orthogonal equivariance) For any $X,Y\in\PSD$ and $V\in{\cal O}(p)$, $D(X,Y)=D(VXV^\transpose, VYV^\transpose)$.
    
    \item (Spectrality) There exists a function $d:\R_+\times \R_+\to \R_+$, called the generator of $D$, such that
    \[
        D(\diag(x),\diag(y))=\sum_{i=1}^p d(x_i,y_i) \quad  \forall x,y\in \R_+^p
    \]
    and $d(a,b)$ is twice continuously differentiable in $a$ for every $b \geq 0$.  
    \item (Rearrangement property) For any $x,y\in \R_+^p$ and $V\in{\cal O}(p)$ we have
    \[
        D(V\diag(x^{\uparrow})V^\transpose, \diag(y^{\uparrow})) \geq D(\diag(x^{\uparrow}), \diag(y^{\uparrow})).
    \]
    If its left side is finite, this inequality becomes an equality if and only if $V\diag(x^{\uparrow})V^\transpose=\diag(x^{\uparrow})$.
\end{enumerate}

\end{assumption}

In the remainder of this paper, we refer to a divergence $D$ satisfying Assumptions~\ref{ass:convex-divergence} and~\ref{ass:Spectral-divergence} as a convex spectral divergence.
The following remark on the generator of convex spectral divergences collects some basic properties of $d$, and will facilitate the upcoming derivation in this paper. 

\begin{remark}[Properties of generator $d$]\label{remark:minimizer-of-b}
    Let $D$ be a divergence that satisfies Assumption~\ref{ass:Spectral-divergence}. Consider $(aI, bI) \in \mathrm{dom}(D)$. Then by Assumption~\ref{ass:Spectral-divergence}(ii), the domain of $d$ is \[\mathrm{dom}(d) = \{(a,b)\in \R_+^2: (aI, bI)\in \mathrm{dom}(D)\},\] and 
    $D(aI, bI) = p \times d(a,b)$.
    It implies that $d$ inherits non-negativity, continuity, identity of indiscernibles, and convexity from D.
    Then some basic calculus shows that $d(b,b)=0$, $\frac{\partial d}{\partial a}(b,b)=0$, and thus $b$ is the unique minimizer of the function $d(\cdot, b)$ for any $b>0$.
\end{remark}
\begin{remark}[Regularity of spectrum of nominal estimator]
\label{Rem:reg-Spec-esti}
    Let $0\leq \widehat{\lambda}_1\leq \widehat{\lambda}_2\leq  \ldots \leq \widehat{\lambda}_p$ be the eigenvalues of $\widehat{\Sigma}$ and $\widehat{\Sigma}=\widehat{V}\diag({\widehat{\lambda}})\widehat{V}^\transpose$ be its spectral decomposition. Under Assumptions~\ref{ass:regularity-nominal}--
    \ref{ass:Spectral-divergence}, it holds that $(\widehat{\lambda}_i, \widehat{\lambda}_i)\in \mathrm{dom}(d)$ for all $i=1, \ldots ,p$. To see this, note that 
    \[
        \sum_{i=1}^p d(\wh\lambda_i,\wh\lambda_i) = D(\diag(\wh\lambda),\diag(\wh\lambda)) = D(\widehat{V}\diag({\widehat{\lambda}})\widehat{V}^\transpose,\widehat{V}\diag({\widehat{\lambda}})\widehat{V}^\transpose) = D(\wh\Sigma, \wh\Sigma) = 0.
    \]
\end{remark}

Compared with~\citet[Assumption 3]{ref:yue2024geometric}, we do not require $\sum_{i=1}^p d(0, \widehat{\lambda}_i)>  \rho$, as the $\log\det$ term in the objective automatically excludes the trivial solution $\Sigma=\bm 0$. Instead, in the following assumption, we exclude the case where $\wh\Sigma$ is proportional to identity to ensure that the robustification is meaningful, since being proportional to identity is already robust enough.

\begin{assumption}[Regularity of nominal estimator]\label{ass:regularity} 
The nominal covariance matrix estimator satisfies $\wh\Sigma\neq \sigma I$ for any $\sigma\in \R_{++}$.
\end{assumption}

The last assumption is to avoid the relatively trivial case as in Proposition~\ref{prop:unbinding}, where we exploit the fact that under Assumption~\ref{ass:Spectral-divergence}(i), we have
\[
    D\left(\sqrt{\frac{1}{\tau}}I, \wh\Sigma\right) = D\left(\sqrt{\frac{1}{\tau}}I, \diag(\wh\lambda_1,\ldots,\wh\lambda_p)\right) = \sum_{i=1}^pd\left(\sqrt{\frac{1}{\tau}},\widehat{\lambda}_i\right).
\]

\begin{assumption}[Regularity of radius]\label{ass:reg-radius}
    The choice of $ \rho$ satisfies $0<  \rho \leq  \rho_{\max}\Let \sum_{i=1}^p d\left(\sqrt{\frac{1}{\tau}},\widehat{\lambda}_i\right)$. 
\end{assumption}

Under the above assumptions, we are able to construct the optimal solution to~\eqref{prob:P-Mat}, and thus the covariance-precision matrix estimator induced by~\eqref{prob:robust-model}, in a quasi-closed form. The construction involves a eigenvalue mapping corresponding to the generator $d$ of the divergence function, which is defined by
\begin{align}\label{eq:def-varphi}
\varphi(\tau, \gamma, b) \Let \text{the unique solution $a\opt> 0$ of the equation } 0 = \frac{1}{a\opt}-\tau a\opt -\gamma \frac{\partial d}{\partial a} (a\opt, b).
\end{align}
The following proposition ensures that the mapping is well-defined.
\begin{proposition}[Well-definedness of eigenvalue mapping $\varphi$]\label{prop:unique-varphi}
    Let Assumption~\ref{ass:convex-divergence} and~\ref{ass:Spectral-divergence} hold. Then for $\tau>0$, $\gamma\geq 0$ and $b\geq0$, the equation $\frac{1}{a}-\tau a-\gamma \frac{\partial d}{\partial a} (a, b)=0$ admits a unique solution in $\R_{++}$.
\end{proposition}

After specifying the mapping $\varphi$, we are ready to propose the following construction theorem formally.

\begin{theorem}[Construction of the covariance-precision matrix estimator]\label{thm:close-form-dro}
    Under Assumptions~\ref{ass:regularity-nominal}-\ref{ass:reg-radius},  the optimal solution to~\eqref{prob:robust-model} is 
    \bgeqn 
    \label{eq:Thm2-construction}
    \Sigma\opt=\wh V \Phi(\tau, \gamma\opt, \widehat{\lambda}) \wh V^\transpose,\quad  X\opt=(\Sigma\opt)^{-1},
    \edeqn 
    where 
$\wh V$ and $\wh \lambda$ are defined as in Remark~\ref{Rem:reg-Spec-esti} and
    $ 
     \Phi(\tau, \gamma\opt, \widehat{\lambda})\Let \diag(\varphi(\tau, \gamma\opt, \widehat{\lambda}_1), \ldots ,\varphi(\tau, \gamma\opt, \widehat{\lambda}_p))$,
    and $\gamma\opt$ is the unique non-negative root of the equation $\sum_{i=1}^p d(\varphi(\tau, \gamma\opt, \widehat{\lambda}_i), \widehat{\lambda}_i)- \rho=0$.
\end{theorem}
From the theorem, we can see that $\Sigma\opt(\tau, \rho)$ shares the same eigenbasis as the sample covariance matrix $\wh\Sigma$, i.e., it is a Stein-type rotation equivariant estimator.
The proof of Theorem~\ref{thm:close-form-dro} involves a number of intermediate results to be presented in Propositions~\ref{prop:equivalence-Mat-Vec}-\ref{prop:solve-prob-vector}. Figure~\ref{fig:proof-thm2} visualizes the flow of these results in the proof of Theorem~\ref{thm:close-form-dro}. We begin with Proposition~\ref{prop:equivalence-Mat-Vec}, which shows that~\eqref{prob:P-Mat} can be reduced to a convex problem that optimizes over all vectors in the non-negative orthant $\R_+^p$:
\begin{equation}\tag{P-Vec}\label{prob:vector}
    \begin{array}{cl}
    \displaystyle\max_{s \in \R_+^p}&\displaystyle \sum_{i=1}^p \log s_i - \frac{1}{2}\tau \sum_{i=1}^p s_i^2\nonumber\\
    \st &\displaystyle \sum_{i=1}^p d(s_i, \widehat{\lambda}_i)\leq  \rho. 
    \end{array}
\end{equation}

\begin{figure}[ht]
\centering
\begin{tikzpicture}[
    node distance=2.5cm and 3.5cm,
    box/.style={draw, rounded corners, thick, align=center, minimum width=4cm, minimum height=1.2cm},
    every edge quotes/.style={auto, font=\footnotesize},
    >=Stealth
]

\node[box] (foc) {Quasi-closed form $s_i\opt=\varphi(\tau,\gamma\opt,\wh\lambda_i)$ with\\$\gamma\opt$ satisfies $\sum_{i=1}^p d(\varphi(\tau,\gamma\opt,\wh\lambda_i),\wh\lambda_i)- \rho=0$};
\node[box, above=of foc] (pvec) {Optimal solution $s\opt$ of~\eqref{prob:vector}};
\node[box, above=of pvec] (pmat) {Optimal solution $\Sigma\opt$ of~\eqref{prob:P-Mat}};
\node[box, above=of pmat] (probro) {Optimal solution pair $(\Sigma\opt, X\opt)$ of~\eqref{prob:robust-model}};

\draw[<->] (pmat) -- (pvec)
    node[midway, left] {Proposition~\ref{prop:equivalence-Mat-Vec}}
    node[midway, right,align=center] {$\Sigma\opt=\wh V\diag(s\opt) \wh V^\transpose$\\$s_i\opt=\lambda_i(\Sigma\opt)~\forall i$};
\draw[->] (foc) --  (pvec)
    node[midway, right] {$s_i\opt=\varphi(\tau,\gamma\opt,\wh\lambda_i)$}
    node[midway, left] {Proposition~\ref{prop:solve-prob-vector}}
    ;
\draw[<->] (pmat) -- (probro)
    node[midway, left] {Theorem~\ref{thm:convex-reform}}
    node[midway, right] {$(\Sigma=\Sigma\opt,X=(\Sigma\opt)^{-1})$};
;

\end{tikzpicture}
\caption{Proof structure of Theorem~\ref{thm:close-form-dro}. Theorem~\ref{thm:convex-reform} states that~\eqref{prob:robust-model} reduces to a convex optimization problem~\eqref{prob:P-Mat}. Proposition~\ref{prop:equivalence-Mat-Vec} shows that the optimal solution of~\eqref{prob:P-Mat} can be constructed by solving~\eqref{prob:vector}, which is over a vector space. Proposition~\ref{prop:solve-prob-vector} reveals that the optimal solution to~\eqref{prob:vector} admits a quasi-closed form, which is specified by eigenvalue mapping $\varphi$.}
\label{fig:proof-thm2}
\end{figure}
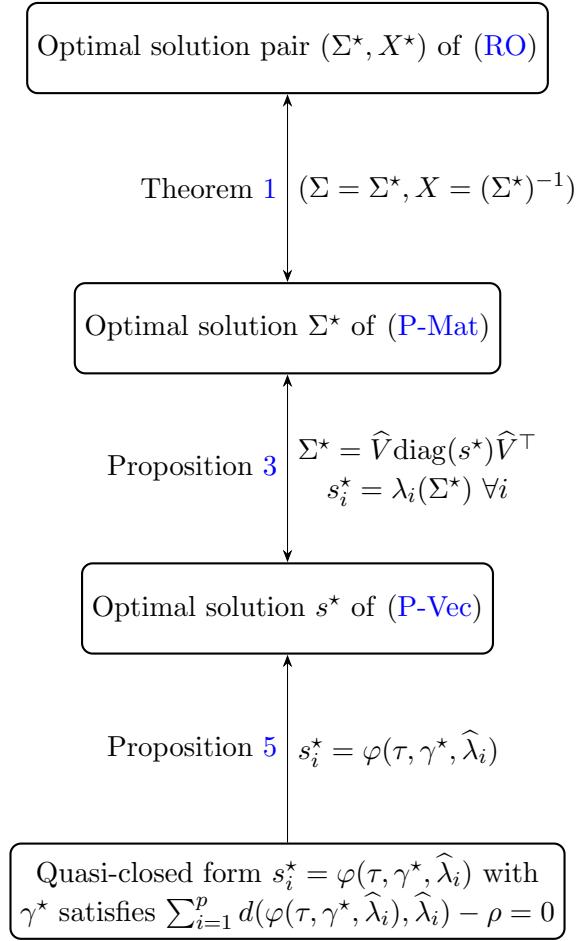

\begin{proposition}[Equivalence of \eqref{prob:P-Mat} and \eqref{prob:vector}]\label{prop:equivalence-Mat-Vec}
    Let Assumptions~\ref{ass:regularity-nominal}-\ref{ass:Spectral-divergence} hold. Then problem~\eqref{prob:P-Mat} is equivalent to~\eqref{prob:vector} in the following sense:
    \begin{enumerate}[label=(\roman*)]
        \item\label{prop:equivalence-Mat-Vec-1} If $s\opt$ solves~\eqref{prob:vector}, then $\widehat{V}\diag(s\opt)\widehat{V}^\transpose$ solves~\eqref{prob:P-Mat}.
        \item\label{prop:equivalence-Mat-Vec-2} If $\Sigma\opt$ solves~\eqref{prob:P-Mat}, then the vector of its eigenvalues $\lambda(\Sigma\opt)$ solves~\eqref{prob:vector}.
    \end{enumerate}
\end{proposition}

Next, we discuss how to solve~\eqref{prob:vector}. The  representation of the solution of~\eqref{prob:vector} 
involves  the eigenvalue mapping $\varphi$ defined in \eqref{eq:def-varphi}.
We begin with the Lagrange function of~\eqref{prob:vector}
\begin{align*}
    L(\gamma, s) &\;=\; \sum_{i=1}^p \log s_i - \frac{1}{2}\tau\sum_{i=1}^p s_i^2 -\gamma \left(\sum_{i=1}^p d(s_i,\widehat{\lambda}_i)- \rho\right)\\
    &\;=\; \sum_{i=1}^p \left(\log s_i - \frac{1}{2}\tau s_i^2 -\gamma d(s_i,\wh{\lambda}_i)\right) - \gamma \rho,
\end{align*}
where $\gamma \geq 0$ is the Lagrange multiplier associated with the constraint $\sum_{i=1}^p d(s_i, \widehat{\lambda}_i)\leq  \rho$. Then for fixed $\gamma$, the eigenvalue mapping $\varphi(\tau, \gamma, b)$ can be viewed as the minimizer of the summation term
\begin{align*}
    \psi(\tau,\gamma,a,b) = \log a - \frac{1}{2}\tau a^2 - \gamma d(a,b).
\end{align*}
Taking the derivative of $\psi$ with respect to $a$ and setting it to $0$, we obtain the definition of $\varphi$. The following proposition~\ref{prop:pro-of-varphi} collects some properties of the mapping $\varphi$ that will facilitate the upcoming discussion and analysis in this paper.

\begin{proposition}[Properties of eigenvalue mapping $\varphi$]\label{prop:pro-of-varphi}
If Assumptions~\ref{ass:regularity-nominal}-\ref{ass:Spectral-divergence} hold, then it follows that
\begin{enumerate}[label = (\roman*)]
    \item Let $\varphi_{\tau, b}(\gamma) \Let \varphi(\tau, \gamma, b)$. If $\tau>0$, $b>0$, then for any $\gamma \geq 0$:
    \begin{align*}
        \left\{\begin{matrix}
            b< \varphi_{\tau, b}(\gamma)\leq 1/\sqrt{\tau},\;\;\text{for}\;\;b< 1/\sqrt{\tau},\\
            1/\sqrt{\tau}\leq \varphi_{\tau, b}(\gamma)< b,\;\;\text{for}\;\;b> 1/\sqrt{\tau},\\
            \varphi_{\tau, b}(\gamma) = 1/\sqrt{\tau},\;\;\text{for}\;\;b= 1/\sqrt{\tau},
        \end{matrix}\right.
    \end{align*}
    
    \item If $\tau>0$ and $b\geq0$, then $\varphi_{\tau, b}(\gamma)$ is continuous, strictly monotonic and differentiable on $\R_{+}$. Specifically, $\varphi_{\tau,b}(\cdot)$ is strictly increasing if $b>1/\sqrt{\tau}$ and strictly decreasing if $b<1/\sqrt{\tau}$,
    \item If $\tau>0$ and $b\geq0$, then $\varphi_{\tau,b}(0) = 1/\sqrt{\tau}$ and $\lim_{\gamma\to\infty} \varphi_{\tau,b}(\gamma) = b$,
    \item If $\tau>0$ and $\gamma\geq0$, then $\phi(b)\Let \frac{\varphi(\tau,\gamma,b)}{b}$ is non-increasing on $b$.
\end{enumerate}
\end{proposition}

The following proposition shows that the optimal solution to~\eqref{prob:vector} admits a quasi-closed form, given that the eigenvalue mapping $\varphi$ is known. The result shows that $\varphi$ can be viewed as the eigenvalue mapping that distorts the empirical eigenvalues $\wh{\lambda}_1,\ldots,\wh{\lambda}_p$.

\begin{proposition}[Solution of~\eqref{prob:vector}]\label{prop:solve-prob-vector}
    If Assumptions~\ref{ass:regularity-nominal}-\ref{ass:reg-radius} hold, then~\eqref{prob:vector} admits a unique optimal solution $s\opt$ and $s_i\opt=\varphi(\tau, \gamma\opt, \widehat{\lambda}_i),i=1,\ldots,p$, where $\gamma\opt$ is the unique solution of the nonlinear equation $\sum_{i=1}^p d(\varphi(\tau, \gamma\opt, \widehat{\lambda}_i), \widehat{\lambda}_i)- \rho=0$.
\end{proposition}

\begin{remark}[Verification of the construction when $\rho=\rho_{\max}$]
    In the case where $ \rho =  \rho_{\max}$, we note that the characterization $\Sigma\opt=\widehat{V}\Phi(\tau, \gamma\opt, \widehat{\lambda})\widehat{V}^\transpose$ given by Theorem~\ref{thm:close-form-dro} is compatible with Proposition~\ref{prop:unbinding}. To see this, recall that 
    $D\left(\sqrt{\frac{1}{\tau}}I, \wh \Sigma\right)=\sum_{i=1}^pd\left(\sqrt{\frac{1}{\tau}},\widehat{\lambda}_i\right)= \rho_{\max}$. Then by Proposition~\ref{prop:pro-of-varphi}(iii), it holds that $\gamma\opt=0$, $\varphi(\tau, \gamma\opt, \widehat{\lambda}_i) = \sqrt{\frac{1}{\tau}}, i=1,\ldots,p$ solve the equation $\sum_{i=1}^p d(\varphi(\tau, \gamma\opt, \widehat{\lambda}_i), \widehat{\lambda}_i)- \rho_{\max}=0$.
    So the optimal solution characterized by Theorem~\ref{thm:close-form-dro} is $\Sigma\opt = \sqrt{\frac{1}{\tau}}\wh V \wh V^\transpose = \sqrt{\frac{1}{\tau}}I$, and it agrees with the result stated in Proposition~\ref{prop:unbinding}. 
\end{remark}

To construct the optimal solution $s\opt$ of~\eqref{prob:vector}, it remains to show that the nonlinear equation $\sum_{i=1}^p d(\varphi(\tau, \gamma\opt, \widehat{\lambda}_i), \widehat{\lambda}_i)- \rho=0$ can be solved efficiently. Define $F_{\wh\lambda}:\R_+ \to\R$ by $F_{\wh\lambda}(\gamma)\Let\sum_{i=1}^p d(\varphi(\tau, \gamma, \widehat{\lambda}_i), \widehat{\lambda}_i)$. The following proposition reveals that $F_{\wh\lambda}(\gamma)= \rho$ admits a unique positive root, and the equation can be solved efficiently by bisection or Newton's method.
\begin{proposition}[Differentiable and strictly decreasing $F_{\wh\lambda}$]\label{prop:property-of-F}
    If Assumptions~\ref{ass:regularity-nominal}-~\ref{ass:regularity} hold, then the function $F_{\wh\lambda}$ is differentiable and strictly decreasing over $\R_+$. If in addition, Assumption~\ref{ass:reg-radius} holds, then $F_{\wh\lambda}(0) =  \rho_{\max} \geq  \rho$ and $\lim_{\gamma\to\infty} F_{\wh\lambda}(\gamma) = 0 \leq  \rho$.
\end{proposition}

To sum up the derivation so far, Theorem~\ref{thm:convex-reform} reduces the robust problem~\eqref{prob:robust-model} to a more tractable formulation~\eqref{prob:P-Mat}. Proposition~\ref{prop:equivalence-Mat-Vec} shows that the unique optimal solution to~\eqref{prob:P-Mat} can be constructed from $s\opt$, the optimal solution to~\eqref{prob:vector}. Proposition~\ref{prop:solve-prob-vector} provides a quasi-closed form of~\eqref{prob:vector} parameterized by the dual variable $\gamma\opt$, which can be characterized by solving a univariate equation. Proposition~\ref{prop:property-of-F} ensures that the equation can be solved efficiently. Finally, taking all these together proves Theorem~\ref{thm:close-form-dro}, the construction of distributionally robust covariance estimators.

\subsection{Nonlinear Shrinkage and Spectral Bias Correction}\label{sec:nonlinear-shrink}

In this section, we show that the covariance matrix estimator $\Sigma\opt$ introduced in Theorem~\ref{thm:close-form-dro} can be interpreted as a nonlinear shrinkage estimator with $\sqrt{\frac{1}{\tau}}I$ as the shrinkage target and $\rho$ as the shrinkage intensity. In Theorem~\ref{thm:close-form-dro}, the estimator is characterized by the Lagrange multiplier $\gamma\opt$. To see how the estimator can be parameterized by the radius $\rho$ when $\rho$ varies in $(0, \rho_{\max}]$, we note that $\gamma\opt$ can be viewed as a function of $\rho$ in the sense of 
\begin{align}\label{def-gamma-opt}
    \gamma_{\wh\lambda}\opt(\rho) \Let \text{the unique solution $\gamma\opt\geq 0$ of the equation } F_{\wh\lambda}(\gamma\opt) = \rho,
\end{align}
where $F_{\wh\lambda}$ is defined as in Proposition~\ref{prop:property-of-F}. By Proposition~\ref{prop:property-of-F}, $\gamma_{\wh\lambda}\opt(\rho_{\max}) = 0$, $\lim_{\rho\downarrow 0}\gamma_{\wh\lambda}\opt(\rho) = \infty$, and $\gamma_{\wh\lambda}\opt(\rho)$ is  strictly decreasing in $\rho$.
Moreover, it follows by \eqref{eq:Thm2-construction} in Theorem~\ref{thm:close-form-dro} that
\begin{align}\label{eq:construct-Sigma-opt}
    \Sigma\opt(\tau,  \rho)=\widehat{V}\Phi(\tau, \gamma_{\wh\lambda}\opt( \rho), \widehat{\lambda})\widehat{V}^\transpose\Let\wh V \diag(\varphi(\tau, \gamma_{\wh\lambda}\opt( \rho), \widehat{\lambda}_1), \ldots ,\varphi(\tau, \gamma_{\wh\lambda}\opt( \rho), \widehat{\lambda}_p))\wh V^\transpose,  \rho\in (0,  \rho_{\max}].
\end{align}
Here we write $\Sigma\opt(\tau,\rho)$ to emphasize its dependence on $\tau$ and $\rho$. Consequently, we can show that $\Sigma\opt(\tau,  \rho)$ is a nonlinear shrinkage estimator, and the spectral bias of the estimator is corrected. The next theorem states these.

\begin{theorem}[Nonlinear shrinkage and spectral bias correction]\label{thm:nonlinear-shrinkage}
    Let Assumptions~\ref{ass:regularity-nominal}-\ref{ass:regularity} hold. Then for any $\tau>0$, $\Sigma\opt(\tau, \rho)$ is continuous in $\rho$ over $(0,  \rho_{\max}]$ and  $\lim_{ \rho\downarrow 0}\Sigma\opt(\tau, \rho)=\wh\Sigma$, $\Sigma\opt(\tau, \rho_{\max})=\sqrt{\frac{1}{\tau}}I$. Moreover, the following assertions hold. 
    \begin{enumerate}[label=(\roman*)]
        \item if $\wh\lambda_i < \sqrt{\frac{1}{\tau}}$, then $\wh\lambda_i < \lambda_i(\Sigma\opt(\tau,\rho))$ for any $\rho\in (0,\rho_{\max}]$,
        \item if $\wh\lambda_i > \sqrt{\frac{1}{\tau}}$, then $\wh\lambda_i > \lambda_i(\Sigma\opt(\tau,\rho))$ for any $\rho\in (0, \rho_{\max}]$,
        \item $\kappa(\Sigma\opt(\tau, \rho)) < \kappa(\wh\Sigma)$ for any $\rho\in(0, \rho_{\max}]$, and $\kappa(\Sigma\opt(\tau, \rho))$ is strictly decreasing on $ \rho$.
    \end{enumerate}
\end{theorem}
Theorem~\ref{thm:nonlinear-shrinkage} ensures that all the eigenvalues are shrunk towards $\sqrt{\frac{1}{\tau}}$, i.e., $\sqrt{\frac{1}{\tau}}I$ is the shrinkage target of $\Sigma\opt(\tau, \rho)$. The choice of $\tau$ balances the shrinkage effects stemming from two different sources: the Stein loss and the Frobenius loss, as modeled in~\eqref{prob:jointly-model-dro}. If $\tau$ is set large, then the shrinkage effect 
from the Frobenius loss is significant, and the estimation tends to shrink all the eigenvalues of $\Sigma\opt(\tau, \rho)$ to $0$. If $\tau$ is set close to $0$, on the other hand, then the shrinkage effect is dominated by the effect 
caused by the Stein loss, and thus the eigenvalues of the inverse of $\Sigma\opt(\tau, \rho)$, i.e., the precision matrix estimator, are nearly shrunk to $0$. 
If we choose $\tau$ suitably so that $\sqrt{\frac{1}{\tau}}$ lying between interval $(\wh\lambda_1, \wh\lambda_p)$, then the shrinkage effects are balanced, the spectral bias~\eqref{bias-2} is corrected from both sides, and thereby the condition number of $\Sigma\opt(\tau, \rho)$ is improved strictly. 
Finally, the radius $ \rho$ controls the shrinkage intensity. When $ \rho$ is large, the eigenvalues of $\Sigma\opt(\tau, \rho)$ are close to $\sqrt{\frac{1}{\tau}}$, and thus condition number of $\Sigma\opt(\tau, \rho)$ is close to one. When $ \rho$ is small, the shrinkage effect is weak and $\Sigma\opt(\tau, \rho)$ is close to the nominal matrix $\wh\Sigma$.

\subsection{Covariance-Precision Matrix Estimator Induced by Spectral Convex Divergences}\label{sec:SCOPE-example}

In this section, we concretize the previous results by taking the specific forms of convex spectral divergence $D$ in Table~\ref{tab:divergences}. These divergences are widely used in statistics, machine learning, and engineering. For a more detailed review of these divergences, we refer readers to~\citet[Section 4]{ref:yue2024geometric}.

\begin{table}[h]
\centering
\renewcommand{\arraystretch}{2}
\begin{tabular}{|l|c|c|}
\hline
\textbf{Divergence function} & Eigenvalue mapping \(\varphi(\tau,\gamma,b)\) & Upper bound of $\gamma\opt$ \\ \hline
Kullback-Leibler  & $\frac{-\gamma+\sqrt{\gamma^2+8\tau(2+\gamma)b^2}}{4\tau b}$
 & $\max\left\{\frac{2p}{ \rho},\frac{2\tau \wh\lambda_p^2+1}{e^{ \rho/p}-1}\right\}$
 \\ \hline
Wasserstein  & \makecell{unique positive root $a$ of \\$\tau a^2+\gamma a - \gamma\sqrt{b}\sqrt{a}-1=0$}
 & $\max\left\{\frac{\eta_1+\sqrt{\eta_1^2+16\eta_1\tau^{2.5}}}{8\tau^2},\frac{\eta_2+\sqrt{\eta_2^2+16\eta_2\tau^{2.5}}}{8\tau^2}\right\}$
 \\ \hline
Symmetrized Stein  & \makecell{unique positive root $a$ of \\$2\tau b a^3 + \gamma a^2 - 2b a-\gamma b^2=0$}
 &$\sqrt{\max\left\{\frac{p\wh\lambda_p^2}{2 \rho \wh\lambda_1^4},\frac{p\wh\lambda_p^2(1-\tau \wh\lambda_p^2)^2}{2 \rho \wh\lambda_1^4}\right\}}$
 \\ \hline
Squared Frobenius & $\frac{\gamma b}{\tau+2\gamma}+\frac{\sqrt{\gamma^2 b^2 + \tau+2\gamma}}{\tau+2\gamma}$
 & \makecell{$\frac{p}{4\rho}+\sqrt{\frac{p^2}{16\rho^2}+\tau\frac{p}{4\rho}}\;\vee$\\ $\frac{p(1-\tau\wh\lambda_p^2)^2}{4\rho}+\sqrt{\frac{p^2(1-\tau\wh\lambda_p^2)^4}{16\rho^2}+\tau\frac{p(1-\tau\wh\lambda_p^2)^2}{4\rho}}$}
 \\ \hline
Weighted Frobenius & $\frac{\gamma b}{\tau b+2\gamma}+\frac{\sqrt{\gamma^2 b^2+b(\tau b+2\gamma)}}{\tau b+2\gamma}$
 & $\sqrt{\max\left\{\frac{p\wh\lambda_p}{4 \rho\wh\lambda_1^2},\frac{p\wh\lambda_p(1-\tau \wh\lambda_p^2)^2}{4 \rho\wh\lambda_1^2}\right\}}$
 \\ \hline
\end{tabular}
\caption{Eigenvalue mappings and upper bounds of dual variables for the divergence functions in Table~\ref{tab:divergences}. Here, $\eta_1 = \frac{p}{\rho}$, $\eta_2 = \frac{p(1-\tau\wh\lambda_p^2)^2}{\rho}$, and $a \vee b \Let \max\{a,b\}$. }
\label{tab:vphi-gamma}
\end{table}

The following proposition shows that the setting of Theorem~\ref{thm:close-form-dro} covers all the divergences in Table~\ref{tab:divergences}. The proof is similar to that of \citet[Theorem 2]{ref:yue2024geometric}, we skip the details.

\begin{proposition}[Convex spectral divergence functions]\label{thm:convex-divergence} All divergence functions listed in Table~\ref{tab:divergences} satisfy Assumptions~\ref{ass:convex-divergence} and~\ref{ass:Spectral-divergence}.
\end{proposition}

Taking $D$ as any divergence function in Table~\ref{tab:divergences}, we can construct the optimal solution to~\eqref{prob:robust-model} by Theorem~\ref{thm:close-form-dro}. The key steps are (i) to determine the mapping of eigenvalues $\varphi$ as a function of $(\tau,\gamma,b)$, and (ii) to determine the dual variable $\gamma\opt$ that solves the equation 
\[
\sum_{i=1}^p d(\varphi(\tau, \gamma\opt, \widehat{\lambda}_i), \widehat{\lambda}_i)- \rho=0.
\]
Proposition~\ref{prop:vphi-gamma} collects the forms of eigenvalue mapping $\varphi$ and gives the suggested upper bound of $\gamma\opt$ so that we can use bisection to find it efficiently for all divergences in Table~\ref{tab:divergences}.

\begin{proposition}[Estimator induced by convex divergences]\label{prop:vphi-gamma}
    If Assumptions~\ref{ass:regularity-nominal},~\ref{ass:regularity} and~\ref{ass:reg-radius} hold, and $\tau$ is taken such that $\wh\lambda_1<\sqrt{\frac{1}{\tau}} < \wh\lambda_p$, then the eigenvalue mappings and upper bounds of dual variable $\gamma\opt$ 
        corresponding to the divergence functions listed in Table~\ref{tab:divergences} 
        can be established as 
        in columns 2-3 of
        Table~\ref{tab:vphi-gamma}.
\end{proposition}

In the case of Kullback-Leibler, Squared Frobenius, and Weighted Frobenius, equation~\eqref{eq:def-varphi} reduces to a quadratic equation, and thus its root admits a simple expression. For the Wasserstein divergence and Symmetrized Stein divergence, the equation can only be reformulated as a cubic equation or a quartic equation. The expressions of the roots are complicated and not useful even for practical implementations. We suggest using a numerical method to find the root. Note that $\lim_{\gamma\to\infty}\varphi(\tau,\gamma,b)=b$ by Proposition~\ref{prop:pro-of-varphi}(iii). Then, for Newton's method, $b$ itself can be a good starting point when $\gamma$ is large.  For bisection, Proposition~\ref{prop:pro-of-varphi}(i) ensures that the root always falls between $\sqrt{\frac{1}{\tau}}$ and $b$, which provides the bisection interval.

\subsection{Consistency and Finite-sample Performance Guarantee}

In this section, we demonstrate that the covariance-precision matrix estimator is consistent and possesses a finite-sample performance guarantee. To emphasize the dependence on sample size $n$, we let the nominal estimator and the radius in~\eqref{prob:robust-model} be chosen as $\wh\Sigma=\wh\Sigma_n$ and $ \rho= \rho_n$, where $\wh\Sigma_n$ is any covariance estimator constructed from $n$ i.i.d.~samples, and $ \rho_n$ is a non-negative radius that may depend on sample size $n$. The corresponding nonlinear shrinkage covariance matrix estimator in this case is denoted by $\Sigma\opt_n(\tau, \rho_n)$. In the following proposition, we recall that an estimator of $\Sigma_0$ is strongly consistent if it converges to $\Sigma_0$ almost surely as $n$ tends to infinity.

\begin{proposition}[Consistency]\label{prop:consistency}
    Suppose that Assumptions~\ref{ass:Spectral-divergence},~\ref{ass:regularity} and~\ref{ass:reg-radius} hold. If $\wh\Sigma_n$ is a strongly consistent estimator of $\Sigma_0$ and $ \rho_n$ converges to $0$ as $n$ tends to infinity, then $\Sigma\opt_n(\tau, \rho_n)$ is strongly consistent for any $\tau>0$.
\end{proposition}

Proposition~\ref{prop:consistency} states that if the nominal covariance estimator is consistent, the covariance-precision matrix estimator inherits the consistency as long as the radius $ \rho_n$ shrinks to $0$ with $n$. A typical choice for a consistent nominal covariance estimator is the sample covariance matrix. In the next proposition, we recall the finite-sample performance guarantees of the ambiguity set $\mathcal{B}_{ \rho}(\wh\Sigma_n)=\{S\in\PD:D(S,\wh\Sigma_n)\leq  \rho\}$ in model~\eqref{prob:robust-model} from~\citet[Proposition 8]{ref:yue2024geometric} for completeness. Note that a probability distribution $\mathbb{P}$ is sub-Gaussian if there exists $\sigma^2\geq 0$ with $\E_{\P}[\exp(z^\transpose\xi)]\leq \exp(\frac{1}{2}\sigma^2\|z\|_2^2)$ for every $z\in\R^p$.

\begin{proposition}[Finite-sample performance guarantee,{~\cite[Proposition 8]{ref:yue2024geometric}}]~\label{prop:finite-guarantee}
    Suppose that $\P$ is sub-Gaussian with covariance matrix $\Sigma_0\in\PSD$, and let $\wh\Sigma_n$ be the sample covariance matrix of $n$ i.i.d.~samples from $\mathbb{P}$. For any divergence function $D$ from Table~\ref{tab:divergences} and $\eta\in (0,1)$, there exist $n_{\min}(\eta)=O(\log\eta^{-1})$ and $ \rho_{\min}(n,\eta) = O(n^{-\frac{1}{2}}(\log\eta^{-1})^{\frac{1}{2}})$ that may depend on $\P$ such that $\P^n\left(\Sigma_0\in\mathbb{B}_ \rho(\wh\Sigma_n)\right)\geq 1-\eta$ for all $n\geq n_{\min}(\eta)$ and $ \rho\geq  \rho_{min}(n,\eta)$.
    
\end{proposition}
Proposition~\ref{prop:finite-guarantee} provides a potential guideline for us to set the radius of the ambiguity set. Let the nominal estimator be the sample covariance matrix $\wh\Sigma_n$. The proposition ensures that there exist constants $c_1,c_2>0$ such that if we choose $ \rho_n \geq c_1 n^{-\frac{1}{2}}$, then with probability at least $1-\exp(-c_2 n)$, the optimal value of the robust problem~\eqref{prob:robust-model} is a upper bound on the actual estimation loss $f(\Sigma_0, \Sigma\opt(\tau, \rho_n), \Sigma\opt(\tau, \rho_n)^{-1})$. The argument of finite-sample performance guarantee has been widely discussed in the DRO literature; see the discussion in the paper of Wasserstein distributionally robust precision matrix estimation model~\citep{ref:nguyen2022distributionally}. In a recent follow-up paper, however,~\cite{ref:BLANCHET2019618} point out that an $n^{-\frac{1}{2}}$-ordered radius could be too conservative and does not agree with the empirical findings in~\cite{ref:nguyen2022distributionally}. They propose a customized radius tuning scheme by minimizing the estimation loss of the precision matrix estimator induced by the radius and show that the asymptotically optimal choice of the radius is of the order $n^{-1}$. Inspired by their theoretical findings, in the next section, we propose a parameter tuning framework for our covariance-precision matrix estimator and show that the order of the optimal choice of radius is also more optimistic than $n^{-\frac{1}{2}}$ in our case.

\section{Optimal Shrinkage Target and Intensity}\label{sec:tuning}

In this section, we propose a parameter tuning scheme for the nonlinear shrinkage covariance matrix estimator $\Sigma\opt(\tau, \rho)$. The parameter tuning can be viewed as deciding the shrinkage target $\sqrt\frac{1}{\tau} I$ and shrinkage intensity $ \rho$ such that $\Sigma\opt(\tau, \rho)$ is close to the true covariance $\Sigma_0$ under some measure such as the Frobenius norm. \cite{ref:LEDOIT2004365} used this idea to construct the optimal linear shrinkage estimator:
\begin{align*}
\Sigma^{LW}(\nu, t)\Let t\nu I + (1- t) \wh\Sigma, 
\end{align*}
where the authors choose the shrinkage target $\nu I$ and intensity $ t$ of the linear shrinkage estimator by solving following problem:
\begin{align}\label{prob:optimal-para-LW}
    \min_{\nu \in \R_+,  t \in [0,1]}\E_n[\|\Sigma^{LW}(\nu, t)-\Sigma_0\|_F^2].
\end{align}
Here $\E_n$ denotes the expectation with respect to the $n$-product measures of the data-generating distribution.
To solve this problem, the authors show that it is equivalent to take two steps: (i) find the optimal target by 
\begin{align}
    \nu\opt=\arg\min_{\nu\in\R_+}\|\Sigma_0-\nu I\|_F,
\end{align}
and (ii) find the optimal intensity by 
\begin{align}
     t\opt=\arg\min_{ t\in[0,1]}\E_n[\|\Sigma^{LW}(\nu\opt, t)-\Sigma_0\|_F^2].
\end{align}
The optimization problems in both steps are convex quadratic programs and have a closed-form expression for the optimal solution. 
To choose the parameter of $\Sigma\opt(\tau,\rho)$, one may mimic the approach to determine $(\tau,\rho)$ by the following minimization problem
\begin{align}
\min_{\tau\in\R_{++}, \rho\in(0, \rho_{\max}]} \E_n[\|\Sigma\opt(\tau, \rho)-\Sigma_0\|_F^2]
\end{align}
and solve the problem in two steps.
Unfortunately, it turns out that each step is a nonconvex problem in general and the optimal solution does not have a closed form  like~\eqref{prob:optimal-para-LW}. 
This is because $\Sigma\opt(\tau, \rho)$ is nonlinear 
in $\tau$ and $ \rho$ (see Table~\ref{tab:vphi-gamma} for specific form of $\varphi$), and the quadratic form of $\Sigma\opt(\tau, \rho)$ is hard to analyze. 

In the forthcoming discussions of this section, we follow the general framework of the parameter tuning procedure of solving~\eqref{prob:optimal-para-LW}, i.e., first finding the optimal shrinkage target and then, with this fixed target, finding the optimal intensity but with some important variations. To formulate a tractable parameter-tuning problem, we propose to use the violation of the optimality of~\eqref{prob:P-Mat} evaluated at $\Sigma=\Sigma_0$, which depends on $\tau$ and $\rho$, as a measure of distance from the target $\sqrt{\frac{1}{\tau}}I$ and the estimator $\Sigma\opt(\tau, \rho)$ to the ground truth $\Sigma_0$. 
Such a concept of the violation of optimality is widely used in inverse optimization, where a decision maker aims to infer the parameters of an optimization model so that the optimal solution is equal to or at least close to the solution we have observed. A typical approach is to minimize the violation of the optimality evaluated at the observed solution by tuning the parameters. See~\cite{ref:chan2025inverse} for a survey of inverse optimization. We note that the choice of the optimal shrinkage target and intensity can be viewed as an inverse optimization task: we want to choose the parameters $\tau$ and $ \rho$ properly so that the ground true optimal solution $\Sigma_0$ is nearly optimal to~\eqref{prob:P-Mat}, and thus the distance between $\Sigma\opt(\tau,\rho)$ and $\Sigma_0$ is close. To simplify the derivation, the nominal covariance estimator $\wh\Sigma$ is chosen as the sample covariance matrix $\wh\Sigma_n=\frac{1}{n}\sum_{i=1}^n \xi_i\xi_i^\transpose$ of $n$ i.i.d. samples in this section, where we use the subscript $n$ in $\wh\Sigma_n$ to emphasize its dependence on sample size $n$. The corresponding estimator is denoted by $\Sigma_n\opt(\tau, \rho)$ throughout this section.

\subsection{Optimal Shrinkage Target}\label{sec:opt-target}

We begin by identifying the optimal shrinkage target which minimizes the violation of the optimality of~\eqref{prob:P-Mat} evaluated at $\Sigma=\Sigma_0$. Let $\ell_\tau(\Sigma) \Let \log\det \Sigma - \frac{1}{2}\tau\|\Sigma\|_F^2$ be the objective function of problem~\eqref{prob:P-Mat}, where we make explicit the dependence of the objective function on the parameter~$\tau$.
Proposition~\ref{prop:unbinding} asserts that the shrinkage target $\sqrt{\frac{1}{\tau}}I$ is the unique maximizer of~\eqref{prob:P-Mat} when the constraint is redundant, i.e., it solves $\max_{\Sigma\in\PSD}\ell_\tau(\Sigma)$. Thus to choose $\tau$ so that $\sqrt{\frac{1}{\tau}}I$ is close to $\Sigma_0$, we minimize the violation of the optimality evaluated at $\Sigma=\Sigma_0$, which is captured by the norm of the gradient $\nabla_{\Sigma} \ell_\tau(\Sigma_0)$:
\begin{align}\label{prob:opt-tau}
    \tau\opt\Let\arg\min_{\tau \in \R_{+}}~\| \nabla_{\Sigma} \ell_\tau(\Sigma_0)  \|_F.
\end{align}
Observe that $\|\nabla_{\Sigma} \ell_\tau(\Sigma_0)\|_F^2 = \|\Sigma_0^{-1}-\tau\Sigma_0\|_F^2$, which is a quadratic function of $\tau$. Thus, we can easily obtain the optimal solution of the problem $\tau\opt= \frac{p}{\|\Sigma_0\|_F^2}$.
The induced optimal shrinkage target is
\begin{align}\label{eq:optimal-target}
    \Sigma_{\text{target}}=\sqrt{\frac{1}{\tau\opt}}I=\frac{\|\Sigma_0\|_F}{\sqrt{p}}I.
\end{align}
Since $\left\|\frac{\|\Sigma_0\|_F}{\sqrt{p}}I\right\|_F=\|\Sigma_0\|_F$, then the target can be viewed as a normalized guess of $\Sigma_0$ with the same scale as $\Sigma_0$, i.e., the total scale is distributed evenly to each dimension. 
This is a reasonable prior when we only have access to the scale of $\Sigma_0$, or an estimation of it.
In practice, if we approximate $\|\Sigma_0\|_F$ by $\|\wh\Sigma_n\|_F$, then the approximation $\wh{\tau}\opt\Let \frac{p}{\|\wh\Sigma_n\|_F^2}$ satisfies $\wh\lambda_1\leq\sqrt{\frac{1}{\wh{\tau}\opt}}\leq \wh\lambda_p$. As discussed after Theorem~\ref{thm:nonlinear-shrinkage}, $\wh{\tau}\opt$ well balances the shrinkage effects of the Stein loss and the Frobenius loss in~\eqref{prob:jointly-model-dro}, and corrects the spectral bias of $\wh\Sigma_n$.

\subsection{Optimal Shrinkage Intensity (Radius) and its Limit Behavior}\label{sec:opt-radius}

With the optimal target chosen as above, we tune the radius $ \rho$ so that the estimator $\Sigma_n\opt(\tau\opt,\rho)$ is close to the true covariance $\Sigma_0$. Recall that $\Sigma\opt(\tau\opt,\rho)$ is the optimal solution of the following problem:
\begin{align}\label{prob:P-Mat-tau*}
    \max_{\Sigma\in\PD: D(\Sigma, \widehat{\Sigma}_n)\leq  \rho}~\log\det \Sigma - \frac{1}{2}\tau\opt \|\Sigma\|_F^2.
\end{align}
In the following, we first derive the optimality condition of~\eqref{prob:P-Mat-tau*}, and then find the optimal radius $ \rho$ by minimizing the violation of the optimality evaluating at $\Sigma=\Sigma_0$. The Lagrangian function of~\eqref{prob:P-Mat-tau*} is
\[
    L(\Sigma,\gamma) \Let \log\det \Sigma - \frac{1}{2}\tau\opt\|\Sigma\|_F^2 - \gamma (D(\Sigma,\wh\Sigma_n)-\rho),
\]
where $\gamma\geq 0$ is the dual variable. Then the optimiality condition of~\eqref{prob:P-Mat-tau*} is given by the Karush-Kuhn-Tucker condition
\begin{align}
    \nabla_1 L(\Sigma,\gamma) = \Sigma^{-1} - \tau\opt \Sigma - \gamma \nabla_1 D(\Sigma,\wh\Sigma_n) = 0,\label{eq:tau*-opt-cond-1}\\
    \gamma (D(\Sigma, \wh\Sigma_n)-\rho) = 0,\label{eq:tau*-opt-cond-2}\\
    \gamma \geq 0, \Sigma\in\PD\label{eq:tau*-opt-cond-3},
\end{align}
where $\nabla_1$ represents the gradient with respect to the first argument. By the construction of $\Sigma_n\opt(\tau\opt,\rho)$ and $\gamma\opt_{\wh\lambda}(\rho)$, $(\Sigma=\Sigma_n\opt(\tau\opt,\rho),\gamma=\gamma\opt_{\wh\lambda}(\rho))$ solves~\eqref{eq:tau*-opt-cond-1}-\eqref{eq:tau*-opt-cond-3}, where $\gamma\opt_{\wh\lambda}(\rho)$ is the optimal dual solution as defined in~\eqref{def-gamma-opt}.
It is intuitive to see $\|\nabla_1 L(\Sigma,\gamma\opt_{\wh\lambda}(\rho))\|_F$ as a measure of the violation of the optimality evaluated at $\Sigma$.
This prompts us to consider the 
$\rho$ that minimizes the violation evaluated at $\Sigma=\Sigma_0$, that is, 
\begin{align}
     \rho^\star(\wh\Sigma_n) \Let \arg\min_{ \rho\in(0,\rho_{\max}]}\;\|\Sigma_0^{-1}-\tau\opt \Sigma_0-\gamma_{\wh\lambda}\opt( \rho)  {\nabla_1 D } (\Sigma_0, \wh \Sigma_n)\|_F^2. \label{eq:opt-eps-1}
\end{align}
The above defines a sample-dependent optimal radius, since both $\gamma\opt_{\wh\Sigma}$ and $\nabla_1 D(\Sigma_0,\wh\Sigma_n)$ depend on the realization of $\wh\Sigma_n$. In practice, however, the radius should be set independently of the samples to facilitate simulations in both model training and validation. To tackle the issue, we replace $\gamma_{\wh\lambda}\opt(\rho)$ with $\gamma_{\lambda}(\rho)$, which is defined as the unique root $\gamma\opt$ of $\sum_{i=1}^d d(\varphi(\tau\opt,\gamma\opt,\lambda_i),\lambda_i)=0$, and 
${\nabla_1 D }(\Sigma_0,\wh\Sigma_n)$ with  $\E_n[{\nabla_1 D }(\Sigma_0,\wh\Sigma_n)]$. Consequently, we consider
\begin{align}
     \rho_n\opt \Let\arg\min_{ \rho\in \R_{++}}\;\|\Sigma_0^{-1}-\tau\opt \Sigma_0-\gamma_{\lambda}( \rho) \E_n[ {\nabla_1 D } (\Sigma_0, \wh \Sigma_n)]\|_F^2. \label{prob:optimal-eps}
\end{align}
The replacement of $\gamma\opt_{\wh\lambda}$ is justified by the fact that $\gamma_{\wh\lambda}\opt(\rho)\to\gamma\opt_{\lambda}(\rho)$  as $n\to\infty$ with probability 1.
The optimal intensity defined by~\eqref{prob:optimal-eps} is still not achievable in that $\Sigma_0$ is unknown. However, the formulation enables us to derive the limiting behavior of $ \rho_n\opt$ as $n\to \infty$. It reveals that an asymptotic optimal intensity is in the form $\rho\opt/ n^{2}$, where $\rho\opt$ is a constant that can be figured out explicitly.
This kind of analysis for the optimal radius was also studied by~\cite{ref:BLANCHET2019618}. Before stating the result, we make some assumptions about the divergence functions.

\begin{assumption}[Locally-quadratic divergence]\label{ass:like-frob}
    For any $b>0$, there exists positive constant $C_{b,d}$ depending on the generator $d$ of the divergence and $b$ such that 
    \[
        \lim_{a\to b} \frac{d(a,b)}{(a-b)^2} = C_{b,d}.
    \]
\end{assumption}

\begin{assumption}[Non-degenerate gradient of divergence function]\label{ass:D'neq0}
     For any sample size $n$, the gradient of divergence function $D$ satisfies $\E_n[{\nabla_1 D }(\Sigma_0, \wh\Sigma_n)]\neq 0$ and there exists strictly positive constants $C_1$ and $C_2$ such that 
     \bgeqn
     \label{eq:asymp1-NbD-1}
        \lim_{n\to\infty} n\|\E_n[{\nabla_1 D }(\Sigma_0, \wh\Sigma_n)]\|_F = C_1
     \edeqn 
     and
     \bgeqn 
     \label{eq:asymp1-NbD-2}
        \lim_{n\to\infty} n\left\la \Sigma_0^{-1}-\tau\opt\Sigma_0, \E_n [{\nabla_1 D }(\Sigma_0, \wh\Sigma_n)]\right\ra=C_2.
     \edeqn 
\end{assumption}

Condition $\E_n[{\nabla_1 D }(\Sigma_0, \wh\Sigma_n)]\neq 0$
means that $\Sigma_0$ is not a minimizer of $\E_n[{D }(\Sigma, \wh\Sigma_n)]$ as 
\[
{\nabla_1  }\E_n[D(\Sigma, \wh\Sigma_n)]|_{\Sigma=\Sigma_0}
=\E_n[{\nabla_1 D }(\Sigma_0, \wh\Sigma_n)]\neq 0,
\]
provided that the differential operator $\nabla_1$ and $\E_n$ are interchangeable. This condition is not satisfied when
$\nabla_1 D (\Sigma_0, \wh\Sigma_n)$ is linear in the second argument. For example, the squared Frobenius and weighted Frobenius divergence functions do not satisfy the condition because $\E_n[{\nabla_1 D }(\Sigma_0,\wh\Sigma_n)]=0$ for all $n$. 
Condition \eqref{eq:asymp1-NbD-1} requires the norm of the expected gradient $\E_n[{\nabla_1 D }(\Sigma_0, \wh\Sigma_n)]$ goes to zero at rate $1/n$ as $n\to\infty$. Condition \eqref{eq:asymp1-NbD-2} means that inner product between $ \Sigma_0^{-1}-\tau\opt\Sigma_0$ and $ \E_n [{\nabla_1 D }(\Sigma_0, \wh\Sigma_n)]$ goes to zero at rate $1/n$. Here we specify the conditions in terms of the gradient instead of the divergence function itself since we apply the first-order optimality condition to the definition of $\rho\opt_n$. 
It can be verified that the first three divergence functions in  Table~\ref{tab:divergences} satisfy Assumptions~\ref{ass:like-frob} and~\ref{ass:D'neq0}, except 
Squared Frobenius and Weighted Frobenius divergence functions.
With the above two assumptions, we are ready to state 
the asymptotic behavior of $ \rho_n\opt$  in the next theorem.

\begin{theorem}[$1/n^2$-order optimal radius]\label{thm:optimal-eps-limit}
    Under Assumptions~\ref{ass:convex-divergence},~\ref{ass:Spectral-divergence},~\ref{ass:like-frob} and~\ref{ass:D'neq0}, there exists $ \rho\opt>0$ such that the optimal intensity $ \rho_n\opt$ defined as in~\eqref{prob:optimal-eps} satisfies
    \begin{equation}
        \lim_{n\to\infty} n^2  \rho_n\opt =  \rho\opt,   
    \end{equation}
    where the limit $ \rho\opt$ depends on the choice of divergence function $D$ and the constants $C_1$ and $C_2$.
\end{theorem}

The above theorem implies that when the sample size is large, the asymptotically optimal radius tends to $0$ and is of the order $ \rho\opt/n^2$.
We emphasize that a $1/n^2$-order radius of our model is parallel with optimal radius tuning schemes proposed in the literature. In~\cite{ref:nguyen2022distributionally}, the authors consider a distributionally robust inverse covariance estimation model with Wasserstein distance, which is a square root version of Wasserstein divergence in Table~\ref{tab:divergences}. Then, \citet[Theorem~1]{ref:BLANCHET2019618} showed that the optimal radius is of rate $1/n$ in the sense of minimizing Stein's loss of the proposed estimator. This is of the same order as $ \rho_n\opt$ after taking the square. 

\subsection{Optimal Shrinkage Intensity of Convex Divergences}\label{sec:optimal-radius}

In this section, we discuss special cases of Theorem~\ref{thm:optimal-eps-limit}, where $D$ is chosen as the Kullback-Leibler divergence, Wasserstein divergence, and Symmetrized Stein divergence as listed in Table~\ref{tab:divergences}. To obtain a concrete expression of the constants $C_1,C_2$ and $ \rho\opt$, we consider a specific probability distribution of the underlying random vector $\xi$. The next assumption specifies this.

\begin{assumption}[Normal distribution]\label{ass:normal-distribution}
    The random vector $\xi$ follows the normal distribution $\mathcal{N}(0, \Sigma_0)$ with $\Sigma_0\neq \sigma I$ for any $\sigma \in \R_{++}$.
\end{assumption}

The next corollary asserts that under the normal distribution assumption, Assumptions~\ref{ass:like-frob} and~\ref{ass:D'neq0} hold for divergences of Kullback-Leibler, Wasserstein, and Symmetrized Stein, and give explicit expressions of the limit constant $ \rho\opt$ for each type of divergence.

\begin{corollary}[Asymptotic convergence rate of the optimal radius]\label{coro:optimal-radius}
    Let $\tau\opt$ and $\rho_n\opt$ be defined as in~\eqref{prob:opt-tau} and \eqref{prob:optimal-eps}  respectively.
    Under Assumption~\ref{ass:normal-distribution}, the following assertions hold.
    \begin{enumerate}[label=(\roman*)]
        \item If $D$ is taken as Kullback-Leibler divergence, then 
        \begin{align}\label{eq:optimal-eps-limit-KL}
            \lim_{n\to\infty} n^2  \rho_n\opt =  \rho_{\text{KL}}\opt\Let\frac{(p+1)^2\|\Sigma_0^{-1}\|_F^4}{16\left(\|\Sigma_0^{-1}\|_F^2-\frac{p^2}{\|\Sigma_0\|_F^2}\right)^2}\sum_{i=1}^p \left(1-\tau\opt \lambda_i^2\right)^2.
        \end{align}
        \item If $D$ is taken as Wasserstein divergence, then 
        \begin{align}\label{eq:optimal-eps-limit-W}
            \lim_{n\to\infty} n^2  \rho_n\opt =  \rho_{\text{W}}\opt\Let\frac{(p+1)^2p^2}{256\left(\tr(\Sigma_0^{-1}) - \frac{p}{\|\Sigma_0\|_F^2}\tr(\Sigma_0)\right)^2}\sum_{i=1}^p \frac{(1-\tau\opt \lambda_i^2)^2}{\lambda_i}.
        \end{align}
        \item If $D$ is taken as symmetrized Stein divergence, then 
        \begin{align}\label{eq:optimal-eps-limit-SS}
            \lim_{n\to\infty} n^2  \rho_n\opt =  \rho_{\text{SS}}\opt\Let\frac{(p+1)^2\|\Sigma_0^{-1}\|_F^4}{32\left(\|\Sigma_0^{-1}\|_F^2-\frac{p^2}{\|\Sigma_0\|_F^2}\right)^2}\sum_{i=1}^p \left(1-\tau\opt \lambda_i^2\right)^2.
        \end{align}
    \end{enumerate}
\end{corollary}
To apply the optimal target and the optimal radius in practice, we may use the sample covariance $\wh\Sigma_n$ and empirical eigenvalues to approximate the coefficient and constants in~\eqref{eq:optimal-target} and~\eqref{eq:optimal-eps-limit-KL}-\eqref{eq:optimal-eps-limit-SS}. For given sample size $n$, the target can be set as $\frac{\|\wh\Sigma_n\|_F}{\sqrt{p}}I$ and  the radius can be set as $\wh\rho_n=\wh{ \rho}\opt/n^2$, where ${\rho}\opt$ is approximated by $\wh{\rho}\opt$ with empirical counterpart of $\Sigma_0$ and $\lambda_1,\ldots,\lambda_p$. As $n$ is sufficiently large, 
the approximations are close to their true counterparts since $\wh\Sigma_n$ is a consistent estimator of $\Sigma_0$.

\section{Numerical Experiments} \label{sec:exp}

In this section, we conduct numerical experiments to compare the performance of the proposed nonlinear shrinkage estimator SCOPE with the classical Ledoit-Wolf's linear shrinkage estimator (LW-L) by~\cite{ref:LEDOIT2004365}, the state-of-the-art nonlinear shrinkage estimator (LW-NL) recently proposed by~\cite{ref:LedoitWolfNonlinear}, the Wasserstein-based distributionally robust covariance matrix estimator (W-DRCOE) by~\cite{ref:yue2024geometric}, and the Wasserstein-based distributionally robust precision matrix estimator (WISE) by~\cite{ref:nguyen2022distributionally}. As mentioned in~\eqref{prob:optimal-para-LW}, the optimal target and optimal intensity of Ledoit-Wolf's linear shrinkage have closed-form expressions and can be directly applied. The implementations of LW-L, LW-NL, W-DRCOE, and WISE are based on the open-source code provided by the authors. In the experiments, we focus on SCOPEs induced by the KL divergence (KL-SCOPE), the Wasserstein divergence (W-SCOPE), and the symmetrized Stein divergence (SS-SCOPE), all of which fall within the framework of Corollary~\ref{coro:optimal-radius}. 

In the first experiment, we provide an empirical validation of the order of the asymptotically optimal radius proposed in Section~\ref{sec:opt-radius}. After that, we compare SCOPEs with other estimators by assessing their estimation error with synthetic data.
In the subsequent experiments, we use real data and synthetic data to compare the performance of the estimators in practical applications, including anomaly detection, A/B tests, and minimum variance portfolio selection.
The implementation of our methods can be found at \url{https://github.com/searsh/SCOPE}.

\begin{figure}[htbp]
    \centering
        \begin{subfigure}[b]{0.33\textwidth}
        \includegraphics[width=\textwidth]{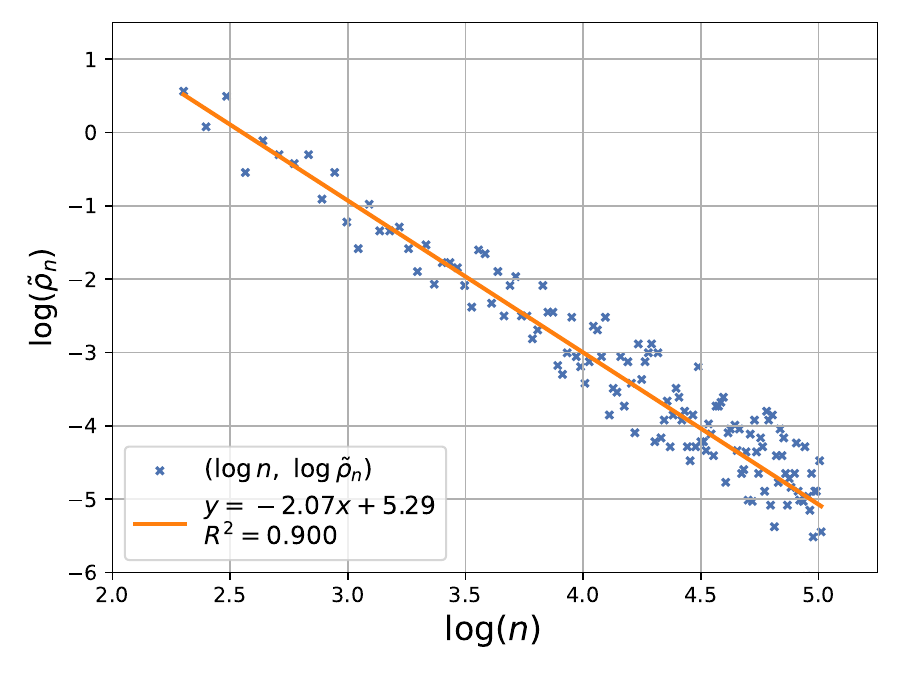}
        \caption{KL divergence}
        \label{fig:rho_kl}
    \end{subfigure}
    \hspace{-0.02\textwidth}
    \begin{subfigure}[b]{0.33\textwidth}
        \includegraphics[width=\textwidth]{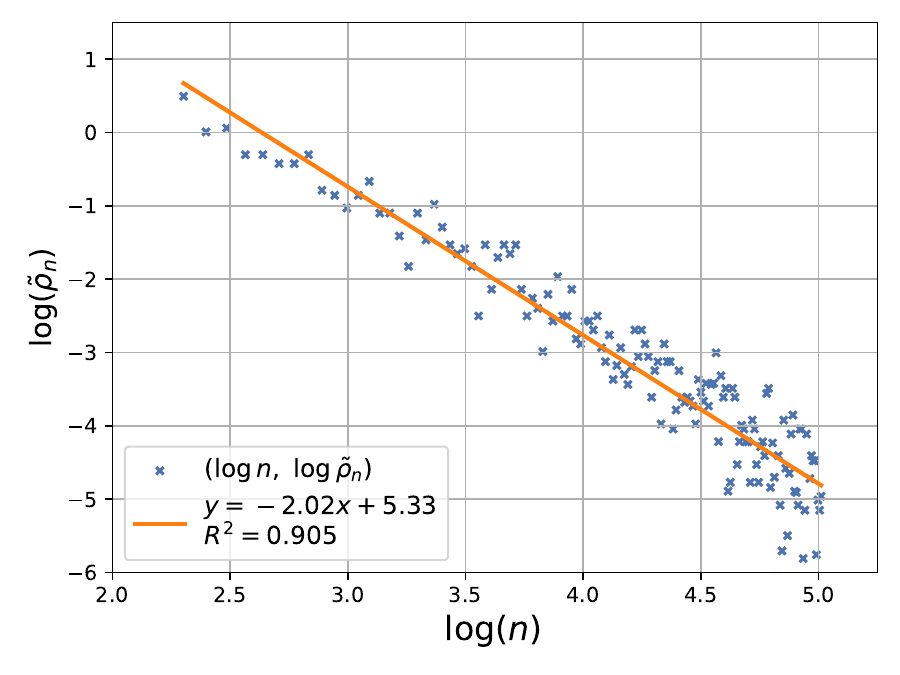}
        \caption{Wasserstein divergence}
        \label{fig:rho_wa}
    \end{subfigure}
    \hspace{-0.02\textwidth}
    \begin{subfigure}[b]{0.325\textwidth}
        \includegraphics[width=\textwidth]{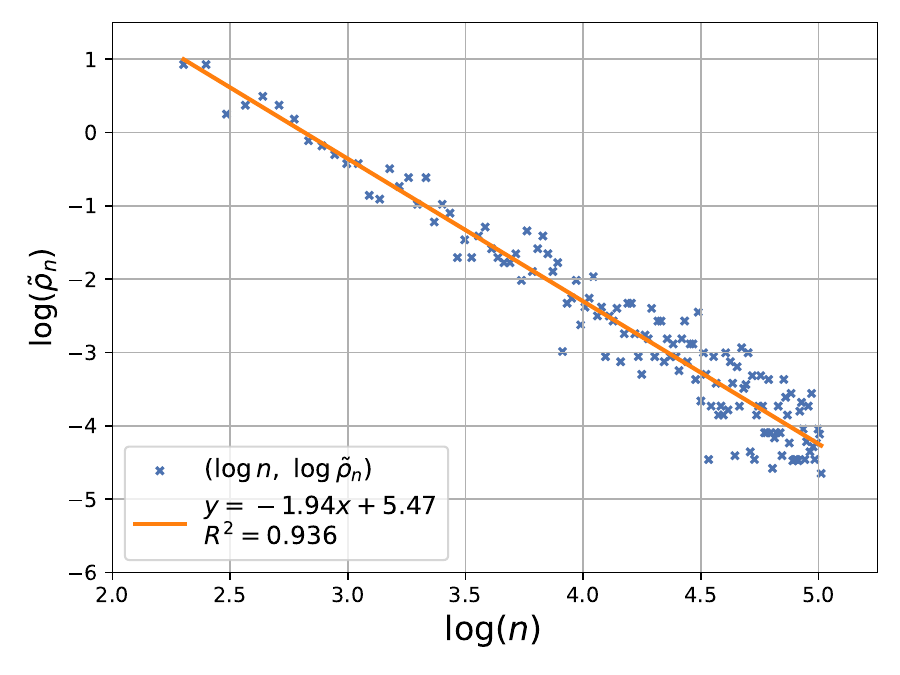}
        \caption{Symmetrized Stein divergence}
        \label{fig:rho_ss}
    \end{subfigure}

    \caption{Log–log regression of optimal radius $\widetilde{\rho}_n$ vs. sample size $n$ under divergences.}
    \label{fig:rho_all}
\end{figure}

\subsection{Validation of Optimal Radius}\label{sec:exp-0}
In the first experiment, we validate the order of optimal radius proposed in Theorem~\ref{thm:optimal-eps-limit} and Corollary~\ref{coro:optimal-radius}, i.e., whether the optimal radius is of order $n^{-2}$, via grid search on radius. To this end, we set $p=5$ and the true covariance matrix is generated via $\Sigma_0=V^\transpose \Lambda V \in \PD$, where $\Lambda=\diag(1,2,\ldots,p)$ and $V$ is an orthogonal matrix. For fixed sample size $n$, we draw $n$ samples from $\mathbb{P}={\cal N}(0,\Sigma_0)$ and construct the sample covariance matrix $\wh\Sigma_n$. Then, taking $\wh\Sigma_n$ as the nominal matrix, we construct KL-SCOPE, W-SCOPE, and SS-SCOPE, denoted by $\Sigma_n\opt(\wh\tau_n\opt, \rho_n)$, where we set $\wh\tau_n\opt=p/\|\wh\Sigma_n\|_F^2$. We write $\Sigma_n\opt$ for short. Recall that $\tau\opt=p/\|\Sigma_0\|_F^2$ and the combined Stein-Frobenius loss of the estimator $\Sigma_n\opt$ is defined by
\begin{align}\label{eq:exp-loss}
    \ell(\Sigma_n\opt,\Sigma_0) \Let -\log\det\left((\Sigma_n\opt)^{-1}\right) + \left\la (\Sigma_n\opt)^{-1}, \Sigma_0\right\ra + \frac{1}{2}\tau\opt\left(\|\Sigma_n\opt\|_F^2-2\la\Sigma_n\opt,\Sigma_0\ra\right).
\end{align}
For the choice of $\rho_n$, we do a grid search over $\rho_n\in [10^{-5}, 2\times 10^3]$, and find the best radius $\widetilde{\rho}_n$ that minimizes the average loss over 50 independent repeats of construction of $\wh\Sigma_n$ and $\wh\Sigma_n\opt$. We repeat the above procedure for sample size $n$ over $\{10,11,\ldots,150\}$, and record the corresponding optimal radius~$\widetilde{\rho}_n$. To find out the dependence of $\widetilde{\rho}_n$ on $n$, we conduct a log-log-linear regression between $\widetilde{\rho}_n$ and $n$. That is, we use a straight line $y = \alpha x + \beta$ to fit the points $(\log n, \log \widetilde{\rho}_n)$. 
Figure~\ref{fig:rho_all} visualizes the regression results. It is shown that the points are well fitted by the line ($R^2>0.9$), and the slopes of the lines in all three cases are close to $-2$. The results indicate that the dependence of $\widetilde{\rho}_n$ on $n$ can be approximated by $\exp{(\log \widetilde{\rho}_n)} \approx \exp{(-2 \log n +\beta)}$, and thus $\widetilde{\rho}_n \sim c\cdot n^{-2}$. This validates the theoretical results in Section~\ref{sec:opt-radius}.

\subsection{Estimation Error}\label{sec:exp-1}
In this experiment, we use synthetic data to compare SCOPEs with other estimators proposed in the literature. The true covariance matrix is generated via 
$\Sigma_0 = V^\transpose \Lambda V\in\PD$, where $\Lambda=\diag(1,2,\ldots,p)$ and $V$ is an orthogonal matrix.
For fixed sample size $n$, we draw $n$ samples from $\mathbb{P}=\mathcal{N}(0,\Sigma_0)$ and construct the the sample covariance matrix (Sample) $\wh\Sigma_n$ as the nominal estimator of $\Sigma_0$. With $\wh\Sigma$, we generate the Wasserstein-based SCOPE (W-SCOPE), Ledoit-Wolf's linear estimator (LW-L), Ledoit-Wolf's nonlinear estimator (LW-NL), Wasserstein-based distributionally robust covariance estimator (W-DRCOE), and Wasserstein-based distributionally robust precision matrix estimator (WISE). 
To make a fair comparison, the sizes of ambiguity sets of SCOPEs, W-DRCOE, and WISE are all tuned via grid search over $\rho_n\in [10^{-5}, 2\times 10^3]$ as in Section~\ref{sec:exp-0}. 
Specifically, let $\Sigma_0'\in\PD$ be the true covariance matrix used in grid search. For different sample size $n$, we draw $n$ i.i.d.~samples from ${\cal N}(0,\Sigma_0')$ repeatedly.
For W-SCOPE, we generate the covariance matrix estimator from the samples repeatedly and choose the radius that minimizes the average combined Stein-Frobenius loss $\ell(\Sigma\opt_n,\Sigma_0)$.
For W-DRCOE, we generate the covariance matrix estimator from the samples repeatedly and choose the radius that minimizes the average Frobenius loss. For WISE, we generate the precision matrix estimator from the samples repeatedly and choose the radius that minimizes the average Stein's loss.

The numerical experiments are conducted for $p\in\{150,200\}$ and various sample sizes~$n$. 
The estimation error is measured by the combined Stein-Frobenius loss. Moreover, to evaluate the estimators' ability to correct the spectral bias, we compute the relative error of the eigenvalues of the estimators. Let $\{\lambda_i\}_{i=1}^p$ and $\{\wh\lambda_i\}_{i=1}^p$ be the eigenvalues of $\Sigma_0$ and the estimator, respectively, both in decreasing order. We define the relative error of the eigenvalue estimation by 
\[
    \frac{1}{p}\sum_{i=1}^p \frac{|\wh\lambda_i-\lambda_i|}{\lambda_i}.
\]

Table~\ref{tab:Frob-p=150}-\ref{tab:Frob-p=200} display the combined Stein-Frobenius loss of the estimators. The loss is recorded as N/A when the covariance matrix estimator is singular.
It is shown that in data-deficient cases, i.e., when $n<p$,  only LW-L, WISE, and W-SCOPE provide invertible estimators of the covariance matrix, and W-SCOPE achieves comparable or smaller loss than WISE. Table~\ref{tab:EigErr-p=150}-\ref{tab:EigErr-p=200} display the relative error of the eigenvalue estimation. The result reveals that W-SCOPE achieves better relative error than other estimators except LW-NL. Moreover, when $n<p$, W-SCOPE achieves the smallest relative error among all the invertible estimators. The result of this experiment shows that W-SCOPE is especially useful when data is insufficient. It provides a choice when both the covariance matrix and the precision matrix are to be estimated.

\begin{table}[H]
\begin{adjustbox}{center}
\begin{tabular}{lccccccc}
\hline
Sample size & 50 & 70 & 100 & 120 & 150 & 200 & 300 \\
\hline
Sample     & N/A & N/A & N/A & N/A & $2.0\times10^{9}$ & 1136.87 & 861.87 \\
LW-L     & 796.62 & 795.59 & 794.37 & 793.50 & 792.35 & 790.52 & 787.23 \\
LW-NL    & N/A & N/A & N/A & N/A & N/A & 784.86 & 777.04 \\
W-DRCOE  & N/A & N/A & N/A & N/A & $2.0\times10^{9}$ & 1150.21 & 866.78 \\
WISE     & 824.16 & 826.39 & 816.29 & 810.30 & 803.58 & 797.10 & 787.99 \\
W-SCOPE     & 941.90 & 868.45 & 813.88 & 799.02 & 791.75 & 788.27 & 782.86 \\
\hline
\end{tabular}
\end{adjustbox}
\caption{Combined Stein–Frobenius loss of the estimators for $p=150$.}
\label{tab:Frob-p=150}
\end{table}

\begin{table}[H]
\begin{adjustbox}{center}
\begin{tabular}{lcccccccc}
\hline
Sample size & 50 & 100 & 120 & 150 & 170 & 200 & 300 & 400 \\
\hline
Sample     & N/A & N/A & N/A & N/A & N/A & $3.1\times 10^5$ & 1379.82 & 1203.59 \\
LW-L     & 1119.95 & 1117.72 & 1116.86 & 1115.62 & 1114.81 & 1113.59 & 1109.91 & 1106.74 \\
LW-NL    & N/A & N/A & N/A & N/A & N/A & N/A & 1099.36 & 1092.63 \\
W-DRCOE  & N/A & N/A & N/A & N/A & N/A & $3.1\times 10^5$ & 1392.79 & 1212.47 \\
WISE     & 1153.17 & 1151.09 & 1145.93 & 1139.13 & 1136.12 & 1131.08 & 1115.28 & 1106.69 \\
W-SCOPE     & 1789.84 & 1347.02 & 1252.64 & 1170.21 & 1141.70 & 1120.44 & 1105.90 & 1100.83 \\
\hline
\end{tabular}
\end{adjustbox}
\caption{Combined Stein–Frobenius loss of the estimators for $p=200$. }
\label{tab:Frob-p=200}
\end{table}

\begin{table}[H]
\begin{adjustbox}{center}
\begin{tabular}{lccccccc}
\hline
Sample size & 50 & 70 & 100 & 120 & 150 & 200 & 300 \\
\hline
Sample     & 0.97 & 0.85 & 0.73 & 0.67 & 0.59 & 0.49 & 0.36 \\
LW-L     & 1.79 & 1.75 & 1.64 & 1.57 & 1.48 & 1.38 & 1.18 \\
LW-NL    & 0.57 & 0.46 & 0.31 & 0.29 & 0.57 & 0.46 & 0.30 \\
W-DRCOE  & 0.97 & 0.85 & 0.73 & 0.67 & 0.57 & 0.48 & 0.36 \\
WISE     & 1.71 & 2.63 & 1.12 & 1.06 & 0.99 & 0.97 & 0.71 \\
W-SCOPE     & 0.83 & 0.75 & 0.71 & 0.72 & 0.98 & 0.70 & 0.59 \\
\hline
\end{tabular}
\end{adjustbox}
\caption{Relative error of eigenvalue estimation of the estimators for $p=150$. }
\label{tab:EigErr-p=150}
\end{table}

\begin{table}[H]
\begin{adjustbox}{center}
\begin{tabular}{lcccccccc}
\hline
Sample size & 50 & 100 & 120 & 150 & 170 & 200 & 300 & 400 \\
\hline
Sample     & 1.07 & 0.83 & 0.77 & 0.69 & 0.65 & 0.59 & 0.45 & 0.36 \\
LW-L     & 2.01 & 1.84 & 1.80 & 1.71 & 1.67 & 1.60 & 1.41 & 1.27 \\
LW-NL    & 0.66 & 0.44 & 0.37 & 0.31 & 0.31 & 0.56 & 0.40 & 0.30 \\
W-DRCOE  & 1.07 & 0.83 & 0.77 & 0.69 & 0.65 & 0.55 & 0.44 & 0.37 \\
WISE     & 2.23 & 1.52 & 1.47 & 1.43 & 1.46 & 1.38 & 0.68 & 0.52 \\
W-SCOPE     & 0.92 & 0.68 & 0.63 & 0.57 & 0.54 & 0.59 & 0.72 & 0.54 \\
\hline
\end{tabular}
\end{adjustbox}
\caption{Relative error of eigenvalue estimation of the estimators for $p=200$. }
\label{tab:EigErr-p=200}
\end{table}

\subsection{Application in Anomaly Detection of Hyperspectral Images}
Hyperspectral imagery, typically used in remote sensing, agriculture, and environmental monitoring, provides detailed data of observed items via hundreds of continuous spectral bands.
Anomaly detection in hyperspectral images seeks to identify materials(pixels) different from the surroundings, spatially or spectrally. Reed–Xiaoli (RX) detector~\citep{ref:reed1990adaptive} is a well-known statistical tool that models anomaly detection as an unsupervised binary classification problem. It assumes the background pixels of a hyperspectral image follow a multivariate normal distribution $\mathcal{N}(\mu_0, \Sigma_0)$. For every pixel in the image, the RX-detector calculates the Mahalanobis distance of the pixel's spectral vector $x$ from the mean of the background distribution, defined by
\[
    d_M(x) = (x-\mu_0)^\transpose \Sigma_0^{-1} (x-\mu_0). 
\] 
The distance measures how far the pixel vector $x$ is away from the background. If the Mahalanobis distance is larger than a threshold set by the statistician, the pixel is identified as an anomaly. In practice, the true covariance $\Sigma_0$ is unknown and needs to be estimated from the data(all pixels). Moreover, with the development of sensing techniques, the increase of spectral bands in hyperspectral images makes anomaly detection more difficult, and calls for a high-dimensional covariance estimator when using the RX-detector.

In this experiment, we use RX-detectors induced by several covariance/precision estimators, i.e., sample covariance matrix (Sample), Ledoit-Wolf’s linear estimator (LW-L), Ledoit-Wolf’s nonlinear estimator (LW-NL), Wasserstein-based SCOPE (W-SCOPE), Kullback-Leibler-based SCOPE (KL-SCOPE), Symmetrized Stein-Based SCOPE (SS-SCOPE), Wasserstein-based distributionally robust covariance matrix estimator (W-DRCOE), and Wasserstein-based distributionally robust precision matrix estimator (WISE) to conduct anomaly detection on the sensing dataset collected from~\cite{ref:lin2022hyperspectral}. The dataset consists of five real hyperspectral images, namely San Diego, Pavia, HYDICE, Texas Coast, and SpecTIR, with labeled anomalies. To compare the detectors' performance under high-dimensional regimes, we choose sub-images of these five images so that the number of samples used to estimate covariance is nearly equal to the dimension of spectral vectors. Figure~\ref{fig:hyperspectral_anomaly_visualization} visualizes the tested images. The first row lists the original images, the second row shows the labeled anomalies, and the red boxes in the first row highlight the sub-images we chose to conduct the detections. The sub-images focus on where the anomaly occurs, and also contain less background information.

\begin{figure}[htbp]
  \centering
  \includegraphics[width=\textwidth]{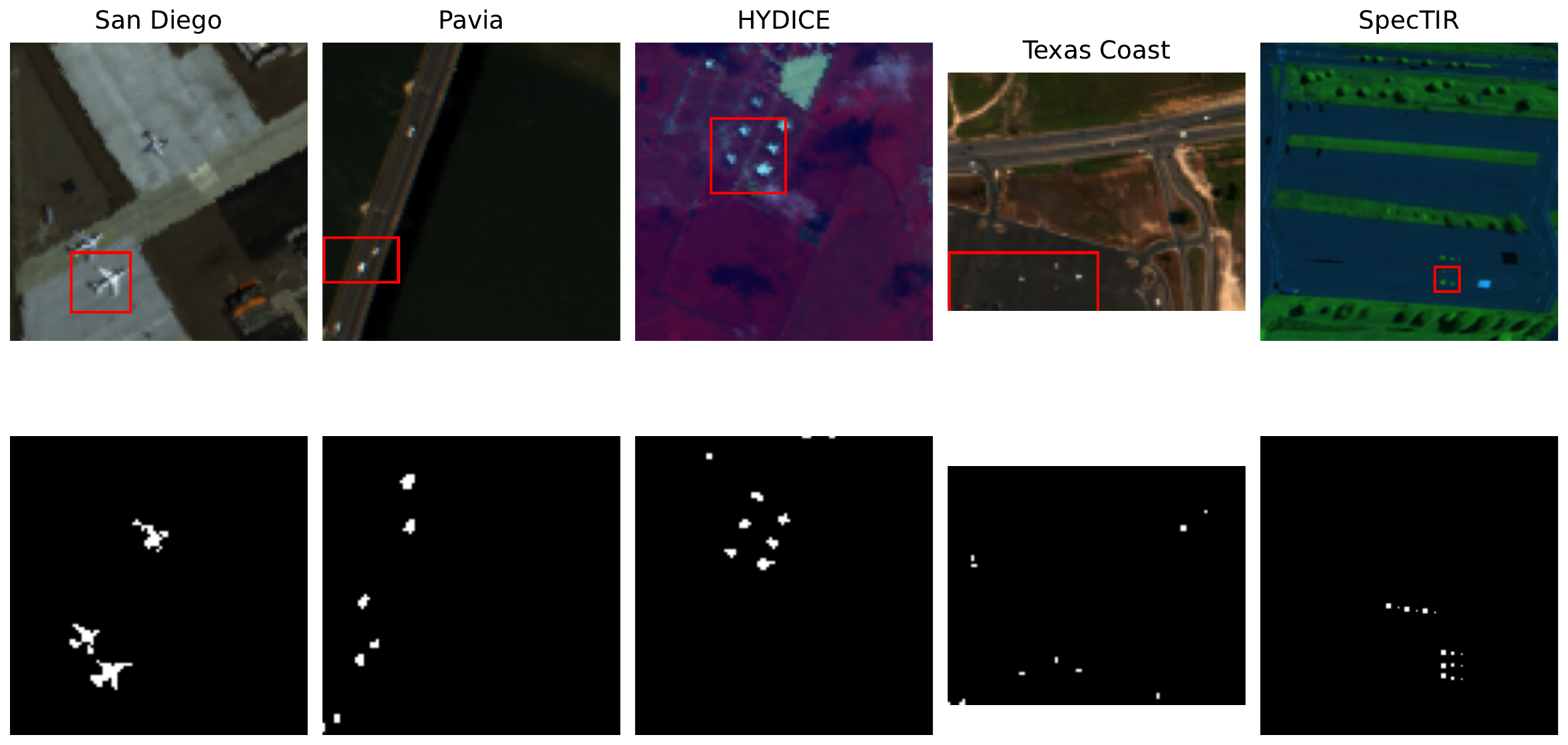}
  \caption{Hyperspectral images used in detection. The red boxes highlight the sub-images we chose to conduct the detections.}
  \label{fig:hyperspectral_anomaly_visualization}
\end{figure}

To compare the performance of different detectors, we report the area under the curve(AUC) values of these binary classifiers in Table~\ref{tab:anomaly-detection-AUC}. We observe that the RX-detectors induced by SCOPEs outperform others throughout all five images, and the advantage is more significant when the sample covariance performs poorly. This highlights the suitability of our estimators in high-dimensional regimes, where sample covariance fails to give a meaningful estimation of true covariance.

\begin{table}[H]
\centering
\begin{tabular}{lccccc}
\hline
Estimator & San Diego & Pavia & HYDICE & Texas Coast & SpecTIR \\ \hline
Sample & 0.7718 & 0.8866 & 0.9980 & 0.8891 & 0.9621 \\ 
LW-L & 0.9529 & 0.9567 & 0.9846 & 0.8666 & 0.9940 \\ 
LW-NL & 0.9132 & 0.9693 & 0.9971 & 0.9100 & 0.8191 \\ 
W-DRCOE &0.7719  &0.8867  &0.9969  &0.9035 &0.9621 \\ 
WISE &0.7718 &0.8866  &0.9980  &0.8891 &0.9621 \\ 
KL-SCOPE & 0.9794 & \textbf{0.9893} & 0.9976 & \textbf{0.9101} & 0.9940 \\ 
W-SCOPE & {0.7803} & {0.8955} & \textbf{0.9986} & {0.8891} & 0.9687 \\ 
SS-SCOPE & \textbf{0.9760} & 0.9824 & {0.9982} & 0.9006 & \textbf{0.9956} \\ \hline
\end{tabular}
\caption{AUC values of the detectors induced by covariance estimators under five hyperspectral images.}\label{tab:anomaly-detection-AUC}
\end{table}

\subsection{Application in A/B Tests}

In this experiment, we consider the application of the covariance estimator in constructing a sensitive metric of A/B tests. A/B testing is an online controlled experiment where users are split into two groups—the control (A) and the treatment (B)  to compare performance between a current production system (software, user interface, etc.) and a new variant. This method uses statistical hypothesis tests to determine whether differences in user engagement metrics are significant enough to infer a true user preference. In modern usage of A/B tests, there are always several relevant features (of the system) being taken into consideration, and it is essential to find a combination of these features that is sensitive to model changes. Let $\mu_A\in \R^p$ be the mean of feature vectors collected from group A, and $\mu_B\in \R^p$ be the mean of feature vectors collected from group B, both with $p$ features. The literature aims to find a linear combination of features $\theta(x)\Let w^\transpose x$ that is sensitive to model performance improvements when switching from A to B. To avoid confusion and fix the direction of the combination coefficient $w$, we stipulate that $\theta(\mu_B) > \theta(\mu_A)$ when the treatment model (B) performs better than the control model (A), and denote it by $B\succ A$.

In \cite{ref:kharitonov2017learning}, the authors propose a machine learning framework that learns sensitive feature combinations from historical data of multiple experiments with known preferences. Without loss of generality, we assume $B_k\succ A_k$ for all $K$ experiments. Let $x_i^k, y_j^k\in \R^p, i=1,...,n_A, j=1,...,n_B$ be feature samples collected from group $A_k$ and group $B_k$, respectively. Let $\wh\mu_{A_k}, \wh \mu_{B_k}$ and $\wh \Sigma_{A_k}, \wh\Sigma_{B_k}$ be the sample mean vectors and sample covariance matrices of group $A_k$ and group $B_k$. The z-score of pair $(A_k, B_k)$ under combination coefficient $w$ is
\begin{align}\label{eq:z-score}
    z(w, A_k, B_k) = \frac{{w}^\transpose(\wh\mu_{B_k}-\wh\mu_{A_k})}{\sqrt{{w}^\transpose \wh\Sigma_{A_k} w + {w}^\transpose \wh\Sigma_{B_k} w}}. 
\end{align}
Then the optimal combination coefficient for experiment $k$ in the sense of maximizing the z-score of combined features is given by 
\begin{align}\label{eq:w-k-opt}
    w_k\opt = \alpha\cdot (\wh \Sigma_{A_k}+\wh\Sigma_{B_k}+ \rho I)^{-1} (\wh \mu_{B_k}-\wh\mu_{A_k})
\end{align}
for any scalar $\alpha> 0$. 
The regularization term $ \rho I$ in~\eqref{eq:w-k-opt} is used to deal with the singularity or ill-conditioning of $\wh \Sigma_A+\wh\Sigma_B$, and it can be viewed as an application of Ledoit-Wolf's linear estimator.
Based on this, the optimal weights $w\opt$ calculated over experiments $1,\ldots,K$ is
\begin{align}\label{eq:w-opt}
    w\opt=\frac{1}{K} \sum_{k=1}^K \frac{w_k\opt}{\|w_k\opt\|}.
\end{align}
We have carried out tests on the performance of metric ${w\opt}^\transpose x$ when the regularization is substituted by a covariance matrix estimator, i.e., $w\opt=k\cdot (f(\wh \Sigma_A)+f(\wh\Sigma_B))^{-1} (\wh\mu_A- \wh \mu_B)$ for different choice of estimator $f(\cdot)$.

Our goal is to test whether the resulting metric ${w\opt}^\transpose x$ can distinguish between A/B groups with significant differences and those without significant differences.
Because of the limitations of open-source data, we generate synthetic data in the following ways. We consider A/B test problems with 50 features. In each round, the training data consists of 100 pairs of experiments $B_k\succ A_k, k=1,\ldots,100$, and the testing data consists of 200 pairs of experiments, with 60\% of them being recognizable A/B pairs and the remaining being A/A pairs without significant differences. All the training control groups A consist of $n$ multivariate normal samples with mean $0$ and covariance $\Sigma_1$, and the testing control groups A consist of 200 multivariate normal samples with the same mean and covariance. All the training treatment groups B consist of $n$ multivariate normal samples with mean $\mu$ and covariance $\Sigma_2$, and the testing treatment groups B consist of 200 multivariate normal samples with mean $\mu$ and covariance $\Sigma_2$
\footnote{We choose the parameters by $\mu=\xi, \Sigma_1=\sum_{i=1}^{100} \xi_i\xi_i^\transpose,\Sigma_2=\sum_{i=101}^{200} \xi_i\xi_i^\transpose$, where $\xi, \xi_1,\ldots,\xi_{400}$ are all drawn from multivariate normal distribution.}.
The sample size $n$ is set to different values throughout the experiments to test the performance of the metric under different sample size regimes. The optimal coefficient $w\opt$ is learned using the training data $A_k, B_k, k=1,\ldots,100$ by~\eqref{eq:w-k-opt} and~\eqref{eq:w-opt}, with $(\wh \Sigma_{A_k}+\wh\Sigma_{B_k}+ \rho I)^{-1}$ being replaced with $(f(\wh \Sigma_{A_k})+f(\wh\Sigma_{B_k}))^{-1}$, for different choice of covariance matrix estimator $f(\cdot)$. For each pair of groups A/Y from testing data, where $Y$ can be A or B, let $x_i, y_j, i=1,...,n_A, j=1,...,n_Y$ be samples from group A and group Y, respectively, and $\wh \mu_A, \wh \mu_Y$ and $\wh\Sigma_A, \wh\Sigma_Y$ be the corresponding sample means and sample covariances. Then we calculate the z-score of A/Y:
\[
    z(w\opt, A, Y) = \frac{{w\opt}^\transpose(\wh\mu_Y-\wh\mu_A)}{\sqrt{{w\opt}^\transpose \wh\Sigma_A w\opt + {w\opt}^\transpose \wh\Sigma_Y w\opt}}. 
\]
We use the z-score as a criterion to infer whether $Y=A$ or $Y=B$. Note that a larger z-score indicates that the difference between Y and A is significant, i.e., more likely that we have $Y=B$. The inference task based on the z-score can be viewed as a binary classification problem. Thus, we use the area under the curve (AUC) to justify whether the z-score induced by different covariance estimators performs well. We repeat the above process 20 times for each choice of sample size $n$, and record the average AUCs in Table~\ref{tab:ab-test-AUC}. It is shown that as the sample size increases, the performance of all estimators becomes better, and the SCOPEs perform the best when $n > p$.
\begin{table}[H]
\centering
\begin{tabular}{lccccc}
\hline
Sample size & $n=50$ & $n=60$ & $n=70$ & $n=100$ & $n=120$ \\ \hline
Sample & 0.7620 &0.8228 & 0.8067 & 0.8494 & 0.8506 \\ 
LW-L & \textbf{0.7895} &0.8193 & 0.8072 & 0.8456 & 0.8460 \\ 
LW-NL & 0.5002 & 0.8139 & 0.8015 & 0.8430 & 0.8463 \\ 
W-SCOPE & 0.7620 &0.8246 & 0.8105 & \textbf{0.8501} & 0.8502 \\ 
KL-SCOPE & 0.7729 &0.8225 & 0.8083 & 0.8459 & 0.8503 \\ 
SS-SCOPE & 0.7675 &\textbf{0.8256} & \textbf{0.8108} & 0.8494 & \textbf{0.8514} \\ 
W-DRCOE &0.7137 &0.8143 & 0.7995 &0.8372  & 0.8388 \\ \hline
\end{tabular}
\caption{Average AUC of z-score classifier induced by different covariance estimators. The results are obtained from 20 independent repetitions, and the AUCs are given in the format of mean(variance).}\label{tab:ab-test-AUC}
\end{table}

\subsection{Application in Minimum Variance Portfolio Selection}
We consider the minimum variance portfolio selection problem~\cite{ref:jagannathan2003risk}:
\begin{align}\label{prob:opt-portfolio}
    \min_{w\in\R_+^p} w^\transpose \Sigma_0 w, \;\;\st\; w^\transpose \mathds{1} = 1, 
\end{align}
where $w\in \R_+^p$ are the weights allocated to $p$ different risky assets, $\mathds{1}$ denotes the vector of all ones, $\Sigma_0$ is the covariance matrix of the random returns of assets, and the objective $w^\transpose \Sigma_0 w$ represents the variance of the portfolio return. Here, we do not allow short selling by setting $w\geq 0$. The solution of~\eqref{prob:opt-portfolio} corresponds to the portfolio that minimizes the risk. In practice, the true covariance $\Sigma_0$ is unknown and can only be estimated from historical return data. 

In this section, we compare the performance of various covariance matrix estimators in the minimum-variance portfolio allocation application~\eqref{prob:opt-portfolio}. We use the ``48 industry portfolios'' dataset from the Fama-French online library\footnote{\url{https://mba.tuck.dartmouth.edu/pages/faculty/ken.french/Data_Library.html}}. The dataset contains the monthly returns of 48 portfolio groups by industry, e.g., agriculture, machinery, or communication. We use data from January 1986 to December 1995 for parameter tuning and data from January 1996 to August 2025 for testing. For both the tuning step and testing step, we adopt the following rolling horizon procedure: First, we estimate the covariance matrix from the historical returns within a rolling estimation window of 60 months (5 years), and construct the minimum variance portfolio. We then compute the return of the portfolio in the next month. In the month after that, we shift the rolling window forward by one month, re-estimate the covariance matrix from the historical returns of the past 60 months again, and rebalance the portfolio according to the newest estimation of the covariance matrix. The re-estimation and rebalancing are conducted every month. In the radius tuning step, we search over $[10^{-3}, 5\times 10^3]$ and choose the radius that yields the highest average portfolio return. In the testing step, we record the monthly returns of the portfolio induced by every covariance matrix estimator. Figure~\ref{fig:industry-cum} displays the cumulative returns of all the portfolios. 

To further compare the performance of the different estimators, Table~\ref{tab:industry} reports the average returns, Sharpe ratios, Sortino ratios, and cumulative returns of the portfolios induced by the estimators. It is shown that the SCOPEs achieve the highest average return and cumulative return, and maintain the Sharpe ratio and Sortino ratio comparable to or better than the other estimators. This result indicates that the proposed shrinkage estimators, through suitable spectral bias correction, effectively balance the risk and the return of the market, and thus produce portfolios with higher quality.

\begin{figure}[htbp]
  \centering
  \includegraphics[width=0.9\textwidth]{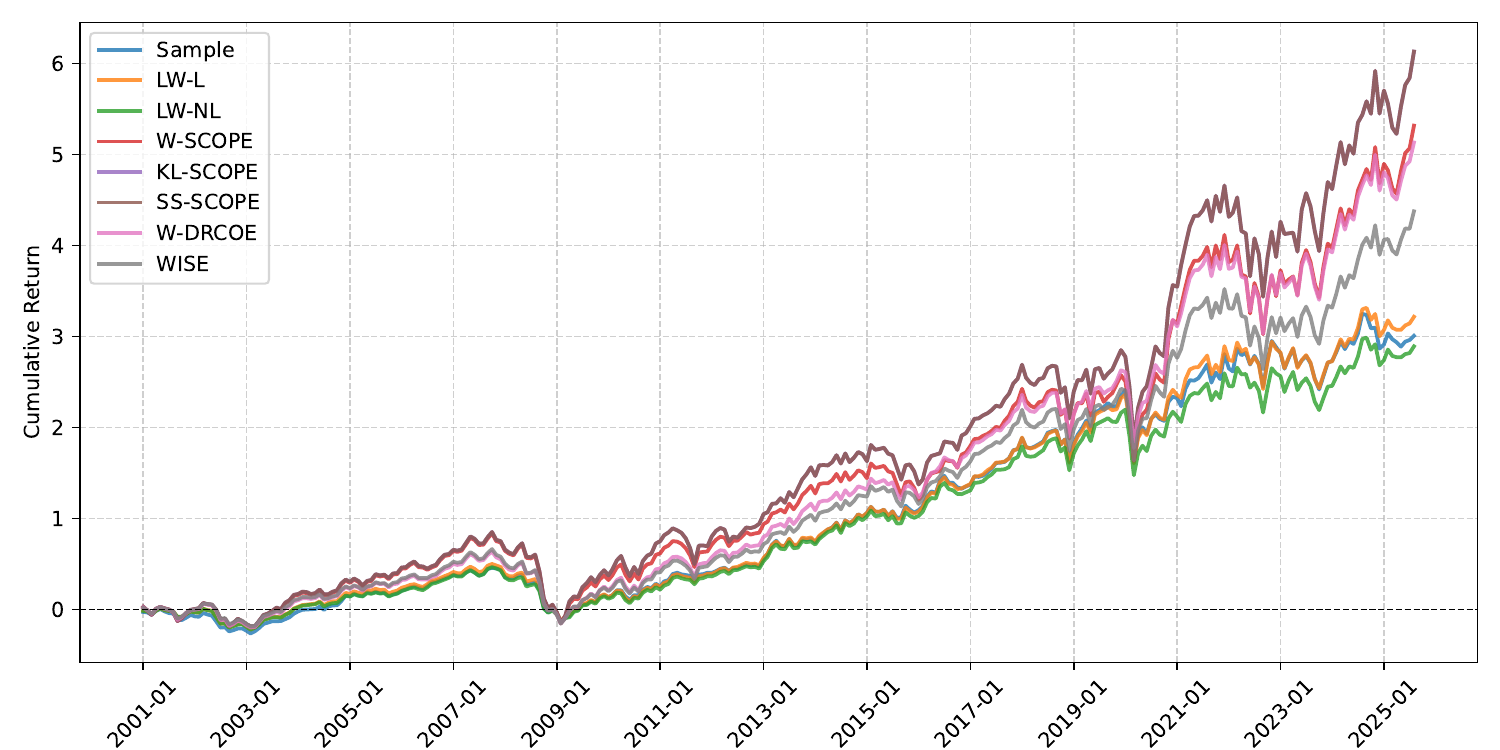}
  \caption{Cumulative returns of the portfolios induced by different estimators.}
  \label{fig:industry-cum}
\end{figure}

\begin{table}[htbp]
\centering
\begin{tabular}{lcccc}
\hline
Estimator & Average Return (\%) & Sharpe Ratio & Sortino Ratio & Cumulative Return (\%) \\
\hline
Sample    & 5.79 & 0.5314 & 0.6593 & 300.55 \\
LW-L      & 6.01 & 0.5491 & 0.6768 & 321.48 \\
LW-NL     & 5.66 & 0.5239 & 0.6463 & 289.00 \\
W-DRCOE   & 7.63 & \textbf{0.5862} & \textbf{0.7596} & 513.02 \\
WISE  & 7.05 & 0.5615 & 0.7101 & 437.38 \\
KL-SCOPE    & \textbf{8.29} & 0.5597 & 0.7532 & 612.88 \\
WA-SCOPE    & 7.76 & 0.5477 & 0.7158 & 531.54 \\
SS-SCOPE    & \textbf{8.29} & 0.5597 & 0.7533 & \textbf{612.91} \\
\hline
\end{tabular}
\caption{Performance summary of the portfolios induced by different estimators.}
\label{tab:industry}
\end{table}

\section{Concluding Remarks}

In this paper, we propose SCOPE, a novel distributionally robust framework for the simultaneous estimation of covariance and precision matrices. By jointly minimizing the worst-case Frobenius loss and Stein’s loss over a divergence-based ambiguity set, we derive a convex optimization formulation that yields a quasi-analytical, nonlinear shrinkage estimator. This estimator corrects spectral bias on both ends, thereby enhancing the condition number and numerical stability of the estimated matrices.

The shrinkage target and intensity are governed by two interpretable parameters, $\tau$ and $\rho$. We introduce a new parameter-tuning approach inspired by inverse optimization and demonstrate that the asymptotically optimal radius scales as $O(n^{-2})$. Extensive numerical experiments support our theoretical findings and show that SCOPE achieves competitive or superior empirical performance compared to state-of-the-art methods in real-world applications.

An interesting direction for future work is the development of a data-driven calibration scheme for the shrinkage radius. Although the asymptotic convergence rate $\rho_n \sim c/n^2$  serves as a useful theoretical guideline, it cannot be directly applied in empirical settings for two main reasons. First, the true distribution of the random vector $\xi$ is generally unknown, making it impossible to derive closed-form expressions for the constants in equations \eqref{eq:optimal-eps-limit-KL}-\eqref{eq:optimal-eps-limit-SS}. Second, in finite-sample regimes, the optimal radius may deviate from its asymptotic counterpart, rendering purely asymptotic tuning potentially suboptimal.
Nevertheless, the asymptotic form $\rho_n \sim c/n^2$ offers a valuable starting point. Since the constant c is invariant across sample sizes, it may be estimated via regression models, which could then guide the tuning of $\rho_n$. Developing such a data-driven calibration method entails significant theoretical and empirical work, which we leave for future research.

\newpage
\bibliography{ref}
\bibliographystyle{chicago}

\clearpage

\appendix
\section*{Appendix}

\section{Side Results}

\begin{lemma}[Unbounded problem under singularity] \label{lemma:unbounded}
    If $\wh\Sigma$ is singular, then the optimal value
    \begin{equation*}
        \begin{array}{cl}
        \min &  -\log \det X + \la X,\wh\Sigma\ra + \frac{\tau}{2} \left(\|\Sigma\|_F^2-2\la\Sigma,\wh\Sigma\ra \right) \\
        \mathrm{s.t.} & \Sigma\in \PD,~X\in \PD,~X\Sigma=I.
        \end{array}
    \end{equation*}
    is unbounded below.
\end{lemma}
\begin{proof}[Proof of Lemma~\ref{lemma:unbounded}]
Let $\wh\Sigma = Q^\transpose \Lambda Q$ be the spectral decomposition of $\wh\Sigma$ with 
$\lambda_1\geq\ldots\geq\lambda_r>\lambda_{r+1}=\ldots=\lambda_p=0$ being eigenvalues of $\wh\Sigma$. Let $\Sigma_\eta = \wh\Sigma + \eta I$ 
and $X_\eta=\Sigma_\eta^{-1}$ for $\eta>0$.  Then $(\Sigma_\eta, X_\eta)$ is a feasible solution. The objective function at the feasible point $(\Sigma_\eta, X_\eta)$ can be evaluated by
\begin{align*}
    &-\log\det (\wh\Sigma + \eta I)^{-1} + \la (\wh\Sigma + \eta I)^{-1},\widehat{\Sigma}\ra +\frac{1}{2}\tau \|\eta I\|_F^2 - \frac{1}{2}\tau \|\wh\Sigma\|_F^2\\ \;=\; &\sum_{i=1}^p \log (\lambda_i+\eta) + \tr\left(Q^\transpose \Lambda Q (Q^\transpose(\Lambda+\eta I)Q)^{-1}\right) + \frac{1}{2}\tau \eta p - \frac{1}{2}\tau \|\wh\Sigma\|_F^2\\ \;=\; &\sum_{i=1}^r \log (\lambda_i+\eta) + \sum_{i=r+1}^p \log \eta + \tr\left(\Lambda(\Lambda+\eta I)^{-1}\right) + \frac{1}{2}\tau \eta p- \frac{1}{2}\tau \|\wh\Sigma\|_F^2\\\;=\; &\sum_{i=1}^r \log (\lambda_i+\eta) + \sum_{i=r+1}^p \log \eta + \sum_{i=1}^r \frac{\lambda_i}{\lambda_i+\eta}+ \frac{1}{2}\tau \eta p- \frac{1}{2}\tau \|\wh\Sigma\|_F^2.
\end{align*}
The function value goes to $-\infty$ as $\eta\downarrow 0$. 
\end{proof}

\section{Proofs}
\subsection{Proof of Theorem~\ref{thm:convex-reform}}
We first show that the optimal solution to~\eqref{prob:P-Mat} exists and is unique. To see the uniqueness, we note that the objective function of~\eqref{prob:P-Mat} is continuous and strongly convex, and the feasible set is convex, and thus,~\eqref{prob:P-Mat} admits at most one optimal solution. Now we show the existence. Let 
\[
    {\cal L} \Let \left\{\Sigma\in \PSD: D(\Sigma, \widehat{\Sigma})\leq  \rho\right\}\cap \left\{\Sigma\in \PSD:\log\det \Sigma - \frac{1}{2}\tau \|\Sigma\|_F^2 \geq \log\det \wh\Sigma - \frac{1}{2}\tau \|\wh\Sigma\|_F^2 \right\}.
\]
and consider the problem
\begin{align}\label{prob:P-Mat-L}
    \max_{\Sigma\in {\cal L}}~\log\det \Sigma - \frac{1}{2}\tau \|\Sigma\|_F^2.
\end{align}
Since $D(\cdot,\wh\Sigma)$ is continuous, convex and coercive, 
the set $\{\Sigma\in \PSD: D(\Sigma,\wh\Sigma)\leq \rho\}$ is nonempty, compact and convex. 
Moreover, 
since $\log\det \Sigma - \frac{1}{2}\tau \|\Sigma\|_F^2$ is continuous and strongly concave over $\PD$, 
then ${\cal L}$ is a convex and compact set, and consequently,~\eqref{prob:P-Mat-L}
admits a unique solution, denoted by $\widetilde{\Sigma}$. Since
\[
    \log\det \Sigma - \frac{1}{2}\tau \|\Sigma\|_F^2 <\log\det \wh\Sigma - \frac{1}{2}\tau \|\wh\Sigma\|_F^2, \forall \Sigma\in \left\{ \Sigma\in \PSD: D(\Sigma, \widehat{\Sigma})\leq  \rho\right\}\setminus {\cal L},
\]
then $\widetilde{\Sigma}$ is also an optimal to~\eqref{prob:P-Mat} and this solution is unique.  

Next, we show that~\eqref{prob:robust-model}, namely the problem
\[
\min_{\substack{\Sigma,X \in \PD \\ X \Sigma = I}} ~\max\limits_{S \in \PSD: D(S, \widehat{\Sigma})\leq  \rho}~\left\{ f(\Sigma, X, S) = -\log\det X+\la X,S\ra +\frac{1}{2}\tau \left(\|\Sigma\|_F^2-2\la \Sigma, S\ra\right) \right\}
\]
admits a unique optimal solution $(\Sigma\opt, (\Sigma\opt)^{-1})$. Let 
\begin{align*}
    F(\Sigma,X) &\;\Let\; \max\limits_{S \in \PSD: D(S, \widehat{\Sigma})\leq  \rho}~\left\{ f(\Sigma, X, S) = -\log\det X+\la X,S\ra +\frac{1}{2}\tau \left(\|\Sigma\|_F^2-2\la \Sigma, S\ra\right) \right\}\\
    &\;=\; -\log\det X + \frac{1}{2}\tau\|\Sigma\|_F^2 + \max\limits_{S \in \PSD: D(S, \widehat{\Sigma})\leq  \rho} \la X, S\ra - \tau \la \Sigma,S\ra.
\end{align*}
The function is strictly convex in $(\Sigma,X)$. Thus it admits at most one minimizer over $\PD\times \PD$. Let $\Sigma^*$ be the optimal solution to 
\eqref{prob:P-Mat}.
In the next Proposition~\ref{prop:saddle-point}, we
show that the pair $(\Sigma\opt, (\Sigma\opt)^{-1})$ is the unique optimal solution to $\min_{{\Sigma,X \in \PD}}~F(\Sigma,X)$. Since $\Sigma\opt X\opt = I$, we conclude that $(\Sigma\opt, X\opt)$ is also the unique optimal solution to~\eqref{prob:robust-model}.

\begin{proposition}[Saddle point of $f$ and minimizer of $F$]\label{prop:saddle-point}
    If Assumptions~\ref{ass:regularity-nominal} and~\ref{ass:convex-divergence} hold, then the tuple $(\Sigma = \Sigma\opt, X = (\Sigma\opt)^{-1}, S=\Sigma\opt)$ is a saddle point of function $f(\Sigma, X, S)$, i.e., for any $(\Sigma, X) \in \PD \times \PD$,
    and $S\in \left\{S\in\PSD:D(S, \wh\Sigma)\leq  \rho\right \}$,
    \begin{align}
        \label{eq:saddle}
        f\left(\Sigma\opt, (\Sigma\opt)^{-1}, S\right)
        \leq
        f\left(\Sigma\opt, (\Sigma\opt)^{-1},\Sigma\opt\right)
        \leq f\left(\Sigma,X,\Sigma\opt\right).
    \end{align}
    The saddle point is 
    an optimal solution to $\min_{\Sigma,X \in \PD}~F(\Sigma,X)$.
\end{proposition}

\begin{proof}[Proof of Proposition~\ref{prop:saddle-point}]
For fixed $\Sigma=\Sigma\opt\in \PD$, a basic calculation of the first-order optimality condition shows that 
$(\Sigma\opt,(\Sigma\opt)^{-1})$ 
is the minimizer of $\min_{\Sigma,X\in \PD} f(\Sigma,X,\Sigma\opt)$.
Thus, the second inequality
of~\eqref{eq:saddle} holds.

Next, we show the first inequality. To this end, we show 
that $\Sigma\opt$ is the optimal solution to 
problem 
\begin{align}
\label{prob:saddle-inner}
    \max\limits_{\Sigma \in \PSD: D(\Sigma, \widehat{\Sigma})\leq  \rho}~f(\Sigma\opt,(\Sigma\opt)^{-1},\Sigma),
\end{align}
which is equivalent to 
\begin{align}\label{prob:S-opt-2}
    \max\limits_{\Sigma \in \PSD: D(\Sigma, \widehat{\Sigma})\leq  \rho} \la (\Sigma\opt)^{-1}, \Sigma\ra - \tau \la \Sigma\opt, \Sigma\ra.
\end{align}
Note that $D(\wh\Sigma,\wh\Sigma)=0<\rho$ by Assumptions~\ref{ass:regularity-nominal} and~\ref{ass:convex-divergence}(a). By~\cite[Section 5.2.3]{ref:boyd2004convex}, Slater's condition ensures that strong duality for~\eqref{prob:P-Mat} and~\eqref{prob:S-opt-2} both hold.
Thus $\Sigma\opt$ is an optimal solution to problem~\eqref{prob:S-opt-2} if and only if $\Sigma=\Sigma\opt$ solves the following KKT system:
\begin{equation} \label{eq:KKT}
\begin{array}{r}
    \Sigma^{-1} - \tau \Sigma-\gamma {\nabla_1 D }(\Sigma, \wh \Sigma) = \bm 0, \\
    \gamma(D(\Sigma, \wh \Sigma)- \rho) = 0,\\
    D(\Sigma, \wh\Sigma) \leq  \rho,\\
    \gamma \geq 0, \Sigma \in \PSD. 
\end{array}
\end{equation}
On the other hand, the strong duality of~\eqref{prob:P-Mat} means that $\Sigma\opt$ satisfies the following KKT system \eqref{eq:KKT}. Thus $\Sigma=\Sigma\opt$ solves~\eqref{prob:saddle-inner} and the first inequality of~\eqref{eq:saddle} holds.
\end{proof}

\subsection{Proof of Proposition~\ref{prop:unbinding}}
\begin{proof}[Proof of Proposition~\ref{prop:unbinding}]
    Let $\ell(\Sigma)\Let \Sigma - \half \tau \| \Sigma \|_F^2$. 
    Then $\nabla \ell(\Sigma) = \Sigma^{-1} - \tau \Sigma$ and 
    $\nabla \ell\left(\sqrt{\frac{1}{\tau}} I\right) = 0$. Since 
    $\ell(\Sigma)$ is strongly concave (see e.g.~\citet[Page 74]{ref:boyd2004convex}), then $\Sigma\opt =\sqrt{\frac{1}{\tau}} I$ is the global maximizer of $\ell(\Sigma)$. The conclusion follows as  $\Sigma\opt$ lies in the feasible set of~\eqref{prob:P-Mat}.
\end{proof}

\subsection{Proof of Proposition~\ref{prop:unique-varphi}}
\begin{proof}[Proof of Proposition~\ref{prop:unique-varphi}]
    To ease the notation, we denote $d(\cdot,b)$ by $d_b(\cdot)$ for any $b\geq 0$ in this proof. Let $f(a) \Let \frac{1}{a}-\tau a-\gamma d^\prime_b(a)$. Since $d_b(\cdot)$ is convex and differentiable over $\R_{++}$, then  $d_b^\prime(\cdot)$ is non-decreasing over $\R_{++}$.
    Moreover,  $b\geq 0$, and $d^\prime_b(b)=0$ (see Remark~\ref{remark:minimizer-of-b}). Thus  $\lim_{a\downarrow 0}d^\prime_b(a) \leq 0$. Together with the fact that $\lim_{a\downarrow 0} \frac{1}{a}=+\infty$, we have that $\lim_{a\downarrow 0} f(a) =+\infty$. Note that $f$ is strictly decreasing over $\R_{++}$ and $\lim_{a\to\infty} f(a) = -\infty$, we conclude that $f(a) = 0$ admits a unique solution over $\R_{++}$.
\end{proof}
\subsection{Proof of Proposition~\ref{prop:equivalence-Mat-Vec}}
The proof is divided into two steps. We first show that~\eqref{prob:P-Mat} is equivalent to the following problem
\begin{equation}
\begin{array}{cl}
     \displaystyle\max_{s\in\R^p_+}\;\;&\displaystyle\sum_{i=1}^p \log s_i - \frac{1}{2}\tau \sum_{i=1}^p s_i^2 \\
    \st\;\; &\displaystyle\sum_{i=1}^p d(s_i, \widehat{\lambda}_i)\leq  \rho\\
    &s_1\leq s_2\leq  \ldots s_p,\tag{$\text{P-Vec}^{\uparrow}$}
\end{array}\label{prob:vector-uparrow}
\end{equation}
in the sense of Proposition~\ref{prop:equivalence-Mat-Vec-upper}. In the second step, we show the equivalence of~\eqref{prob:vector-uparrow} and~\eqref{prob:vector} in the sense of Proposition~\ref{Prop:P-Vec==P-Vec_uparrow}. Combining these two propositions readily proves Proposition~\ref{prop:equivalence-Mat-Vec}.

\begin{proposition}[Equivalence of~\eqref{prob:P-Mat} and~\eqref{prob:vector-uparrow}]\label{prop:equivalence-Mat-Vec-upper}
    If Assumptions~\ref{ass:regularity-nominal}-\ref{ass:Spectral-divergence} hold, then the following assertions hold.
    \begin{enumerate}[label=(\roman*)]
        \item Problem~\eqref{prob:P-Mat} is feasible if and only if problem~\eqref{prob:vector-uparrow} is feasible.
        \item The feasible set of~\eqref{prob:P-Mat} is compact if and only if the feasible set of~\eqref{prob:vector-uparrow} is compact.
        \item If $s\opt$ is an optimal solution to~\eqref{prob:vector-uparrow}, then $\wh V\diag(s\opt){\wh V}^\transpose$ is an optimal solution to~\eqref{prob:P-Mat}.
        \item If $\Sigma\opt$ is an optimal solution to~\eqref{prob:P-Mat}, then $\lambda(\Sigma\opt)$ is an optimal solution to~\eqref{prob:vector-uparrow}. 
        \item The optimal values of~\eqref{prob:P-Mat} and~\eqref{prob:vector-uparrow} are equal.
    \end{enumerate}
\end{proposition}
\begin{proof}[Proof of Proposition~\ref{prop:equivalence-Mat-Vec-upper}]
    The result is similar to~\cite[Proposition 9]{ref:yue2024geometric}. The main difference is that the objective becomes $\ell(\Sigma)=\log\det \Sigma - \frac{1}{2}\tau\|\Sigma\|_F^2$. Note that $\ell$ is strongly concave, and the value of $\ell$ only depends on the eigenvalues of $\Sigma$. That means the nature of the problem does not change, and thus, with a slight modification, proof of~\cite[Proposition 9]{ref:yue2024geometric} works here. We omit the details. 
\end{proof}

\begin{proposition}
\label{Prop:P-Vec==P-Vec_uparrow}
    Under Assumptions~\ref{ass:regularity-nominal}-\ref{ass:Spectral-divergence},~\eqref{prob:vector} admits a unique optimal solution, denoted by $s\opt$, and $s\opt$ is the unique optimal solution to~\eqref{prob:vector-uparrow}.
\end{proposition}
\begin{proof}
    We proceed with the proof in two steps.

    \textbf{Step 1}. We show that problem~\eqref{prob:vector} admits a unique optimal solution. To this end, we show that the feasible set 
    $$
    \mathcal{S}\Let\{s\in\R^p_+:\sum_{i=1}^p d(s_i,\wh\lambda_i)\leq  \rho\}
    $$ 
    is convex and compact and the objective is strictly convex over the set. Convexity of the set is implied by the convexity of $d$. To show compactness, it suffices to show closeness and boundedness of the set. The former is guaranteed by the continuity of $d$ under Assumption~\ref{ass:Spectral-divergence}~(ii). To show the boundedness, we note that by Remark~\ref{remark:minimizer-of-b}, $d(s_i,\wh\lambda_i) = D(s_i I, \wh\lambda_i I) / p$. Consequently, the feasible set can be written as $$
    \mathcal{S}=\{s\in\R^p_+:\sum_{i=1}^p D(s_i I, \wh\lambda_i I) / p\leq  \rho\}.
    $$
    Since $p>1$, then 
    $$
    S\subset \{s_i\in\R_+: D(s_i I,\wh\lambda_i I)\leq  \rho\}.
    $$
    Assumption~\ref{ass:convex-divergence} (iii) ensures that the latter is bounded, which implies that $\mathcal{S}$ is bounded as desired. Since the objective function is continuous and strictly convex over the feasible set, the existence of a unique optimal solution is evident. 

    \textbf{Step 2}. Let $s\opt$ be the optimal solution to problem~\eqref{prob:vector}.
    We show that $s\opt={s\opt}^{\uparrow}$, and thus $s\opt$ is also optimal to~\eqref{prob:vector-uparrow}. Let $q(s)\Let \sum_{i=1}^p \log s_i - \tau \sum_{i=1}^p s_i^2$ be the objective function of~\eqref{prob:vector}. Assume for the sake of a contradiction that $s\opt\neq {s\opt}^\uparrow$. Then by the feasibility of $s\opt$ and Lemma~\ref{lemma:rearrangement-ineq}, 
    \begin{equation}
        \sum_{i=1}^p d({{s_i\opt}^\uparrow}, \widehat{\lambda}_i)<  \sum_{i=1}^p d({{s_i\opt}}, \widehat{\lambda}_i)\leq   \rho,
    \end{equation}
    which means ${s\opt}^\uparrow$ is feasible to~\eqref{prob:vector}. Moreover, $q(s\opt)=q({s\opt}^\uparrow)$, which means ${s\opt}^\uparrow$ is also optimal to~\eqref{prob:vector}. This is a contradiction to the fact that $s\opt$ is the unique optimal solution. Finally, we note that the feasible set of~\eqref{prob:vector-uparrow} is a subset of that of~\eqref{prob:vector}, and thus $s\opt$ is the unique optimal solution to~\eqref{prob:vector-uparrow}.
\end{proof}

\begin{lemma}[{\citet[Lemma 5]{ref:yue2024geometric}}]\label{lemma:rearrangement-ineq} If Assumptions~\ref{ass:convex-divergence}-\ref{ass:Spectral-divergence} hold, then 
\[
    \sum_{i=1}^p d(s_i^{\uparrow}, y_i^{\uparrow}) \leq \sum_{i=1}^p d(s_i, y_i^{\uparrow}) \;\;\forall x, y\in \R_+^p.
\]
If the right-hand side is finite, then 
the equality holds if and only if $s=s^{\uparrow}$.
\end{lemma}

\subsection{Proof of Proposition~\ref{prop:pro-of-varphi}}
To ease the notation, we denote $d(\cdot,b)$ by $d_b(\cdot)$ for any $b\geq 0$ in the following proof.

\begin{lemma}[Derivative of $d_b$]\label{lemma:derivative-of-db} 
Let Assumptions~\ref{ass:convex-divergence} and~\ref{ass:Spectral-divergence} hold and $b>0$ be any fixed positive constant. Then $d_b^\prime(a) < 0$ for $a\in (0,b)$ and $d_b^\prime(a) > 0$ for $a\in (b,\infty)$.
\end{lemma}

\begin{proof}
   Let $b>0$ be fixed. Since $d(\cdot, b)$ is convex over $\R_+$, then
    \[
        0=d(b,b) \geq d(a,b) + (b-a) d^\prime_b(a),\;\; \forall a\in \R_{++}.
    \]
Thus, for any $a\in (0,b)$,  
    $
        d_b^\prime(a) \leq -\frac{d(a,b)}{b-a} < 0,
    $
and for any $a>b$, $d_b^\prime(a) \geq \frac{d(a,b)}{a-b} > 0$.
    In both cases, we use the fact that $d(a,b) > 0$ for any $a>0, a\neq b$.
\end{proof}

\begin{proof}[Proof of Proposition~\ref{prop:pro-of-varphi}]
    Part (i). Let $f(a) = \frac{1}{a}-\tau a-\gamma d^\prime_b(a)$. We prove by the fact that $f$ is strictly decreasing over $\R_{++}$. If $b < \sqrt{\frac{1}{\tau}}$, we have 
    $$
    f(b) = \frac{1}{b}-\tau b> 0, f(1/\sqrt{\tau})=-\gamma d_b^\prime (1/\sqrt{\tau})\leq 0
    $$
    and thus $b< \varphi_{\tau, b}(\gamma)\leq \sqrt{\frac{1}{\tau}} $. If $b > \sqrt{\frac{1}{\tau}}$, then $f(b)< 0, f(1/\sqrt{\tau})\geq 0$ and thus $$
    \sqrt{\frac{1}{\tau}}\leq \varphi_{\tau, b}(\gamma)< b.
    $$
    If $b=1/\sqrt{\tau}$, then $f(1/\sqrt{\tau})=0$ and 
    $$
    \varphi_{\tau, b}(\gamma) = 1/\sqrt{\tau}.
    $$

    Part (ii). Let $H(\gamma, a)=\frac{1}{a}-\tau a - \gamma d_b^\prime(a)$. By Assumption~\ref{ass:Spectral-divergence}(ii), $d_b(a)$ is convex and $d_b^\prime(a)$ is continuously differentiable on $\R_{++}$. Thus 
    \[
        \frac{\partial H(\gamma, a)}{\partial a} = -\frac{1}{a^2}-\tau-\gamma d_b^{\prime\prime}(a) < 0, \forall a\in \R_{++}.
    \]
    Since $\varphi_{\tau,b}(\gamma)>0$ by Part (i), we can use classical implicit function theorem (see e.g.~\cite[Section 1.2]{dontchev2009implicit}) to assert that $\varphi_{\tau, b}$ is differentiable (and thus continuous) over $\R_{+}$ and
    \begin{align*}
        \varphi_{\tau, b}^\prime(\gamma) = -\frac{d_b^\prime(\varphi_{\tau, b}(\gamma))}{\frac{1}{a^2}+\tau+\gamma d_b^{\prime\prime}(\varphi_{\tau, b}(\gamma))}.
    \end{align*}
    The sign of $\varphi_{\tau, b}^\prime(\gamma)$ depends on the sign of $-d^\prime_b(a)$. By (i) and Lemma~\ref{lemma:derivative-of-db}, $\varphi_{\tau, b}(\gamma) > b$ if $b<1/\sqrt{\tau}$ and thus $-d^\prime_b(a)<0$ and $\varphi_{\tau, b}(\gamma) < b$ if $b>1/\sqrt{\tau}$ and thus $-d^\prime_b(a)>0$.

    Part (iii). Since $\varphi_{\tau,b}(\gamma)$ is continuous at $0$, then $$
    \lim_{\gamma\downarrow 0} \varphi_{\tau,b}(\gamma) = \varphi_{\tau,b}(0) = 1/\sqrt{\tau}.
    $$
    If $b=1/\sqrt{\tau}$, then $\varphi_{\tau,b}(\gamma)=1/\sqrt{\tau}$ for any $\gamma \geq 0$ and thus $\lim_{\gamma \to\infty}=b$. Now we assume $b\neq 1/\sqrt{\tau}$. Note that $\varphi_{\tau,b}(\gamma)$ is strictly increasing and bounded above if $b>1/\sqrt{\tau}$ and strictly decreasing and bounded below if $b<1/\sqrt{\tau}$. So in both cases, the limit is well defined. By definition, we have 
    \[
        d_b^\prime(\varphi_{\tau,b}(\gamma)) = \frac{1}{\gamma} \left(\frac{1}{a}-\tau a\right)
    \]
    for $\gamma>0$.
    Taking limits on both sides and exploiting the fact that $d_b^\prime$ is continuous gives
    \[
        0 = \lim_{\gamma\to\infty} d_b^\prime\left(\varphi_{\tau,b}(\gamma)\right)= d_b^\prime\left(\lim_{\gamma\to\infty} \varphi_{\tau,b}(\gamma)\right).
    \]
    By Lemma~\ref{lemma:derivative-of-db} and the fact that $d_b(a)$ is minimized by $a=b$, we prove that $\lim_{\gamma\to\infty} \varphi_{\tau,b}(\gamma) = b$.

    Proof of Part (iv). Let $a^\prime = \phi(b) = \frac{\varphi(\tau,\gamma,b)}{b}$. Then $a^\prime$ is the unique solution to the equation 
    \[
        \frac{1}{a^\prime b} - \tau a^\prime b -\gamma d_b^\prime(a^\prime b) = 0.
    \]
    By treating $\phi(b)$ as the solution to the equation above, we can use the classical implicit function theorem again to derive
    \begin{align*}
        \phi^\prime(b) = -\frac{-\frac{1}{a^\prime b^2}-\tau a^\prime -\gamma a^\prime d_b^{\prime\prime}(a^\prime b)}{-\frac{1}{{a^\prime}^2b}-\tau b - \gamma b d_b^{\prime\prime}(a^\prime b)} \leq 0,
    \end{align*}
    which completes the proof.
\end{proof}

\subsection{Proof of Proposition~\ref{prop:solve-prob-vector}}

\begin{proof}[Proof of Proposition~\ref{prop:solve-prob-vector}]
The uniqueness is established at Step 1 in the proof of Proposition \ref{Prop:P-Vec==P-Vec_uparrow}. 
  So we are left to 
  show that the unique optimal solution can be represented by 
  $s_i\opt=\varphi(\tau, \gamma\opt, \widehat{\lambda}_i),i=1,\ldots,p$, where $\gamma\opt$ is the unique solution of the nonlinear equation $\sum_{i=1}^p d(\varphi(\tau, \gamma\opt, \widehat{\lambda}_i), \widehat{\lambda}_i)- \rho=0$.
   Define the Lagrangian function of problem~\eqref{prob:vector}  
    \[
        L(\gamma, s) = \sum_{i=1}^p \log s_i - \frac{1}{2}\tau\sum_{i=1}^p s_i^2 -\gamma \left(\sum_{i=1}^p d(s_i,\widehat{\lambda}_i)- \rho\right).
    \]
  By~\citet[Theorem 28.3]{ref:Rockafellar1970},~\eqref{prob:vector} can be represented as
    \begin{align}
        \max_{s\in\R_+^p} \min_{\gamma \geq 0} L(\gamma, s) = \min_{\gamma \geq 0}\max_{s\in\R_+^p}  L(\gamma, s), \label{prob:vector-Lagrangian}
    \end{align}
    where the left-hand side is the primal problem and the right-hand side is the Lagrange dual problem.
    Equation~\eqref{prob:vector-Lagrangian} means that strong duality holds.  
Let $\gamma\opt$ be an optimal solution of the dual problem. Then
    \[
        \max_{s\in\R_+^p} L(\gamma\opt, s)= \gamma\opt \rho + \sum_{i=1}^p\max_{s_i>0}\left\{ \log s_i-\frac{1}{2}\tau s_i^2 - \gamma\opt d(s_i,\widehat{\lambda}_i)\right\}.
    \]
The problem above is decomposable, which means that it can be solved by solving the following univariate maximization problem 
    \begin{align}
        \max_{s_i>0}\;\; \log s_i-\frac{1}{2}\tau s_i^2 - \gamma\opt d(s_i,\widehat{\lambda}_i), \label{prob:univariate-prob}
    \end{align}
    for $i=1, \ldots ,p$.
    The first-order optimality condition of the problem above  
    \begin{align}
        \frac{1}{s_i}-\tau s_i-\gamma\opt\frac{\partial d}{\partial s_i}(s_i,\widehat{\lambda}_i)=0. \label{eq:first-order-optimality}
    \end{align}
    By Proposition~\ref{prop:unique-varphi},~\eqref{eq:first-order-optimality} always admits a unique strictly positive solution for 
    any 
    fixed $\tau>0,\gamma\opt\geq 0$ and $\widehat{\lambda}_i\geq 0$. Exploiting the notation defined in~\eqref{eq:def-varphi}, the optimal solution to~\eqref{prob:univariate-prob} is $s_i\opt=\varphi(\tau, \gamma\opt, \widehat{\lambda}_i)$.

    Next, we prove that $\sum_{i=1}^p d(\varphi(\tau, \gamma\opt, \widehat{\lambda}_i), \widehat{\lambda}_i)- \rho=0$. If $\rho = \rho_{\max}$, then $\gamma\opt=0$ solves the equation since $\varphi(\tau,0,\wh\lambda_i)=\sqrt{\frac{1}{\tau}}$ and thus $\sum_{i=1}^p d(\varphi(\tau, \gamma\opt, \widehat{\lambda}_i), \widehat{\lambda}_i) = \rho_{\max}$ by definition of $\rho_{\max}$. Let $\rho\in(0,\rho_{\max})$.    It suffices to prove $\gamma\opt>0$ since the complementarity condition in the KKT conditions will ensure the equality and the feasibility. Assume for the sake of a contradiction that $\gamma\opt=0$. Then by~\eqref{eq:first-order-optimality}, $\left(0, \left(\sqrt{\frac{1}{\tau}}, \ldots ,\sqrt{\frac{1}{\tau}}\right)^\transpose\right)$ is a saddle point of~\eqref{prob:vector-Lagrangian}, which contradicts the assumption that $\sum_{i=1}^p d\left(\sqrt{\frac{1}{\tau}},\widehat{\lambda}_i\right) = \rho_{\max} >  \rho$. 
    This shows that $\gamma\opt>0$ as desired. Moreover, Lemma~\ref{lemma:property-of-F} below ensures the uniqueness of $\gamma\opt$, which completes the proof.
\end{proof}

\begin{lemma}[Strictly decreasing $F_{\wh\lambda}$]\label{lemma:property-of-F}
    Let the function $
    F_{\wh\lambda}:\R_+\to \R$ be defined by 
    \bgeqn 
    F_{\wh\lambda}(\gamma)=\sum_{i=1}^p d(\varphi(\tau, \gamma, \widehat{\lambda}_i), \widehat{\lambda}_i).
    \edeqn 
    If Assumptions~\ref{ass:regularity-nominal}-\ref{ass:regularity} hold, then the function $F_{\wh\lambda}$ is strictly decreasing over $\R_+$. If Assumption~\ref{ass:reg-radius} holds in addition, then 
    \bgeqn 
    F_{\wh\lambda}(0) =  \rho_{\max} \geq  \rho \quad  \text{and}\quad  \lim_{\gamma\to\infty} F_{\wh\lambda}(\gamma) = 0 \leq  \rho.
    \edeqn 
\end{lemma}
\begin{proof}[Proof of Lemma~\ref{lemma:property-of-F}]
    We first show that $F_{\wh\lambda}$ is strictly decreasing. If $\wh\lambda_i=1/\sqrt{\tau}$, then by Proposition~\ref{prop:pro-of-varphi}~(i), $d(\varphi(\tau, \gamma, {\wh\lambda}_i), \wh\lambda_i)\equiv 0$. Note that by Assumption~\ref{ass:regularity}, there always exists $\wh\lambda_i$ such that $\wh\lambda_i\neq 1/\sqrt{\tau}$. So it suffices to prove that $d(\varphi(\tau, \gamma, {\wh\lambda}_i), \wh\lambda_i)$ is strictly decreasing on $\gamma$ if $\wh\lambda_i\neq 1/\sqrt{\tau}$. If $\wh\lambda_i > 1/\sqrt{\tau}$, then for $\gamma_1 > \gamma_2$, we have 
    \[
        \varphi(\tau, \gamma_2, \wh{\lambda}_i) < \varphi(\tau, \gamma_1, \wh{\lambda}_i) < \wh\lambda_i
    \]
    by Proposition~\ref{prop:pro-of-varphi} (i) and (ii). Then by Lemma~\ref{lemma:monoton-d} below, we have 
    \[
        d(\varphi(\tau, \gamma_2, \wh{\lambda}_i), \wh\lambda_i) > d(\varphi(\tau, \gamma_1, \wh{\lambda}_i), \wh\lambda_i),
    \]
    which proves that $d(\varphi(\tau, \gamma, \wh{\lambda}_i), \wh\lambda_i)$ is strictly decreasing on $\gamma$ if $\wh\lambda_i > 1/\sqrt{\tau}$. If $\wh\lambda_i < 1/\sqrt{\tau}$, then for $\gamma_1 > \gamma_2$, similarly by Proposition~\ref{prop:pro-of-varphi} (i) and (ii) we have 
    \[
        \varphi(\tau, \gamma_2, \wh{\lambda}_i) > \varphi(\tau, \gamma_1, \wh{\lambda}_i) > \wh\lambda_i.
    \]
    By Lemma~\ref{lemma:monoton-d}, we have 
    \[
        d(\varphi(\tau, \gamma_2, \wh{\lambda}_i), \wh\lambda_i) > d(\varphi(\tau, \gamma_1, \wh{\lambda}_i), \wh\lambda_i).
    \]
    Thus, we derive that $F$ is strictly decreasing over $\R_+$. 

    Finally, by Proposition~\ref{prop:pro-of-varphi}(iii) and Assumption~\ref{ass:reg-radius}, we have $F_{\wh\lambda}(0) = \sum_{i=1}^p d(\sqrt{\frac{1}{\tau}}, \wh\lambda_i) = \rho_{\max} \geq  \rho$ and $\lim_{\gamma\to\infty} F_{\wh\lambda}(\gamma) = \sum_{i=1}^p d(\wh\lambda_i,\wh\lambda_i) = 0 <  \rho$ and thus completes the proof.
\end{proof}

\begin{lemma}[Monotonicity of $d$]\label{lemma:monoton-d}
    If Assumptions~\ref{ass:convex-divergence} and~\ref{ass:Spectral-divergence} hold, then we have 
    \begin{enumerate}[label=(\roman*)]
        \item if $a_1 < a_2 < b$, then $d(a_1,b) > d(a_2,b)$,
        \item if $a_1 > a_2 > b$, then $d(a_1,b) > d(a_2,b)$.
    \end{enumerate}
\end{lemma}
\begin{proof}[Proof of Lemma~\ref{lemma:monoton-d}]
    By convexity of $d$, for $a_1< a_2<b$ we have 
    \[
        d(a_1,b)\geq d(a_2,b) + (a_1-a_2)d_b^\prime(a_2) > d(a_2,b),
    \]
    which proves assertion (i). Assertion (ii) can be proved similarly.
\end{proof}

\subsection{Proof of Proposition~\ref{prop:property-of-F}}

\begin{proof}[Proof of Proposition~\ref{prop:property-of-F}]
    In view of Lemma~\ref{lemma:property-of-F}, it only remains to show that $F_{\wh\lambda}$ is differentiable over $\R_+$. Note that $d_b(\cdot)$ is differentiable by Assumption~\ref{ass:Spectral-divergence}(ii) and $\varphi_{\tau,b}(\cdot)$ is differentiable by Proposition~\ref{prop:pro-of-varphi}(ii). Then by the chain rule, we conclude that $F_{\wh\lambda}$ is differentiable over $\R_+$.
\end{proof}

\subsection{Proof of Theorem~\ref{thm:nonlinear-shrinkage}}
\begin{proof}[Proof of Theorem~\ref{thm:nonlinear-shrinkage}]
We show that $\Sigma\opt(\tau, \rho)$ is continuous in $\rho$ for each fixed $\tau$. To facilitate reading, we recall that
\begin{align}\label{eq:construct-Sigma-opt-1}
    \Sigma\opt(\tau,  \rho)=\widehat{V}\Phi(\tau, \gamma_{\wh\lambda}\opt( \rho), \widehat{\lambda})\widehat{V}^\transpose\Let\wh V \diag(\varphi(\tau, \gamma_{\wh\lambda}\opt( \rho), \widehat{\lambda}_1), \ldots ,\varphi(\tau, \gamma_{\wh\lambda}\opt( \rho), \widehat{\lambda}_p))\wh V^\transpose,  \rho\in (0,  \rho_{\max}].
\end{align}
By definition in~\eqref{def-gamma-opt}, $\gamma_{\wh\lambda}\opt(\rho)$ is continuous in $\rho$. On the other hand, it follows by  Proposition~\ref{prop:pro-of-varphi} (ii) that  $\varphi(\tau,\gamma,b)$ is continuous in $\gamma$. The continuity of $\Sigma\opt$ follows directly from \eqref{eq:construct-Sigma-opt-1}.
Recall that $\lim_{\rho\downarrow 0}\gamma_{\wh\lambda}\opt(\rho) = \infty$ and by Proposition~\ref{prop:pro-of-varphi}(iii), $\lim_{\gamma\to\infty} \varphi(\tau, \gamma, \widehat{\lambda}_i) = \widehat{\lambda}_i$. Then
\[
    \lim_{ \rho\downarrow 0}\Sigma\opt(\tau, \rho) = \lim_{\gamma\to\infty}\wh V \diag(\varphi(\tau, \gamma, \widehat{\lambda}_1), \ldots ,\varphi(\tau, \gamma, \widehat{\lambda}_p))\wh V^\transpose =\wh V \diag(\wh\lambda_1,\ldots,\wh\lambda_p)\wh V^\transpose =\wh\Sigma.
\]
Likewise, since $\gamma_{\wh\lambda}\opt(\rho_{\max}) = 0$, and $\varphi(\tau,0,\wh\lambda_i)=\sqrt{1/\tau}$, then
\[
    \Sigma\opt(\tau, \rho_{\max}) = \wh V \diag(\varphi(\tau, 0, \widehat{\lambda}_1), \ldots ,\varphi(\tau, 0, \widehat{\lambda}_p))\wh V^\transpose =\wh V \diag(\sqrt{1/\tau},\ldots,\sqrt{1/\tau})\wh V^\transpose = \sqrt{\frac{1}{\tau}} I.
\]
The assertions (i)-(iii) of the theorem follow directly from Proposition~\ref{prop:pro-of-varphi} (i)-(ii).
\end{proof}

\subsection{Proof of Proposition~\ref{prop:vphi-gamma}}

\begin{proof}[Proof of Proposition~\ref{prop:vphi-gamma}]
The conclusion on the eigenvalue mappings follows directly from 
the definition of $\varphi$ (see equation~\eqref{eq:def-varphi}) by
    substituting the specific forms of generators $d$ into the equation.
    We omit the details of the calculations.

In what follows, we derive the upper bound,  denoted by $\overline{\gamma}$, of the dual variables $\gamma\opt$.
By Proposition~\ref{prop:property-of-F}, $\overline{\gamma}\geq \gamma\opt$ if and only if $F_{\wh\lambda}(\overline{\gamma})\leq  \rho$, i.e., 
    \bgeqn 
        \sum_{i=1}^p d(\varphi(\tau, \overline{\gamma}, \wh\lambda_i), \wh\lambda_i) \leq  \rho.
    \edeqn
    It suffices to choose $\overline{\gamma}$ such that 
    \begin{align}\label{ineq:d<=eps/p}
        d(\varphi(\tau, \overline{\gamma}, \wh\lambda_i), \wh\lambda_i) \leq \frac{ \rho}{p},\;\; \text{for}\;  i=1,...,p.
    \end{align}
    To ease the exposition, let
    $a\opt \Let \varphi(\tau,\gamma,b)$. We also recall that by Theorem~\ref{thm:nonlinear-shrinkage} (i) and (ii)
    \bgeqn 
    \wh\lambda_1\leq \varphi(\tau,\gamma,\wh\lambda_i)\leq \wh\lambda_p\;\; \forall \gamma>0, 
    \edeqn 
   for $i=1,...,p$. 

    \underline{Upper bound of Kullback-Leibler divergence}. 
    Substituting $d(a,b) = \frac{1}{2}\left(\frac{a}{b}-1-\log\frac{a}{b}\right)$ into~\eqref{ineq:d<=eps/p}, we obtain 
    \[
        \frac{1}{2}\left(\frac{\varphi(\tau, \overline{\gamma}, \wh\lambda_i)}{\wh\lambda_i} - 1 - \log\frac{\varphi(\tau, \overline{\gamma}, \wh\lambda_i)}{\wh\lambda_i}\right)\leq \frac{ \rho}{p},\;\; \text{for}\; i=1,...,p.
    \]
   The inequalities above are guaranteed by 
    \begin{align}\label{ineq:gamma-bar-KL}
        \frac{\varphi(\tau, \overline{\gamma}, \wh\lambda_i)}{\wh\lambda_i} - 1 \leq \frac{ \rho}{p} \text{ and }  
        -\log\frac{\varphi(\tau, \overline{\gamma}, \wh\lambda_i)}{\wh\lambda_i} \leq \frac{ \rho}{p},\;\; \text{for} \;  i=1,...,p.
    \end{align}
    By Proposition~\ref{prop:pro-of-varphi}(iv), $\phi(b) \Let\frac{\varphi(\tau,\gamma,b)}{b}$ is non-increasing in $b$. Moreover, by Proposition~\ref{prop:pro-of-varphi}(ii) and the assumption that $\wh\lambda_1<\sqrt{\frac{1}{\tau}} < \wh\lambda_p$, we have $\frac{\varphi(\tau, \overline{\gamma}, \wh\lambda_1)}{\wh\lambda_1} > 1$ and $\frac{\varphi(\tau, \overline{\gamma}, \wh\lambda_p)}{\wh\lambda_p}<1$. 
   Inequalities \eqref{ineq:gamma-bar-KL} are ensured by  
    \begin{align}\label{ineq:gamma-bar-KL-2}
        \frac{\varphi(\tau, \overline{\gamma}, \wh\lambda_1)}{\wh\lambda_1} - 1 \leq \frac{ \rho}{p} \text{ and }  
        \log\frac{\wh\lambda_p}{\varphi(\tau, \overline{\gamma}, \wh\lambda_p)} \leq \frac{ \rho}{p}.
    \end{align}
    Note that 
    \begin{align*}
        \varphi(\tau, \overline{\gamma}, b) = \frac{2(2+\overline{\gamma})b}{\left(\overline{\gamma}+\sqrt{\overline{\gamma}^2+8\tau(2+\overline{\gamma})b^2}\right)} \leq \frac{(2+\overline{\gamma})b}{\overline{\gamma}}
    \end{align*}
    for $b>0$.
    By taking $\overline{\gamma}\geq \frac{2p}{ \rho}$, we have 
    \[
        \varphi(\tau, \overline{\gamma}, \wh\lambda_1)\leq \frac{(2+\overline{\gamma})\wh\lambda_1}{\overline{\gamma}} \leq (\frac{ \rho}{p}+1)\wh\lambda_1,
    \]
    which gives rise to the first inequality of~\eqref{ineq:gamma-bar-KL-2}. On the other hand, since 
    \begin{align*}
        \varphi(\tau, \overline{\gamma}, b) = \frac{2(2+\overline{\gamma})b}{\left(\overline{\gamma}+\sqrt{\overline{\gamma}^2+8\tau(2+\overline{\gamma})b^2}\right)} \geq \frac{2(2+\overline{\gamma})b}{\overline{\gamma}+\overline{\gamma}+4\tau b^2+2} \geq \frac{\overline{\gamma}b}{,\overline{\gamma}+2\tau b^2+1}
    \end{align*}
    then 
    \[
        \overline{\gamma}\geq \frac{2\tau \wh\lambda_p^2+1}{e^{ \rho/p}-1}\;\Rightarrow\; \frac{\overline{\gamma}\wh\lambda_p}{\overline{\gamma}+2\tau \wh\lambda_p^2+1} \geq \wh\lambda_p e^{-\frac{ \rho}{p}}
    \]
which gives rise to the second inequality of~\eqref{ineq:gamma-bar-KL-2}. Combining the above results, we obtain an upper bound of $\gamma\opt$: 
    \[
        \overline{\gamma} = \max\left\{\frac{2p}{ \rho},\frac{2\tau \wh\lambda_p^2+1}{e^{ \rho/p}-1}\right\}.
    \]

    \underline{Upper bound of Wasserstein divergence}. The generator is $d(a,b) = a+b-2\sqrt{ab}$, and the eigenvalue mapping is the unique positive solution $a\opt$ of 
    \[
        \frac{1}{a\opt} - \tau a\opt - \gamma\left(1-\frac{\sqrt{b}}{\sqrt{a\opt}}\right) = 0.
    \]
    Reformulating the equation gives
    \begin{align*}
        d(a\opt,b)=\left(\sqrt{a\opt}-\sqrt{b}\right)^2 = \frac{(1-\tau {a\opt}^2)^2}{\gamma^2 a\opt}.
    \end{align*}
    Then~\eqref{ineq:d<=eps/p} becomes 
    \begin{align*}\label{ineq:gamma-bar-W}
        \frac{(1-\tau \varphi(\tau,\overline{\gamma},\wh\lambda_i)^2)^2}{\overline{\gamma}^2 \varphi(\tau,\overline{\gamma},\wh\lambda_i)}\leq \frac{ \rho}{p}\;\Longleftrightarrow\;
        \overline{\gamma}^2 \geq \frac{p(1-\tau \varphi(\tau,\overline{\gamma},\wh\lambda_i)^2)^2}{ \rho\varphi(\tau,\overline{\gamma},\wh\lambda_i)}
        ,\;\; \text{for}\; i=1,...,p.
    \end{align*}
    Note that 
    \begin{align*}
        \frac{p(1-\tau \varphi(\tau,\overline{\gamma},\wh\lambda_i)^2)^2}{ \rho\varphi(\tau,\overline{\gamma},\wh\lambda_i)} \leq \frac{p(1-\tau \varphi(\tau,\overline{\gamma},\wh\lambda_i)^2)^2}{ \rho\varphi(\tau,\overline{\gamma},0)} \leq \max\left\{\frac{p}{ \rho \varphi(\tau,\overline{\gamma},0)}, \frac{p(1-\tau \wh\lambda_p^2)^2}{ \rho \varphi(\tau,\overline{\gamma},0)}\right\},
    \end{align*}
    where the last inequality is by Lemma~\ref{lemma:1-tau-phi} and 
    \[
        \varphi(\tau,\overline{\gamma},0) = \frac{-\overline{\gamma}+\sqrt{\overline{\gamma}^2+4\tau}}{2\tau}.
    \]
    Then it suffices to have 
    \[
        \max\left\{\frac{p}{ \rho \varphi(\tau,\overline{\gamma},0)}, \frac{p(1-\tau \wh\lambda_p^2)^2}{\rho \varphi(\tau,\overline{\gamma},0)}\right\} \leq \overline{\gamma}^2.
    \]
    Solving the above inequality, we conclude that 
    \[
        \overline{\gamma} = \max\left\{\frac{\eta_1+\sqrt{\eta_1^2+16\eta_1\tau^{2.5}}}{8\tau^2},\frac{\eta_2+\sqrt{\eta_2^2+16\eta_2\tau^{2.5}}}{8\tau^2}\right\}
    \]
    satisfies the inequality, where
    \[
        \eta_1 = \frac{p}{\rho}~~~\text{and}~~~\eta_2 = \frac{p(1-\tau\wh\lambda_p^2)^2}{\rho}.
    \]

    \underline{Upper bound of symmetrized Stein divergence}. The generator is $d(a,b)=\frac{1}{2}\left(\frac{b}{a}+\frac{a}{b}-2\right)$, and the eigenvalue mapping is the unique positive solution $a\opt$ of 
    \begin{align*}\label{eq:def-SS-solution-mapping}
        \frac{1}{a\opt}-\tau a\opt -\frac{\gamma}{2}\left(-\frac{b}{{a\opt}^2}+\frac{1}{b}\right)=0.
    \end{align*}
    Reordering the terms of the above equation gives 
    \begin{align*}
        d(a\opt,b)=\frac{(a\opt-b)^2}{2a\opt b} = \frac{2a\opt b(1-\tau {a\opt}^2)^2}{\gamma^2(a\opt+b)^2}.
    \end{align*}
    Then~\eqref{ineq:d<=eps/p} becomes 
    \begin{align*}
        \frac{2\wh\lambda_i\varphi(\tau,\overline{\gamma},\wh\lambda_i)(1-\tau \varphi(\tau,\overline{\gamma},\wh\lambda_i)^2)^2}{\overline{\gamma}^2(\varphi(\tau,\overline{\gamma},\wh\lambda_i)+\wh\lambda_i)^2} \leq \frac{ \rho}{p} \;\Leftrightarrow \; \overline{\gamma}^2 \geq \frac{2p\wh\lambda_i\varphi(\tau,\overline{\gamma},\wh\lambda_i)(1-\tau \varphi(\tau,\overline{\gamma},\wh\lambda_i)^2)^2}{\rho(\varphi(\tau,\overline{\gamma},\wh\lambda_i)+\wh\lambda_i)^2}, 
        \forall i=1,\ldots,p.
    \end{align*}
    Note that 
    \begin{align*}
        \frac{2p\wh\lambda_i\varphi(\tau,\overline{\gamma},\wh\lambda_i)(1-\tau \varphi(\tau,\overline{\gamma},\wh\lambda_i)^2)^2}{\rho(\varphi(\tau,\overline{\gamma},\wh\lambda_i)+\wh\lambda_i)^2} \leq 
        \frac{p\wh\lambda_p^2(1-\tau \varphi(\tau,\overline{\gamma},\wh\lambda_i)^2)^2}{2 \rho \wh\lambda_1^4}\leq \max\left\{\frac{p\wh\lambda_p^2}{2 \rho \wh\lambda_1^4},\frac{2p\wh\lambda_p^2(1-\tau \wh\lambda_p^2)^2}{4 \rho \wh\lambda_1^4}\right\},
    \end{align*}
    where the last inequality is by Lemma~\ref{lemma:1-tau-phi}.
    So we set
    \[
        \overline{\gamma} \Let\sqrt{\max\left\{\frac{p\wh\lambda_p^2}{2 \rho \wh\lambda_1^4},\frac{p\wh\lambda_p^2(1-\tau \wh\lambda_p^2)^2}{2 \rho \wh\lambda_1^4}\right\}}.
    \]

    \underline{Upper bound of squared Frobenius divergence}. The generator is $d(a,b)=(a-b)^2$, and the eigenvalue mapping is the unique positive solution $a\opt$ of 
    \begin{align*}\label{eq:def-SF-solution-mapping}
        \frac{1}{a\opt}-\tau a\opt - 2\gamma(a\opt-b) = 0.
    \end{align*}
    Reordering the terms of the above equation gives
    \[
        d(a\opt,b)=(a\opt-b)^2 = \frac{(1-\tau {a\opt}^2)^2}{4\gamma^2{a\opt}^2}.
    \]
    Consequently,~\eqref{ineq:d<=eps/p} reduces to
    \begin{align*}\label{ineq:gamma-bar-W}
        \frac{(1-\tau \varphi(\tau,\overline{\gamma},\wh\lambda_i)^2)^2}{4\overline{\gamma}^2\varphi(\tau,\overline{\gamma},\wh\lambda_i)^2}
        \leq \frac{ \rho}{p}
        \;\Longleftrightarrow\;
        \overline{\gamma}^2\geq \frac{p(1-\tau \varphi(\tau,\overline{\gamma},\wh\lambda_i)^2)^2}{4 \rho\varphi(\tau,\overline{\gamma},\wh\lambda_i)^2}
        ,\;\; \text{for}\;  i=1,...,p.
    \end{align*}
    Note that
    \begin{align*}
        \frac{p(1-\tau \varphi(\tau,\overline{\gamma},\wh\lambda_i)^2)^2}{4 \rho\varphi(\tau,\overline{\gamma},\wh\lambda_i)^2} \leq \frac{p(1-\tau \varphi(\tau,\overline{\gamma},\wh\lambda_i)^2)^2}{4 \rho\varphi(\tau,\overline{\gamma},0)^2} \leq \max\left\{\frac{p}{4 \rho\varphi(\tau,\overline{\gamma},0)^2},\frac{p(1-\tau \wh\lambda_p^2)^2}{4 \rho\varphi(\tau,\overline{\gamma},0)^2}\right\},
    \end{align*}
    where the last inequality is by Lemma~\ref{lemma:1-tau-phi} and 
    \[
        \varphi(\tau,\overline{\gamma},0) = \sqrt{\frac{1}{\tau+2\overline{\gamma}}}.
    \]
    Let 
    \[
        \max\left\{\frac{p}{4 \rho}(\tau+2\overline{\gamma}),\frac{p(1-\tau \wh\lambda_p^2)^2}{4 \rho}(\tau+2\overline{\gamma})\right\} \leq \overline{\gamma}^2.
    \]
    We get
    \[
        \overline{\gamma} \geq \max\left\{\frac{p}{4\rho}+\sqrt{\frac{p^2}{16\rho^2}+\tau\frac{p}{4\rho}}, \frac{p(1-\tau\wh\lambda_p^2)^2}{4\rho}+\sqrt{\frac{p^2(1-\tau\wh\lambda_p^2)^4}{16\rho^2}+\tau\frac{p(1-\tau\wh\lambda_p^2)^2}{4\rho}}\right\}.
    \]

    \underline{Upper bound of weighted Frobenius divergence}. The generator is $d(a,b)=\frac{(a-b)^2}{b}$, and the eigenvalue mapping is the unique positive solution $a\opt$ of 
    \begin{align*}
        \frac{1}{a\opt}-\tau a\opt - \frac{2\gamma}{b}(a\opt-b) = 0.
    \end{align*}
    Reordering the terms of the above equation gives
    \[
        d(a\opt,b)=\frac{(a\opt-b)^2}{b} = \frac{b(1-\tau {a\opt}^2)^2}{4\gamma^2{a\opt}^2}.
    \]
    Consequently~\eqref{ineq:d<=eps/p} reduces to
    \begin{align*}
        \frac{\wh\lambda_i(1-\tau \varphi(\tau,\overline{\gamma},\wh\lambda_i)^2)^2}{4\overline{\gamma}^2\varphi(\tau,\overline{\gamma},\wh\lambda_i)^2}
        \leq \frac{ \rho}{p}
        \;\Longleftrightarrow\;
        \overline{\gamma}^2\geq \frac{p\wh\lambda_i(1-\tau \varphi(\tau,\overline{\gamma},\wh\lambda_i)^2)^2}{4 \rho\varphi(\tau,\overline{\gamma},\wh\lambda_i)^2}
        ,\;\; \text{for}\; i=1,...,p.
    \end{align*}
    Since 
    \begin{align*}
        \frac{p\wh\lambda_i(1-\tau \varphi(\tau,\overline{\gamma},\wh\lambda_i)^2)^2}{4 \rho\varphi(\tau,\overline{\gamma},\wh\lambda_i)^2} \leq \frac{p\wh\lambda_p(1-\tau \varphi(\tau,\overline{\gamma},\wh\lambda_i)^2)^2}{4 \rho\wh\lambda_1^2} \leq \max\left\{\frac{p\wh\lambda_p}{4 \rho\wh\lambda_1^2},\frac{p\wh\lambda_p(1-\tau \wh\lambda_p^2)^2}{4 \rho\wh\lambda_1^2}\right\},
    \end{align*} where the last inequality is by Lemma~\ref{lemma:1-tau-phi}, then we can set
    \begin{align*}
        \overline{\gamma} \Let\sqrt{\max\left\{\frac{p\wh\lambda_p}{4 \rho\wh\lambda_1^2},\frac{p\wh\lambda_p(1-\tau \wh\lambda_p^2)^2}{4 \rho\wh\lambda_1^2}\right\}}.
    \end{align*}
\end{proof}

\begin{lemma}\label{lemma:1-tau-phi}
    Let $\tau>0$ and $f(x) = (1-\tau x^2)^2$. If $0 \leq a \leq \overline{a}$, then $f(a) \leq \max\{1, (1-\tau \overline{a}^2)^2\}$.
\end{lemma}
\begin{proof}
    It is easy to verify that $f(a) \leq 1$ for $0\leq a\leq \sqrt{\frac{2}{\tau}}$. Then it suffices to show that $f(a) \leq f(\overline{a})$ for $\sqrt{2/\tau}\leq a \leq \overline{a}$, i.e.,  
    $f$ is increasing over $[\sqrt{2/\tau}, \infty)$. The derivative of $f$ is $f'(x) = 4\tau x(\tau x^2-1)$. Solving the inequality $f'(x) \geq 0$, we conclude that $f'(x) > 0$ for $x\in[\sqrt{2/\tau}, \infty)$.
\end{proof}

\subsection{Proof of Proposition~\ref{prop:consistency}}
\begin{proof}[Proof of Proposition~\ref{prop:consistency}]
    The proof  is similar to~\citet[Proposition 7]{ref:yue2024geometric}. In this proof, we write $x\opt_{i,n}=\lambda_i(\Sigma\opt_{n}(\tau, \rho_n))$ and $\wh{x}_{i,n}=\lambda_i(\wh\Sigma)$ for $i=1,\ldots,n$ and $n\in \mathbb{N}$. By the strong consistency assumption, $\wh\Sigma_n$ converges to $\Sigma_0$ almost surely as $n\to\infty$. Now fix a particular realization of the uncertainties, where $\wh\Sigma_n$ converges to $\Sigma_0$ deterministically. Then it holds that $\lim_{n\to\infty} \wh{x}_{i,n}=\lambda_i(\Sigma_0)$ for $i=1,\ldots,p$ since the eigenvalue mapping $\lambda_i$ is continuous. On the other hand, by Proposition~\ref{prop:pro-of-varphi}(ii) we know that $0\leq x\opt_{i,n} \leq \max\left\{\sqrt{\frac{1}{\tau}},b\right\}$, i.e., $\{x\opt_{i,n}\}_{n\in\mathbb{N}}$ is bounded. Let $\{x\opt_{i,n_k}\}_{k\in\mathbb{N}}$ be a convergent subsequence of $\{x\opt_{i,n}\}_{n\in\mathbb{N}}$. Then by Assumption~\ref{ass:Spectral-divergence}, it holds that 
    \bgeqn 
    \label{eq:prop9-proof-1}
    d\left(x\opt_{i,n_k},\wh{x}_{i,n_k}\right) \leq \sum_{j=1}^p d(x\opt_{j,n_k},\wh{x}_{j,n_k}) = D\left(\Sigma\opt_n(\tau, \rho_n), \wh\Sigma_n\right) \leq  \rho_{n_k},\quad \forall k \in\mathbb{N},
    \edeqn 
    for $i=1,\ldots,p$. Since $ \rho_{n_k}$ converges to $0$ and $d$ is continuous by Assumption~\ref{ass:Spectral-divergence} (ii), then \eqref{eq:prop9-proof-1} implies that 
    \bgeqn
    \label{eq:prop9-proof-2}
        d\left(\lim_{k\to\infty} x\opt_{i,n_k},\lambda_i(\Sigma_0)\right) = d\left(\lim_{k\to\infty} x\opt_{i,n_k},\lim_{k\to\infty} \wh{x}_{i,n_k}\right) = \lim_{k\to\infty}d\left(x\opt_{i,n_k},\wh{x}_{i,n_k}\right) = 0.
    \edeqn 
    Recall that by Assumption~\ref{ass:Spectral-divergence}, $d$ satisfies the identity of indiscernibles. Thus \eqref{eq:prop9-proof-2} implies that $\lim_{k\to\infty} x\opt_{i,n_k}=\lambda_i(\Sigma_0)$. This shows that every convergent subsequence of the bounded sequence $\{x\opt_{i,n}\}_{n\in\mathbb{N}}$ must have the same limit $\lambda_i(\Sigma_0)$. By~\cite[Exercise 2.5.5]{ref:abbott2015understanding}, it holds that $\lim_{n\to\infty} x\opt_{i,n} = \lambda_i(\Sigma_0)$. The derivation so far applies to every uncertainty realization where $\wh\Sigma_n$ converges to $\Sigma_0$ deterministically. Since $\wh\Sigma_n$ converges to $\Sigma_0$ almost surely, we conclude that $x\opt_{i,n}$ converges to $\lambda_i(\Sigma_0)$ almost surely. This implies that 
    \begin{align*}
        \P\left(\lim_{n\to\infty} \|\Sigma_n\opt(\tau, \rho_n)-\Sigma_0\|_F=0\right) \;\geq\; &\P\left(\lim_{n\to\infty}\left( \|\Sigma_n\opt(\tau, \rho_n)-\wh\Sigma_n\|_F+\|\wh\Sigma_n-\Sigma_0\|_F\right)=0\right)\\
        \;=\; &\P\left(\lim_{n\to\infty}\left(\|x_n\opt-\wh{x}_n\|_2+\|\wh\Sigma_n-\Sigma_0\|_F\right)=0\right)\\
        \;\geq \;&\P\left(\lim_{n\to\infty}\left(\|x_n\opt-\lambda(\Sigma_0)\|_2+\|\wh{x}_n-\lambda(\Sigma_0)\|_2+\|\wh\Sigma_n-\Sigma_0\|_F\right)=0\right) \\
        \;= \; &1,
    \end{align*}
    where $x_n\opt$ and $\wh{x}_n$ are the vectors of $x\opt_{i,n},i=1,\ldots,p$ and $\wh{x}_{i,n},i=1,\ldots,p$, and $\lambda(\Sigma_0)$ is the eigenvalue vector of $\Sigma_0$. Both inequalities hold because of the triangle inequality. The first equality holds by the fact that $\|\Sigma_n\opt(\tau, \rho_n)-\wh\Sigma_n\|_F=\|x_n\opt-\wh{x}_n\|_2$ since $\Sigma_n\opt(\tau, \rho_n)$ and $\wh\Sigma_n$ share the same eigenvectors. The last equality holds by the fact that $x\opt_n$ converges to $\lambda(\Sigma_0)$ almost surely, $\wh{x}_n$ converges to $\lambda(\Sigma_0)$ almost surely, and $\wh\Sigma_n$ converges to $\Sigma_0$ almost surely. This shows that $\Sigma\opt_n(\tau, \rho_n)$ converges to $\Sigma_0$ almost surely and completes the proof.
\end{proof}

\subsection{Proof of Theorem~\ref{thm:optimal-eps-limit}}

Recall that $\gamma_{\lambda}(\rho)$ is defined as the unique solution to $F_{\lambda}(\gamma) = \rho$, where $F_{\lambda}$ is a strictly decreasing function mapping from $[0,+\infty)$ to $(0,\rho_{\max}]$. Then $\rho_n\opt$ can be represented by $\rho_n\opt = F_{\lambda}(\gamma_n\opt)$, where
\begin{align}\label{def:gamma-n-opt}
    \gamma_n\opt \Let \arg\min_{ \gamma\in[0,+\infty)}\;\|\Sigma_0^{-1}-\tau\opt \Sigma_0-\gamma \E_n[ {\nabla_1 D } (\Sigma_0, \wh \Sigma_n)]\|_F^2.
\end{align}
To show the rate of $\rho_n\opt$ as $n\to\infty$, we first analyze the rate of $F_{\lambda}(\gamma)$ as $\gamma\to\infty$.

\begin{proposition}[Limit of $\gamma^2F_{\lambda}(\gamma)$]\label{prop:rate-of-d}
    Let Assumptions~\ref{ass:regularity-nominal}-~\ref{ass:Spectral-divergence} and~\ref{ass:like-frob} hold, and 
    $\tau>0$ and $\lambda\in\R^p_+$ be such that $\lambda_i\neq 1/\sqrt{\tau}$ for some $i=1,\ldots,p$.
    Then
        \bgeqn 
        \label{eq:Prop14-proof-1}
        \lim_{\gamma\to\infty} \gamma^2  F_{\lambda}(\gamma) = \sum_{i=1}^p C_{\lambda_i, d} \left(\frac{1/\lambda_i-\tau \lambda_i}{d_{\lambda_i}^{\prime\prime}(\lambda_i)} \right)^2,
    \edeqn
    where $C_{\lambda_i, d}$ is the constant defined as in Assumption~\ref{ass:like-frob}.
\end{proposition}

\begin{proof}[Proof of Proposition~\ref{prop:rate-of-d}]
    Recall that $F_{\lambda}(\gamma) = \sum_{i=1}^p d(\varphi(\tau, \gamma, \lambda_i), \lambda_i)$.
    It suffices to prove that 
    \bgeqn 
        \lim_{\gamma\to \infty} \gamma^2 d(\varphi(\tau, \gamma, b), b) = C_{b,d} \left(\frac{1/b-\tau b}{d_b^{\prime\prime}(b)} \right)^2
    \edeqn 
    for any fixed $\tau >0$ and $b>0$.
    Let 
    $
        f(a) = \frac{1}{a}-\tau a -\gamma d_b^\prime(a).
    $
    Recall that $\varphi(\tau, \gamma, b)$ is defined as the unique positive solution of $f(a)=0$.
 By Taylor expansion of $f$ at point $b$ to the first order, we obtain
    \begin{align*}
        f(a) = \frac{1}{b} - \tau b + \left(-\frac{1}{b^2}-\tau -\gamma d_b^{\prime\prime}(b)\right)(a-b) + o(|a-b|).
    \end{align*}
    Let $\varphi_{\tau,b}(\gamma) \Let \varphi(\tau, \gamma, b)$, and set $f(\varphi_{\tau,b}(\gamma))=0$, i.e., 
    \begin{align*}
        \varphi_{\tau,b}(\gamma)-b = \frac{1/b-\tau b}{1/b^2+\tau+\gamma d_b^{\prime\prime}(b)} + o(|\varphi_{\tau,b}(\gamma)-b|).
    \end{align*}
    Then 
    \begin{align*}
        (\varphi_{\tau,b}(\gamma)-b)^2 = 
         \;&\left[\frac{1/b-\tau b}{1/b^2+\tau+\gamma d_b^{\prime\prime}(b)}\right]^2 + o(|\varphi_{\tau,b}(\gamma)-b|^2).
    \end{align*}
    Consequently
    \begin{align*}
        \lim_{\gamma\to\infty} \gamma^2 (\varphi_{\tau,b}(\gamma)-b)^2 =\lim_{\gamma \to \infty}\left[\frac{1/b-\tau b}{1/\gamma b^2+\tau/\gamma+d_b^{\prime\prime}(b)}\right]^2 = \left(\frac{1/b-\tau b}{d_b^{\prime\prime}(b)}\right)^2.
    \end{align*}
    Together with Assumption~\ref{ass:like-frob}, we have 
    \[
        \lim_{\gamma \to \infty} \gamma^2 d(\varphi(\tau, \gamma, b), b) = 
        \lim_{\gamma \to \infty}  
        \frac{d(\varphi(\tau, \gamma, b), b)}{(\varphi_{\tau,b}(\gamma)-b)^2 } \cdot \gamma^2 (\varphi_{\tau,b}(\gamma)-b)^2
        =
        C_{b,d} \left(\frac{1/b-\tau b}{d_b^{\prime\prime}(b)}\right)^2.
    \]
    The proof is complete.
\end{proof}
Now we are ready to prove Theorem~\ref{thm:optimal-eps-limit}.

\begin{proof}[Proof of Theorem~\ref{thm:optimal-eps-limit}]
    Since~\eqref{def:gamma-n-opt} is a convex quadratic minimization problem of $\gamma$, we can figure out
    the optimal solution 
    \[
        \gamma\opt_n = \frac{2\left\la \Sigma_0^{-1}-\frac{p}{\|\Sigma_0\|_F^2}\Sigma_0, \E_n {\nabla_1 D }(\Sigma_0, \wh \Sigma_n)\right\ra}{\|\E_n {\nabla_1 D }(\Sigma_0, \wh \Sigma_n)\|_F^2}.
    \]
    By Assumption~\ref{ass:D'neq0},  
    \bgeqn 
       \label{eq:Thm4-proof-CNX-2}
        \lim_{n\to\infty} \gamma\opt_n/n = \lim_{n\to\infty} \frac{2n\left\la \Sigma_0^{-1}-\frac{p}{\|\Sigma_0\|_F^2}\Sigma_0, \E_n {\nabla_1 D }(\Sigma_0, \wh \Sigma_n)\right\ra}{n^2\|\E_n {\nabla_1 D }(\Sigma_0, \wh \Sigma_n)\|_F^2} = \frac{2C_2}{C_1^2}.
\edeqn 
    A combination of \eqref{eq:Prop14-proof-1} and \eqref{eq:Thm4-proof-CNX-2} yields
    \begin{align*}
        \lim_{n\to\infty} n^2  \rho_n\opt = \lim_{n\to\infty} (\gamma_n\opt)^2  F_{\lambda}(\gamma_n\opt)\cdot\lim_{n\to\infty} n^2/ (\gamma_n\opt)^2 =  \rho\opt \Let \frac{C_1^4}{4C_2^2} \sum_{i=1}^p C_{\lambda_i, d} \left(\frac{1/\lambda_i-\tau \lambda_i}{d_{\lambda_i}^{\prime\prime}(\lambda_i)} \right)^2.
    \end{align*}
    The proof is complete.
\end{proof}

\subsection{Proof of Corollary~\ref{coro:optimal-radius}}
To prove Corollary~\ref{coro:optimal-radius}, 
we first verify that some divergence functions satisfy Assumption~\ref{ass:like-frob}-\ref{ass:D'neq0} under Assumption~\ref{ass:normal-distribution} and figure out constants $C_{b,d}, C_1$ and $C_2$. Then Corollary~\ref{coro:optimal-radius} follows directly from Theorem~\ref{thm:optimal-eps-limit}. 

\subsubsection{Proof of equation~\eqref{eq:optimal-eps-limit-KL}}
\begin{proposition}[Locally quadratic KL divergence]\label{prop:KL-frob-like}
    Let $d(a,b)=\frac{1}{2}(\frac{a}{b}-1-\log\frac{a}{b})$ be defined on $\R_{++}\times \R_{++}$. Then for any $b>0$,  
    \bgeqn 
        \lim_{a\to b} \frac{d(a,b)}{(a-b)^2} = \frac{1}{4b^2}.
    \edeqn
\end{proposition}
\begin{proof}[Proof of Proposition~\ref{prop:KL-frob-like}]
    Let $y = \frac{a}{b}-1$. Then 
    \begin{align*}
        \lim_{y\to 0}\frac{d(a,b)}{(a-b)^2} = \lim_{y\to 0} \frac{y-\log(y+1)}{2b^2y^2} = \lim_{y\to 0} \frac{y-\left(y-\frac{y^2}{2}+\frac{y^3}{3}- \ldots \right)}{2b^2y^2} = \lim_{y\to 0} \frac{\frac{y^2}{2}+o(y^2)}{2b^2y^2}=\frac{1}{4b^2}.
    \end{align*}
\end{proof}

\begin{proposition}[Non-degenerate gradient of KL divergence]
\label{prop:KL-D'-limit}
    Let 
    \bgeqn 
    D(\Sigma_1, \Sigma_2) \Let \frac{1}{2}\big(\tr(\Sigma_2^{-1}\Sigma_1)-p+\log\det(\Sigma_2\Sigma_1^{-1})\big)
    \edeqn 
    be defined on $\PD \times \PD$. Then
    \bgeqn 
    {\nabla_1 D }(\Sigma_1, \Sigma_2) = \frac{1}{2}\left(\Sigma_2^{-1}-\Sigma_1^{-1}\right).
    \edeqn 
    Under Assumption~\ref{ass:normal-distribution}, 
    \bgeqn 
    \label{eq:Prop-state-1}
        \lim_{n\to\infty} n\|\E_n[{\nabla_1 D }(\Sigma_0, \wh\Sigma_n)]\|_F = \frac{p+1}{2}\|\Sigma_0^{-1}\|_F > 0
     \edeqn 
     and hence
     \bgeqn 
     \label{eq:Prop-state-2}
        \lim_{n\to\infty} n\left\la \Sigma_0^{-1}-\frac{p}{\|\Sigma_0\|_F^2}\Sigma_0, \E_n [{\nabla_1 D }(\Sigma_0, \wh\Sigma_n)]\right\ra=\frac{p+1}{2}\left(\|\Sigma_0^{-1}\|_F^2-\frac{p^2}{\|\Sigma_0\|_F^2}\right) > 0.
     \edeqn 
\end{proposition}
\begin{proof}[Proof of Proposition~\ref{prop:KL-D'-limit}]
    Note that under Assumption~\ref{ass:normal-distribution}, $
    (n\wh \Sigma_n)^{-1}$ follows Inverse-Wishart distribution and thus 
    \bgeq 
    \E_n[(n\wh \Sigma_n)^{-1}] = \frac{1}{n-p-1}\Sigma_0^{-1}.
    \edeq
    Consequently we have 
    \bgeqn 
      \label{eq:Prop16-proof-1}
        \E_n[{\nabla_1 D }(\Sigma_0-\wh \Sigma_n)] = \frac{1}{2}\E_n [\wh \Sigma_n^{-1}-\Sigma_0^{-1}] = \frac{p+1}{2(n-p-1)}\Sigma_0^{-1}.
    \edeqn 
    A combination of the two equations above yields 
    \[
        \lim_{n\to\infty} n\|\E_n[{\nabla_1 D }(\Sigma_0, \wh\Sigma_n)]\|_F = \lim_{n\to\infty} \frac{n(p+1)}{2(n-p-1)}\|\Sigma_0^{-1}\|_F = \frac{p+1}{2}\|\Sigma_0^{-1}\|_F,
     \]
     which is \eqref{eq:Prop-state-1}.
     Likewise, we can use \eqref{eq:Prop16-proof-1} to establish 
     \eqref{eq:Prop-state-2}, i.e.,
    \begin{align*}
        \lim_{n\to\infty} n\left\la \Sigma_0^{-1}-\frac{p}{\|\Sigma_0\|_F^2}\Sigma_0, \E_n [{\nabla_1 D }(\Sigma_0, \wh\Sigma_n)]\right\ra=\;& \lim_{n\to\infty} \frac{n(p+1)}{2(n-p-1)}\left(\|\Sigma_0^{-1}\|_F^2-\frac{p^2}{\|\Sigma_0\|_F^2}\right) \\=\;& \frac{p+1}{2}\left(\|\Sigma_0^{-1}\|_F^2-\frac{p^2}{\|\Sigma_0\|_F^2}\right) > 0,
    \end{align*}
    where the last inequality is due to the fact that 
    \begin{align*}
        \|\Sigma_0\|_F^2\|\Sigma_0^{-1}\|_F^2 = \left(\sum_{i=1}^p \lambda_i^2\right)\left(\sum_{i=1}^p \frac{1}{\lambda_i^2}\right) > p^2.
    \end{align*}
    Note that the above inequality never becomes equality because $\lambda_1,\ldots,\lambda_p$ are not identical. 
\end{proof}

\begin{proof}[Proof of equation~\eqref{eq:optimal-eps-limit-KL}]
    All the assumptions of Theorem~\ref{thm:optimal-eps-limit} are fulfilled and the constants are $C_{b,d}=\frac{1}{4b^2}$, $C_1=\frac{p+1}{2}\|\Sigma_0^{-1}\|_F$ and $C_2=\frac{p+1}{2}\left(\|\Sigma_0^{-1}\|_F^2-\frac{p^2}{\|\Sigma_0\|_F^2}\right)$. So the limit is 
    \begin{align*}
        \lim_{n\to\infty} n^2  \rho_n\opt = \frac{C_1^4}{4C_2^2} \sum_{i=1}^p C_{\lambda_i, d} \left(\frac{1/\lambda_i-\tau\opt \lambda_i}{d_{\lambda_i}^{\prime\prime}(\lambda_i)} \right)^2 = \frac{(p+1)^2\|\Sigma_0^{-1}\|_F^4}{16\left(\|\Sigma_0^{-1}\|_F^2-\frac{p^2}{\|\Sigma_0\|_F^2}\right)^2}\sum_{i=1}^p \left(1-\tau\opt {\lambda_i}^2\right)^2.
    \end{align*}
\end{proof}

\subsubsection{Proof of equation~\eqref{eq:optimal-eps-limit-W}}
\begin{proposition}[Locally-quadratic Wasserstein divergence]\label{prop:W-frob-like}
    Let $d(a,b)=a+b-2\sqrt{ab}$ defined on $\R_{+}\times \R_{+}$. Then for any $b>0$, we have 
    \[
        \lim_{a\to b} \frac{d(a,b)}{(a-b)^2} = \frac{1}{4b}.
    \]
\end{proposition}
\begin{proof}[Proof of Proposition~\ref{prop:W-frob-like}]
    Note that 
    \begin{align*}
        d(a,b) = a+b-2\sqrt{ab} = (\sqrt{a}-\sqrt{b})^2 = \left(\frac{a-b}{\sqrt{a}+\sqrt{b}}\right)^2.
    \end{align*}
    So
    \begin{align*}
        \lim_{a\to b} \frac{d(a,b)}{(a-b)^2} = \lim_{a\to b} \frac{1}{(\sqrt{a}+\sqrt{b})^2} = \frac{1}{4b}.
    \end{align*}
    This completes the proof.
\end{proof}
\begin{proposition}[Non-degenerate gradient of Wasserstein divergence]\label{prop:W-D'-limit}
    Let \bgeqn 
    D(\Sigma_1, \Sigma_2) = \tr(\Sigma_1+\Sigma_2-2(\Sigma_1\Sigma_2)^{1/2})\edeqn 
    be defined on $\PSD \times \PSD$
    and  ${\nabla_1 D }(\Sigma_1, \Sigma_2) = I - \Sigma_2(\Sigma_1\Sigma_2)^{-1/2}$. Under Assumption~\ref{ass:normal-distribution}, we have 
    \[
        \lim_{n\to\infty} n\|\E_n[{\nabla_1 D }(\Sigma_0, \wh\Sigma_n)]\|_F = \frac{(p+1)\sqrt{p}}{8}
     \]
     and
     \[
        \lim_{n\to\infty} n\left\la \Sigma_0^{-1}-\frac{p}{\|\Sigma_0\|_F^2}\Sigma_0, \E_n [{\nabla_1 D }(\Sigma_0, \wh\Sigma_n)]\right\ra=\frac{p+1}{8}\left(\tr(\Sigma_0^{-1}) - \frac{p}{\|\Sigma_0\|_F^2}\tr(\Sigma_0)\right)> 0.
     \]
\end{proposition}
\begin{proof}[Proof of Proposition~\ref{prop:W-D'-limit}]
    For ${\nabla_1 D }(\Sigma_0, \wh\Sigma_n) = I - \wh\Sigma_n(\Sigma_0\wh\Sigma_n)^{-1/2}$, we have 
    \begin{align*}
        \E_n[{\nabla_1 D }(\Sigma_0, \wh\Sigma_n)] = \E_n[I - \wh\Sigma_n(\Sigma_0\wh\Sigma_n)^{-1/2}] = \E_n[I-\Sigma_0^{-1}(\Sigma_0\wh\Sigma_n)^{1/2}].
    \end{align*}
    The following proof is based on the expansion of $\E_n [(\Sigma_0\wh\Sigma_n)^{1/2}]$ given in Lemma~\ref{lemma:Wishart-matrix} (iv). For the first limit, we have
    \begin{align}
    \|\E_n [I - \Sigma_0^{-1}(\Sigma_0\wh\Sigma_n)^{1/2}]\|_F^2 = \;&\tr\left((I-\Sigma_0^{-1}\E_n [(\Sigma_0\Sigma_n)^{1/2}])^2\right)\nonumber\\
    =\;&\tr\left(\left(\frac{p+1}{8n}I+\Sigma_0^{-1}\E_n [R(\wh \Sigma_n)]\right)^2\right)\nonumber\\
    =\;&\frac{(p+1)^2p}{64n^2}+o(1/n^2).
    \end{align}
    Then 
    \begin{align*}
        \lim_{n\to\infty} n\|\E_n[{\nabla_1 D }(\Sigma_0, \wh\Sigma_n)]\|_F = \sqrt{\lim_{n\to\infty} n^2\|\E_n[{\nabla_1 D }(\Sigma_0, \wh\Sigma_n)]\|_F^2} = \frac{(p+1)\sqrt{p}}{8}.
    \end{align*}
    For the second limit, we have 
    \begin{align}
        &\left\la \Sigma_0^{-1}-\frac{p}{\|\Sigma_0\|_F^2}\Sigma_0, \E_n [I - \Sigma_0^{-1}(\Sigma_0\wh\Sigma_n)^{1/2}]\right\ra \nonumber\\=\;& \frac{p+1}{8n}\left(\tr(\Sigma_0^{-1}) - \frac{p}{\|\Sigma_0\|_F^2}\tr(\Sigma_0)\right) + \left\la\Sigma_0^{-1}-\frac{p}{\|\Sigma_0\|_F^2}\Sigma_0, \E_n [R(\wh\Sigma_n)]\right\ra\nonumber\\
        =\;& \frac{p+1}{8n}\left(\tr(\Sigma_0^{-1}) - \frac{p}{\|\Sigma_0\|_F^2}\tr(\Sigma_0)\right) + o (1/n).\
    \end{align}
    Then 
    \begin{align*}
        \lim_{n\to\infty} n\left\la \Sigma_0^{-1}-\frac{p}{\|\Sigma_0\|_F^2}\Sigma_0, \E_n [{\nabla_1 D }(\Sigma_0, \wh\Sigma_n)]\right\ra= \frac{p+1}{8}\left(\tr(\Sigma_0^{-1}) - \frac{p}{\|\Sigma_0\|_F^2}\tr(\Sigma_0)\right)> 0,
    \end{align*}
    where the last inequality is by Lemma~\ref{lemma:was-nonnegative}.
\end{proof}

\begin{proof}[Proof of equation~\eqref{eq:optimal-eps-limit-W}]
    All the assumptions of Theorem~\ref{thm:optimal-eps-limit} are fulfilled and the constants are $C_{b,d}=\frac{1}{4b}$, $C_1=\frac{(p+1)\sqrt{p}}{8}$ and $C_2=\frac{p+1}{8}\left(\tr(\Sigma_0^{-1}) - \frac{p}{\|\Sigma_0\|_F^2}\tr(\Sigma_0)\right)$. So the limit is 
    \begin{align*}
        \lim_{n\to\infty} n^2  \rho_n\opt = \frac{C_1^4}{4C_2^2} \sum_{i=1}^p C_{\lambda_i, d} \left(\frac{1/\lambda_i-\tau\opt \lambda_i}{d_{\lambda_i}^{\prime\prime}(\lambda_i)} \right)^2 = \frac{(p+1)^2p^2}{256\left(\tr(\Sigma_0^{-1}) - \frac{p}{\|\Sigma_0\|_F^2}\tr(\Sigma_0)\right)^2}\sum_{i=1}^p \frac{(1-\tau\opt \lambda_i^2)^2}{\lambda_i}.
    \end{align*}
\end{proof}
\begin{lemma}[Expectation of functions of sample covariance matrix]\label{lemma:Wishart-matrix} Let $\xi_1, \ldots ,\xi_n\in\R^p$ be i.i.d.~normal random vectors with mean $0$ and covariance $\Sigma_0$, and let $\wh\Sigma_n = \frac{1}{n}\sum_{i=1}^n \xi_i\xi_i^\transpose$ be the sample covariance matrix. Then the following holds.
\begin{enumerate}[label = (\roman*)]
        \item $\E_n [\wh \Sigma_n^2]=\frac{1}{n}\Sigma_0^2+\frac{1}{n}\tr(\Sigma_0)\Sigma_0+\Sigma_0^2$,
        \item $\E_n [\wh \Sigma_n \Sigma_0^{-1} \wh \Sigma_n]=\frac{1+p+n}{n}\Sigma_0$,
        \item $\E_n [\|\Sigma_0-\wh\Sigma_n\|_F^2]=\E_n \left[\tr\left((\wh\Sigma_n-\Sigma_0)^2\right)\right] = \frac{1}{n}(\tr(\Sigma_0^2)+\tr(\Sigma_0)^2)$,
        \item $\E_n [(\Sigma_0\wh \Sigma_n)^{1/2}]=\Sigma_0 - \frac{p+1}{8n}\Sigma_0 + \E_n [R(\wh\Sigma_n)]$, where the remainder satisfies $\|\E_n  R(\wh\Sigma_n)\|_F=o(1/n)$.
    \end{enumerate}
\end{lemma}
\begin{proof}[Proof of Lemma~\ref{lemma:Wishart-matrix}]
    Let $S\sim \mc W_p(n, \Sigma)$ be a random matrix following Wishart distribution and $A$ be a constant symmetric matrix, then by~\citet[Theorem 3.3.15 (ii)]{ref:gupta2018matrix}, $\E_n [SAS]=n\Sigma A \Sigma+n\tr(\Sigma A)\Sigma+n^2 \Sigma A \Sigma$. By the fact that $n\wh\Sigma_n\sim \mc W_p(n, \Sigma_0)$, we prove the following results.

   Part (i). $\E_n [\wh \Sigma_n^2] = \frac{1}{n^2}\E_n [(n\wh\Sigma_n)^2] = \frac{1}{n^2}(n\Sigma_0^2+n\tr(\Sigma_0)\Sigma_0+n^2\Sigma_0^2)=\frac{1}{n}\Sigma_0^2+\frac{1}{n}\tr(\Sigma_0)\Sigma_0+\Sigma_0^2$.

     Part (ii). 
    \begin{align*}
        \E_n [\wh \Sigma_n \Sigma_0^{-1} \wh \Sigma_n] =\;& \frac{1}{n^2} \E_n [n \wh \Sigma_n \Sigma_0^{-1} n \wh \Sigma_n]\\ =\;& \frac{1}{n^2}(n\Sigma_0\Sigma_0^{-1}\Sigma_0+n\tr(\Sigma_0\Sigma_0^{-1})\Sigma_0+n^2\Sigma_0\Sigma_0^{-1}\Sigma_0)\\
        =\;& \frac{1}{n^2}(n\Sigma_0+np\Sigma_0+n^2\Sigma_0)
        \\=\;&\frac{1+p+n}{n}\Sigma_0.
    \end{align*}

      Part (iii). 
    \begin{align*}
        \E_n [\|\Sigma_0-\wh\Sigma_n\|_F^2] = \;&\E_n \left[\tr\left((\wh\Sigma_n-\Sigma_0)^2\right)\right]\\=\;&\E_n [\tr(\Sigma_0^2)+\tr(\wh \Sigma_n^2)-2\tr(\Sigma_0\wh\Sigma_n)] 
        \\=\;& \tr(\Sigma_0^2) +\tr(\E_n[\wh\Sigma_n^2]) - \frac{2}{n}\tr(\Sigma_0\E_n[n\wh\Sigma_n])\\
        =\;& \tr\left(\frac{1}{n}\Sigma_0^2+\frac{1}{n}\tr(\Sigma_0)\Sigma_0+\Sigma_0^2\right)-\tr(\Sigma_0^2)
        \\=\;& \frac{1}{n}(\tr(\Sigma_0^2)+\tr(\Sigma_0)^2).
    \end{align*}

     Part (iv). Define matrix-valued function $f(X) = (\Sigma_0 X)^{\half}$. Taylor's expansion (in the Fr\'echet differentiable sense) at $X = \Sigma_0$ gives
    \begin{align*}
        f(X) = \Sigma_0 + \frac{1}{2}(X-\Sigma_0) - \frac{1}{8}(X-\Sigma_0)\Sigma_0^{-1}(X-\Sigma_0) + R(X),
    \end{align*}
    where $\|R(X)\|_F=o(\|\Sigma_0-\wh\Sigma_n\|_F^2)$.
    Evaluating $\E_n [f(X)]$ at $X=\wh \Sigma_n$ gives
    \begin{align*}
        \E_n [(\Sigma_0\wh \Sigma_n)^{\half}] = \E_n [f(\wh\Sigma_n)] = \Sigma_0 - \frac{1}{8}\E_n \left[(\wh\Sigma_n-\Sigma_0)\Sigma_0^{-1}(\wh\Sigma_n-\Sigma_0)\right] + \E_n [R(\wh\Sigma_n)].
    \end{align*}
    By Part (ii), we have
    $
        \E_n \left[(\wh\Sigma_n-\Sigma_0)\Sigma_0^{-1}(\wh\Sigma_n-\Sigma_0)\right] =\E_n[\wh\Sigma_n \Sigma_0^{-1}\wh \Sigma_n- 2\wh\Sigma_n+\Sigma_0] = \frac{p+1}{n}\Sigma_0
    $.
    Then 
    \begin{align*}
        \E_n [(\Sigma_0\wh \Sigma_n)^{1/2}] = \Sigma_0 - \frac{p+1}{8n}\Sigma_0 + \E_n [R(\wh\Sigma_n)],
    \end{align*}
    where 
    \begin{align*}
        \|\E_n [R(\wh\Sigma_n)]\|_F\leq \E_n [\|R(\wh\Sigma_n)\|_F]=\E_n [o(\|\Sigma_0-\wh\Sigma_n\|_F^2)]=o(\E_n [\|\Sigma_0-\wh\Sigma_n\|_F^2]) = o(1/n).
    \end{align*}
    This completes the proof.
\end{proof}

\begin{lemma}[Trace-inverse inequality]\label{lemma:was-nonnegative}
    For any $\Sigma_0\in \PD$, it holds that 
    \[
        \tr(\Sigma_0^{-1}) - \frac{p}{\|\Sigma_0\|_F^2}\tr(\Sigma_0) \geq 0,
    \]
    where the equality holds if and only if all the eigenvalues of $\Sigma_0$ are identical.
\end{lemma}
\begin{proof}[Proof of Lemma~\ref{lemma:was-nonnegative}]
    Let $\lambda_1,\ldots,\lambda_p$ be eigenvalues of $\Sigma_0$. Then, it suffices to prove
    \bgeqn 
        \sum_{i=1}^p \frac{1}{\lambda_i} \geq \frac{p}{\sum_{i=1}^p \lambda_i^2}\sum_{i=1}^p \lambda_i.
    \edeqn
    By the arithmetic mean-harmonic mean, we have
 $
        \sum_{i=1}^p \frac{1}{\lambda_i} \geq \frac{p^2}{\sum_{i=1}^p \lambda_i},
    $
    where the inequality becomes equality if and only if all the eigenvalues are identical, i.e., $\lambda_1=\lambda_2=\ldots=\lambda_p$.
    By the Cauchy-Schwarz inequality, we have 
    \begin{align*}
        \left(\sum_{i=1}^p \lambda_i\cdot 1\right)^2 \leq p\sum_{i=1}^p \lambda_i^2 \;\Rightarrow\; \frac{p}{\left(\sum_{i=1}^p \lambda_i\right)^2} \geq \frac{1}{\sum_{i=1}^p \lambda_i^2}\;\Rightarrow\; \frac{p^2}{\sum_{i=1}^p \lambda_i} \geq \frac{p}{\sum_{i=1}^p \lambda_i^2}\left(\sum_{i=1}^p \lambda_i\right),
    \end{align*}
    where the inequalities become equalities if and only if $\lambda_1=\lambda_2=\ldots=\lambda_p$.
    Combining the above two inequalities completes the proof. 
\end{proof}

\subsubsection{Proof of equation~\eqref{eq:optimal-eps-limit-SS}}
\begin{proposition}[Locally-quadratic Symmetrized Stein divergence]\label{prop:SS-frob-like}
    Let $d(a,b)=\frac{1}{2}(\frac{b}{a}+\frac{a}{b}-2)$ defined on $\R_{++}\times \R_{++}$. Then for any $b>0$, we have 
    \[
        \lim_{a\to b} \frac{d(a,b)}{(a-b)^2} = \frac{1}{2b^2}.
    \]
\end{proposition}
\begin{proof}[Proof of Proposition~\ref{prop:SS-frob-like}]
    Note that 
    \[
        d(a,b) = \frac{b^2+a^2-2ab}{2ab} = \frac{(a-b)^2}{2ab}.
    \]
    So 
    \[
        \lim_{a\to b} \frac{d(a,b)}{(a-b)^2} = \frac{1}{2b^2}.
    \]
\end{proof}
\begin{proposition}[Non-degenerate gradient of Symmetrized Stein divergence]\label{prop:SS-D'-limit}
    Let 
    \bgeqn D(\Sigma_1, \Sigma_2) = \frac{1}{2}\left(\tr(\Sigma_1\Sigma_2^{-1}+\Sigma_2\Sigma_1^{-1})-2p\right)\edeqn 
    be defined on $\PD \times \PD$ and ${\nabla_1 D }(\Sigma_1, \Sigma_2) = \frac{1}{2}\left(\Sigma_2^{-1}-\Sigma_1^{-1}\Sigma_2\Sigma_1^{-1}\right)$. Under Assumption~\ref{ass:normal-distribution}, we have 
    \[
        \lim_{n\to\infty} n\|\E_n[{\nabla_1 D }(\Sigma_0, \wh\Sigma_n)]\|_F = \frac{p+1}{2}\|\Sigma_0^{-1}\|_F
    \]
    and
    \[
    \lim_{n\to\infty} n\left\la \Sigma_0^{-1}-\frac{p}{\|\Sigma_0\|_F^2}\Sigma_0, \E_n [{\nabla_1 D }(\Sigma_0, \wh\Sigma_n)]\right\ra=\frac{p+1}{2}\left(\|\Sigma_0^{-1}\|_F^2-\frac{p^2}{\|\Sigma_0\|_F^2}\right)>0.
    \]
\end{proposition}
\begin{proof}[Proof of Proposition~\ref{prop:SS-D'-limit}]
    For ${\nabla_1 D }(\Sigma_0, \wh\Sigma_n) = \frac{1}{2}\left(\wh\Sigma_n^{-1}-\Sigma_0^{-1}\wh\Sigma_n\Sigma_0^{-1}\right)$, 
    we have 
    \[
        \E_n[{\nabla_1 D }(\Sigma_0, \wh\Sigma_n)] = \frac{1}{2}\left(\frac{n}{n-p-1}\Sigma_0 - \Sigma_0^{-1}\right) = \frac{p+1}{2(n-p-1)}\Sigma_0^{-1}.
    \]
    We are in the same case as in Proposition~\ref{prop:KL-D'-limit}. The result can be proven similarly, and thus we omit the proof.
\end{proof}
\begin{proof}[Proof of equation~\eqref{eq:optimal-eps-limit-SS}]
    All the assumptions of Theorem~\ref{thm:optimal-eps-limit} are fulfilled and the constants are $C_{b,d}=\frac{1}{2b^2}$, $C_1=\frac{p+1}{2}\|\Sigma_0^{-1}\|_F$ and $C_2=\frac{p+1}{2}\left(\|\Sigma_0^{-1}\|_F^2-\frac{p^2}{\|\Sigma_0\|_F^2}\right)$. So the limit is 
    \begin{align*}
        \lim_{n\to\infty} n^2  \rho_n\opt = \frac{C_1^4}{4C_2^2} \sum_{i=1}^p C_{\lambda_i, d} \left(\frac{1/\lambda_i-\tau\opt \lambda_i}{d_{\lambda_i}^{\prime\prime}(\lambda_i)} \right)^2 = \frac{(p+1)^2\|\Sigma_0^{-1}\|_F^4}{32\left(\|\Sigma_0^{-1}\|_F^2-\frac{p^2}{\|\Sigma_0\|_F^2}\right)^2}\sum_{i=1}^p \left(1-\tau\opt {\lambda_i}^2\right)^2.
    \end{align*}
\end{proof}

\end{document}